\newif\ifarxiv
 \arxivtrue

\ifarxiv
\documentclass{article}
\usepackage[utf8]{inputenc}

\usepackage{geometry, amsmath, amssymb, mathtools, bbm, amsthm, natbib, xcolor, hyperref, graphicx}
\usepackage[ruled]{algorithm2e} 
\usepackage{float}

\usepackage{fullpage}
\usepackage{authblk}
\allowdisplaybreaks
\else
\documentclass[anon,12pt]{colt2025/colt2025} 

\usepackage{geometry, amsmath, amssymb, mathtools, bbm, algorithm2e, algorithm, natbib, xcolor, hyperref, graphicx}

\title[Sample Efficient Downstream Swap Regret and Omniprediction for Non-Linear Losses]{Sample Efficient Downstream Swap Regret and Omniprediction for Non-Linear Losses}
\usepackage{times}

\coltauthor{%
 \Name{Author Name1} \Email{abc@sample.com}\\
 \addr Address 1
 \AND
 \Name{Author Name2} \Email{xyz@sample.com}\\
 \addr Address 2%
}
\fi

\DeclareMathOperator*{\E}{\mathbb{E}}
\newcommand{\eps}{\varepsilon}
\newcommand{\1}{\mathbbm{1}}
\newcommand{\cA}{\mathcal{A}}
\newcommand{\cC}{\mathcal{C}}
\newcommand{\cE}{\mathcal{E}}

\newcommand{\cL}{\mathcal{L}}

\newcommand{\cP}{\mathcal{P}}
\newcommand{\cH}{\mathcal{H}}
\newcommand{\cY}{\mathcal{Y}}
\newcommand{\cS}{\mathcal{S}}
\newcommand{\cX}{\mathcal{X}}
\newcommand{\cD}{\mathcal{D}}

\newcommand{\cF}{\mathcal{F}}
\newcommand{\R}{\mathbb{R}}

\DeclareMathOperator*{\argmin}{arg\,min}

\newcommand{\Leontief}{\text{Leon}}
\newcommand{\Round}{\operatorname{Round}}
\newcommand{\ReLU}{\operatorname{ReLU}}
\newcommand{\MReLU}{\operatorname{MReLU}}

\newcommand{\convex}{\mathrm{cvx}}
\newcommand{\bI}{\mathbb{I}}
\newcommand{\bZ}{\mathbb{Z}}
\newcommand{\bR}{\mathbb{R}}
\newcommand{\cG}{\mathcal{G}}

\ifarxiv
\newtheorem{definition}{Definition}

\newtheorem{lemma}{Lemma}
\newtheorem{theorem}{Theorem}
\newtheorem{remark}{Remark}
\newtheorem{corollary}{Corollary}
\newtheorem{proposition}{Proposition}

\newtheorem{assumption}{Assumption}

\else
\newtheorem{assumption}{Assumption}
\fi

\newenvironment{proofof}[1]{\vspace{0.8em}\par{\noindent \textit{Proof of #1.}}}{\hspace*{\fill} $\qed$ \par}

\newcommand{\BR}{\mathrm{BR}}

\newcommand{\mirah}[1]{\textcolor{magenta}{[M: #1]}}

\newcommand{\ar}[1]{\textcolor{blue}{[AR: #1]}}

\begin{document}
\ifarxiv
\title{Sample Efficient Omniprediction and Downstream Swap Regret for Non-Linear Losses}
\author[1]{Jiuyao Lu}
\author[2]{Aaron Roth}
\author[2]{Mirah Shi}
\affil[1]{Department of Statistics and Data Science, University of Pennsylvania}
\affil[2]{Department of Computer and Information Sciences, University of Pennsylvania}
\fi

\maketitle

\begin{abstract}
We define ``decision swap regret'' which generalizes both prediction for downstream swap regret and omniprediction, and give algorithms for obtaining it for arbitrary multi-dimensional Lipschitz loss functions in online adversarial settings. We also give sample complexity bounds in the batch setting via an online-to-batch reduction. When applied to omniprediction, our algorithm gives the first polynomial sample-complexity bounds for Lipschitz loss functions---prior bounds either applied only to linear loss (or binary outcomes) or scaled exponentially with the error parameter even under the assumption that the loss functions were convex.  When applied to prediction for downstream regret, we give the first algorithm capable of guaranteeing swap regret bounds for all downstream agents with non-linear loss functions over a multi-dimensional outcome space: prior work applied only to linear loss functions, modeling risk neutral agents. Our general bounds scale exponentially with the dimension of the outcome space, but we give improved regret and sample complexity bounds for specific families of multidimensional functions of economic interest: constant elasticity of substitution (CES), Cobb-Douglas, and Leontief utility functions.
\end{abstract}

\section{Introduction}
Predictions are useful insofar as they can be used to inform downstream \emph{action} --- but different agents have different objectives. In a dynamically changing environment, each agent --- if sufficiently sophisticated --- could run their \emph{own} learning algorithm, attempting to learn the right actions to play to minimize their loss function. But if a single service could broadcast predictions that were simultaneously useful to \emph{all} agents, no matter what their loss functions were, this would be more compelling. This idea has been recently been pursued in two closely related literatures:
\begin{enumerate}
    \item \textbf{Omniprediction:} Omniprediction, introduced by \cite{gopalan2022omnipredictors}, aims to produce a predictor for a label space $\cY$ that can be efficiently post-processed to optimize a class of downstream loss functions in a way that is competitive with some benchmark class of functions $\cC$ containing hypotheses $c:\cX\rightarrow \mathbb{R}$. The omniprediction literature has for the most part (with a few important exceptions which we discuss in Section \ref{sec:relatedwork}) focused on one dimensional binary labels in the batch setting when examples are drawn i.i.d. from an unknown distribution \cite{gopalan2022omnipredictors,gopalan2023loss,gopalan2024swap,globus2023multicalibration,hu2024omnipredicting}. The fact that labels are paired with informative features is crucial to make the problem interesting in this setting.  An important extension beyond the binary label setting is given by \cite{gopalan2024omnipredictors} who  study omniprediction for real valued labels $\cY = [0,1]$, and show how to obtain real valued omnipredictors for \emph{convex} loss functions, albeit with sample complexity bounds that scale exponentially with the error parameter. 

    \item \textbf{Prediction for Downstream Regret:} A related line of work \cite{kleinberg2023u,noarov2023highdimensional,roth2024forecasting,hu2024predict} has focused on making sequential predictions of adversarially chosen outcomes from some space $\cY$ so that downstream agents, who may have arbitrary discrete action spaces, can optimize their own loss. Several papers from this literature \cite{noarov2023highdimensional,roth2024forecasting} have gone beyond the binary label setting and have focused on arbitrary loss functions that for each action are \emph{linear} in the state $y \in [0,1]^d$ to be predicted. This is useful for modeling risk neutral agents, who wish to maximize their \emph{expected utility} over the label outcome, as expectations are linear. The goal is to guarantee that all downstream decision makers have diminishing regret in the worst case. Motivated by decision makers interacting in competitive environments,  \cite{noarov2023highdimensional,roth2024forecasting,hu2024predict} give results to guarantee downstream decision makers diminishing \emph{swap} regret at near optimal rates  
     (agents with no swap regret guarantees cannot be exploited in the same ways that agents with standard no regret guarantees can be \citep{braverman2018selling,deng2019strategizing,rubinstein2024strategizing,arunachaleswaran2024pareto}). 
    Because calibration is known to be unobtainable in online adversarial settings at $O(\sqrt{T})$ rates \citep{qiao2021stronger,dagan2024improved}, this literature has from the beginning used techniques (like ``U-calibration'' \citep{kleinberg2023u} and extensions of decision calibration \cite{zhao2021calibrating,noarov2023highdimensional,roth2024forecasting}) that sidestep these lower bounds while  being strong enough to guarantee downstream decision makers no (swap) regret. In contrast to the omniprediction literature, this literature has focused on settings in which there are no features. 
\end{enumerate}

In this paper we make use of techniques from each literature to give new results in the other. In particular, we use and extend the approach of \cite{gopalan2024omnipredictors} of approximately representing non-linear loss functions over a higher dimensional basis in which they can be represented linearly. Using this basis we make online adversarial predictions that are useful to downstream decision makers with arbitrary action spaces and Lipschitz loss functions in multi-dimensional real-valued action spaces, and that guarantee no swap regret with respect to rich classes of context-dependent benchmarks. Modeling agents as optimizing a non-linear loss is a common way of modeling risk aversion in strategic settings, and so this extends the prediction-for-downstream-action literature beyond prediction for risk neutral agents. Rather than using calibration in a batch setting as \cite{gopalan2024omnipredictors} do, we obtain our results using algorithms for guaranteeing an extension of decision calibration that guarantees swap regret in an online adversarial setting \cite{noarov2023highdimensional}. This gives regret bounds, that when translated back into sample complexity bounds via an online to batch reduction in the style of \cite{gupta2022online}, gives exponentially improved bounds compared to \cite{gopalan2024omnipredictors}. \citet{gopalan2024omnipredictors} give sample complexity bounds for 1-dimensional convex Lipschitz loss functions that scale exponentially with the desired error parameter, whereas we give \emph{the first omniprediction sample complexity bounds for Lipschitz losses that scale polynomially, rather than exponentially in the error parameter} --- without requiring convexity. We also give sample complexity bounds for arbitrary $d$-dimensional Lipschitz loss functions which scale exponentially in $d$, but give improved bounds for specific families of loss functions of economic interest: constant elasticity of substitution (CES), Cobb-Douglas, and Leontief utility functions. For CES loss functions, our sample complexity and regret bounds scale polynomially with both $d$ and the error parameter. For Cobb-Douglas, we give \emph{pseudo-polynomial} bounds, scaling exponentially with the \emph{log} of the inverse error parameter. For Leontief loss functions, our bounds remain exponential in the dimension, but have only a polynomial dependence on the error parameter.

\ifarxiv
\subsection{Additional Related Work}
\label{sec:relatedwork}
Omniprediction was introduced by \cite{gopalan2022omnipredictors}. 
The main technique that has emerged from this literature is to produce predictions that are \emph{multi-calibrated} \citep{hebert2018multicalibration} with respect to some benchmark class of policies \cite{gopalan2022omnipredictors}, or that satisfy a related set of calibration conditions (e.g. calibration and multi-accuracy) \cite{gopalan2023loss}. In particular, \cite{gopalan2023loss} show that omniprediction can be obtained by jointly promising \emph{hypothesis outcome indistinguishability (OI)} and \emph{decision OI}. Decision OI is a generalization of decision calibration, first introduced in \cite{zhao2021calibrating}, and informally requires that the predicted loss of optimizing for a given loss function matches the realized loss, in aggregate over many examples, conditional on the action chosen. Recently several papers have studied omniprediction in the online adversarial setting and have directly optimized for variants of decision OI (rather than calibration, which implies it) to circumvent lower bounds for calibration in the online adversarial setting. \cite{garg2024oracle, dwork2024fairness,newkim} These papers (like almost all of the omniprediction literature with the exception of \cite{gopalan2024omnipredictors}) continue to restrict attention to binary label spaces $\cY = \{0,1\}$. Decision IO does \emph{not} guarantee downstream decision makers no swap regret. In our results, we use a different generalization of decision calibration, first given by \cite{noarov2023highdimensional} and used by \cite{roth2024forecasting}, which informally guarantees that conditional on the action taken by a downstream decision maker, the predicted loss is equal to the realized loss \emph{simultaneously for every action} (not just for the action selected by the decision maker). This strengthening of decision calibration is sufficient to give swap regret guarantees. The notion of decision swap regret that we give generalizes both swap regret and omniprediction, but is distinct from (and weaker than) the notion of ``swap omniprediction'' which was studied in batch settings by \cite{gopalan2024swap}. Swap omniprediction implies calibration, and hence is subject to the same lower bounds that calibration is in online settings \citep{garg2024oracle}. 
\fi



\ifarxiv
\subsection{Organization}

In Section \ref{sec:preliminaries}, we define our notion of decision swap regret. Section \ref{sec:linear} develops machinery to obtain low decision swap regret for any agent with a linear loss function. In Section \ref{sec:basis}, we construct basis functions that give uniform approximations to $d$-dimensional Lipchitz losses, as well as better constructions for several structured special cases of interest. Together, the results established in Sections \ref{sec:linear} and \ref{sec:basis} lead to our results for non-linear losses, which we present in Section \ref{sec:convex}. Here we also show how to obtain a predictor for the offline batch setting via an online-to-batch reduction, and derive guarantees for offline omniprediction as an application of our results. 
\fi

\section{Model and Preliminaries}\label{sec:preliminaries}

Let $\cX$ denote the feature space and $\cY$ denote the outcome/label space. Throughout, we consider $\cY = [0,1]^d$ (more generally, we can consider any convex, bounded $d$-dimensional space with a constant degradation in our bounds). 
We let $\cP$ denote the prediction space. We model agents as having an arbitrary action space $\cA$ and a loss function $\ell: \cA \times \cY \to [0,1]$ taking as input an action and an outcome. We let $\cL$ denote a family of loss functions. 

We take the role of an online forecaster producing predictions that will be consumed by agents with loss functions belonging to a family $\cL$. We consider the following repeated interaction between a forecaster, agents, and an adversary. In every round $t \in [T]$:
    (1) The adversary selects a feature vector $x_t \in \cX$ and a distribution over outcomes $Y_t \in \Delta \cY$;
    (2) The forecaster produces a distribution over predictions $\pi_t\in\Delta\cP$, from which a prediction $p_t \in \cP$ is sampled; 
    (3) Any agent equipped with a loss $\ell \in \cL$ chooses an action $a_t$ as a function of the prediction $p_t$; 
    (4) The adversary reveals an outcome $y_t \sim Y_t$, and the agent (with loss function $\ell$) suffers loss $\ell(a_t, y_t)$.


\subsection{Decision Swap Regret}

We measure our performance using a notion of regret, which we call \textit{decision swap regret}. Informally, decision swap regret compares the loss of actions chosen based on our predictions against the loss of actions suggested by a collection of benchmark policies. A strength of our definition comes from its ability to condition the choice of benchmark policy on the \textit{action} the decision maker chooses. The definition implies that in hindsight, simultaneously for each action $a \in \cA$, on the subsequence of rounds on which the decision maker chose action $a$, they could not have done any better by acting according to any policy $c_a \in \cC$ instead. As we will see this definition generalizes both omniprediction (which does not condition the benchmark policy on the action) and swap regret (which does condition on the action, but uses a benchmark class of constant policies).

\begin{definition}[$(\cL, \cC, \eps)$-Decision Swap Regret] 
Let $\cL$ be a family of loss functions and $\cC$ be a collection of policies $c: \cX \to \cA$. For a sequence of outcomes $y_1,...,y_T\in\cY$, we say that a sequence of predictions $p_1,...,p_T\in\cP$ has $(\cL, \cC, \eps)$ decision swap regret if for any loss function $\ell \in \cL$, there is some action selection rule $k^\ell: \cP \to \cA$ such that choosing actions $a_t = k^\ell(p_t(x_t))$ achieves: 
\[
\frac{1}{T} \sum_{t=1}^T \left( \ell(a_t, y_t) - \ell(c_{a_t}(x_t), y_t) \right) \leq \eps
\]
for any assignment of policies $\{c_a \in \cC\}_{a\in\cA}$.

\end{definition}

\ifarxiv
In other words, decision swap regret measures the worst-case regret incurred by any agent with loss $\ell\in\cL$ who follows our predictions, 
when they could have instead taken actions suggested by $c_a(x_t)$ rather than playing action $a$ for some set of benchmark policies $\{c_a\}_{a\in\cA}$.

\begin{remark}
        Our definition of decision swap regret is not to be confused with the notion of swap omniprediction, defined in \citet{gopalan2024swap} (see also \cite{globus2023multicalibration} for a related characterization of multicalibration in terms of a contextual swap regret condition.) Swap omniprediction allows both the loss function and the benchmark policy to be indexed by the forecaster's prediction. In contrast, in decision swap regret, the loss function is fixed, and the benchmark policy is indexed by the decision maker's best response action (which is a coarsening of the forecaster's prediction). Swap omniprediction is equivalent to multicalibration \citep{gopalan2024swap}, whereas we explicitely use the fact that decision swap regret can be obtained without calibration, which is what allows us to obtain our substantial sample complexity improvements compared to \cite{gopalan2024omnipredictors}.
\end{remark}
\fi

We will  give regret guarantees when agents \textit{best respond} to our predictions---that is, choose the action that would minimize loss if our prediction was accurate. Decision swap regret guarantees that trusting our predictions and acting accordingly  dominates playing any policy in $\cC$, even informed by the best response action to the prediction made by the forecaster, and is akin to a truthfulness condition in mechanism design.

\begin{definition}[Best Response]
    Fix a loss $\ell:\cA\times\cY\to[0,1]$ and a prediction $p \in \cY$. The best response to $p$ according to $\ell$ is the action $\BR^\ell(p) = \argmin_{a\in\cA} \ell(a, p)$.
\end{definition}


\subsection{Examples of Decision Swap Regret}

\paragraph{Omniprediction:} 
\textit{Omniprediction} corresponds the special case of decision swap regret in which the benchmarking policies are action-independent---i.e. $c_{a} = c$ for for some $c\in\cC$, independently of $a$.  We can thus view omniprediction as the ``external regret" version of decision regret. \ifarxiv Omniprediction is often studied in the batch setting, although \cite{garg2024oracle} recently introduced omniprediction in the online setting. \fi Online omniprediction is strictly more general than batch omniprediction---we derive our bounds in the online setting, and then reduce to the batch setting in Section \ref{sec:convex}.

\begin{definition}[$(\cL,\cC,\eps)$-Omniprediction regret]
Let $\cL$ be a family of loss functions and $\cC$ be a collection of policies $c: \cX \to \cA$. We say that a sequence of predictions $p_1,...,p_T$ has $(\cL, \cC, \eps)$ omniprediction regret with respect to a sequence of outcomes $y_1,...,y_T$ if for any loss function $\ell \in \cL$, there is some action selection rule $k^\ell: \cP \to \cA$ such that choosing actions $a_t = k^\ell(p_t(x_t))$ achieves: 
\[
\max_{c \in \cC} \frac{1}{T} \sum_{t=1}^T \left( \ell(a_t, y_t) - \ell(c(x_t), y_t) \right) \leq \eps
\]
\end{definition}

\begin{definition}[$(\cL,\cC,\eps)$-Omnipredictor]
Let $\cL$ be a family of loss functions and $\cC$ be a collection of policies $c : \cX \to \cA$. Let $\cD$ be a join distribution over pairs $(x,y)$. 
We say that a predictor $p: \cX \to \cP$ is a $(\cL,\cC,\eps)$-Omnipredictor if for any loss function $\ell \in \cL$, there is some action selection rule $k^\ell: \cP \to \cA$ such that: 
\[
\max_{c \in \cC} \E_{(x,y) \sim \cD} \left[ \ell(k^\ell(p(x)), y) - \ell(c(x), y) \right] \leq \eps
\]
\end{definition}

\paragraph{Swap Regret:}
The traditional notion of \textit{swap regret} is studied in a context free setting and is the special case of decision swap regret in which $\cC = |\cA|$ is the set of constant policies that play a single action independently of the context.

\begin{definition}[$(\cL, \eps)$-Swap Regret]
    Let $\cL$ be a family of loss functions. A sequence of predictions $p_1,...,p_T \in \cY$ has $(\cL, \eps)$-swap regret with respect to outcomes $y_1,...,y_T$ if for any loss function $\ell \in \cL$ and swap function $\phi:\cA\to\cA$, choosing actions $a_t = \BR^\ell(p_t(x_t))$ achieves:
    \[
    \frac{1}{T} \sum_{t=1}^T \left( \ell(a_t, y_t) - \ell(\phi(a_t), y_t) \right) \leq \eps
    \]
\end{definition}

\ifarxiv
\begin{remark}
    Our notion of swap regret is equivalent to ``Calibration Decision Loss" as defined by \citet{hu2024predict}. They consider the family of loss functions defined as expectations over a 1 dimensional boolean outcome, or equivalently linear over a 1-dimensional outcome. 
\end{remark}
\fi

\section{Decision Swap Regret for Linear Losses}\label{sec:linear}

We first restrict our attention to the task of obtaining low decision swap regret for losses $\ell$ that are linear and Lipschitz functions of the outcome. More precisely, $\ell$ is linear if for any $a\in\cA$, we can write $\ell(a, y) = \sum_{i=1}^d r^\ell_i(a) \cdot y_i$, where $\{r_i^\ell\}_{i\in[d]}$ is a set of coefficient functions. Furthermore, $\ell$ is $L$-Lipschitz (in the $L_\infty$ norm) if for any $y,y'\in\cY$, $|\ell(a,y)-\ell(a,y')|\leq L \|y-y'\|_\infty$. 

\ifarxiv
Our starting point is a general recipe for omniprediction. Specifically, as shown by \cite{gopalan2023loss}, for a family $\cL$ consisting of linear losses, $(\cL,\cC)$-Omniprediction is implied by:
\begin{enumerate}
    \item Calibration: predictions are unbiased conditional on the value of the prediction itself, i.e.\\ $\left|\sum_{t=1}^T \1[p_t=p](p_t-y_t)\right|=0$ for predictions $p\in\cP$
    \item Multiaccuracy: predictions are unbiased conditional on a class of ``tests" derived from $\cL$ and $\cC$, i.e. $\left|\sum_{t=1}^T h_{\ell,c}(x_t)(p_t-y_t)\right|=0$ for ``tests" $h_{\ell,c}$ that depend on losses $\ell\in\cL$ and policies $c\in\cC$.
\end{enumerate}
\fi
The omniprediction literature has focused on 1-dimensional outcomes because calibration is expensive in high dimensions; in $d$ dimensions, the number of (discretized) predictions is exponential in $d$, and so the sample complexity of producing non-trivial calibrated forecasts using existing algorithms scales exponentially in $d$. The downstream regret literature, meanwhile, has tackled the high-dimensional setting using the key insight that when predictions are used to inform downstream actions, it suffices to enforce a coarsening of calibration: \textit{decision calibration} \citep{zhao2021calibrating, noarov2023highdimensional}. (A more recent literature \cite{garg2024oracle,dwork2024fairness,newkim} has used a related observation to get polynomial improvements in omniprediction sample complexity in the 1 dimensional outcome setting.) Decision calibration asks that the forecaster's predictions be unbiased conditional not on the predictions themselves, but rather on the downstream actions induced by best responding to the predictions. Unlike calibration, decision calibration can be obtained at rates polynomial in $d$ with runtime polynomial in $d$. Alongside decision calibration, we give a notion of \textit{decision cross calibration} and show that decision swap regret bounds follow from bounds on decision calibration and decision cross calibration.


We define decision calibration and decision cross calibration below. These differ slightly from the notion of decision calibration originally given by \cite{zhao2021calibrating} and from the subsequent definition of decision-OI \cite{gopalan2023loss} --- those definitions would suffice to give omniprediction bounds, but not decision \emph{swap} regret bounds. The versions we use here were first used in \cite{noarov2023highdimensional}.

\begin{definition}[$(\cL, \beta)$-Decision Calibration]
    Let $\cL$ be a family of loss functions $\ell: \cA \times \cY \to [0,1]$. Let $\beta: \R \to \R$. We say that a sequence of predictions $p_1,...,p_T$ is $(\cL, \beta)$-decision calibrated with respect to a sequence of outcomes $y_1,...,y_T$ if, for every $\ell \in \cL$ and $a\in\cA$:
    \[
    \left\| \sum_{t=1}^T \1[\BR^\ell(p_t) = a] (p_t - y_t) \right\|_\infty \leq \beta(T^\ell(a))
    \]
    where $T^\ell(a) = \sum_{t=1}^T \1[\BR^\ell(p_t) = a]$.
\end{definition}

Informally speaking, decision calibration is enough to guarantee an decision maker  that their estimated reward for playing any action is unbiased conditionally on their decision  --- which is enough to guarantee that they will have no \emph{swap} regret. However to be able to promise that we can compare their performance to that of other \emph{action policies} which map context to action, we will also want the following decision \emph{cross-calibration} condition to hold. 

\begin{definition}[$(\cL, \cC, \alpha)$-Decision Cross Calibration]
    Let $\cL$ be a family of loss functions $\ell: \cA \times \cY \to [0,1]$ and $\cC$ be a collection of policies $c: \cX \to \cA$. Let $\alpha: \R \to \R$. We say that a sequence of predictions $p_1,...,p_T$ is $(\cL, \cC, \alpha)$-decision cross calibrated with respect to a sequence of outcomes $y_1,...,y_T\in\cY$ if for for every $\ell \in \cL$, $a,b\in\cA$, $c\in\cC$:
    \[
    \left\| \sum_{t=1}^T \1[\BR^\ell(p_t) = a, c(x_t)=b] (p_t - y_t) \right\|_\infty \leq \alpha(T^\ell(a, b))
    \]
\end{definition}

\begin{assumption}
    We assume that $\beta, \alpha: \R\to\R$ are concave functions. This will be the case in all the bounds we give; in general, this condition holds for any sublinear error bound $T^c$ for $c<1$.
\end{assumption}

We now show that these two ingredients imply no decision swap regret for linear losses. At a high level, we can minimize decision swap regret by ensuring that our predictions allow every agent to accurately assess their losses (on average) conditionally on their chosen action---for both their selected actions and the actions suggested by benchmark policies. Decision calibration ensures the former, and decision cross calibration ensures the latter. Since agents select the optimal actions given our predictions, their actions are competitive against any choice of benchmark policies.


\begin{theorem}[Decision Swap Regret for Linear Losses]\label{thm:decisionreg-linear}
    Consider an outcome space $\cY \subseteq [0,1]^d$. Let $\cL$ be a family of loss functions $\ell: \cA \times \cY \to [0,1]$ that are linear and $L$-Lipschitz in the second argument. Let $\cC$ be a collection of policies $c: \cX \to \cA$. If the sequence of predictions $p_1,...,p_T$ is $(\cL, \beta)$-decision calibrated and $(\cL, \cC, \alpha)$-decision cross calibrated, then it has $(\cL, \cC, \eps)$-decision swap regret, where $$\eps \leq \frac{L|\cA|\beta(T/|\cA|) + L|\cA|^2 \alpha\left(T/|\cA|^2\right)}{T}$$
    when agents choose actions $a_t = \BR^\ell(p_t)$.
\end{theorem}
    \ifarxiv
    For any $\ell \in \cL$, by definition of best response, we have that for any choice of benchmark policies $\{c_a\in\cC\}_{a\in\cA}$:
    \[
    \sum_{t=1}^T \ell(a_t, p_t) \leq \sum_{t=1}^T \ell(c_{a_t}(x_t), p_t)
    \]
    
    Thus, it suffices to bound the difference in loss under our predictions $p_t$ and the outcomes $y_t$ for both our sequence of chosen actions and the sequence of actions recommended by any choice of benchmark policies. We show this in the next two lemmas, using decision calibration and decision cross calibration, respectively. 
    \fi

    \ifarxiv
    \begin{lemma}\label{lem:decisioncalibration-linear}
        If $p_1,...,p_T$ is $(\cL, \beta)$-decision calibrated, then for any $\ell\in\cL$:
        \[
        \left| \sum_{t=1}^T (\ell(a_t, p_t) - \ell(a_t, y_t)) \right| \leq |\cA|\beta(T/|\cA|)
        \]
    \end{lemma}
    \fi
    \newcommand{\prooflemdecisioncalibrationlinear}{
        Using the linearity of $\ell$, we can write:
        \begin{align*}
            \left| \sum_{t=1}^T (\ell(a_t, p_t) - \ell(a_t, y_t)) \right| &= \left| \sum_{a\in\cA} \sum_{t=1}^T \1[a_t = a] (\ell(a, p_t) - \ell(a, y_t)) \right| \\
            &= \left| \sum_{a\in\cA} \left(\ell\left(a, \sum_{t=1}^T \1[a_t = a] p_t\right) - \ell\left(a, \sum_{t=1}^T \1[a_t = a] y_t\right)\right) \right| \\
            &\leq  L \sum_{a\in\cA} \left\|  \sum_{t=1}^T \1[a_t = a] (p_t - y_t) \right\|_\infty \\
            &\leq L \sum_{a\in\cA} \beta(T^\ell(a))
        \end{align*}
        where the first inequality follows by $L$-Lipschitzness of $\ell$, the second inequality follows from $(\cL, \beta)$-decision calibration. By concavity of $\beta$ and the fact that $\sum_{a\in\cA} T^\ell(a) = \sum_{a\in\cA} \sum_{t=1}^T \1[a_t=a] = T$, this expression is at most:
        \[
        L |\cA|\beta(T/|\cA|)
        \]
    }
    \ifarxiv
    \begin{proof}
        \prooflemdecisioncalibrationlinear
    \end{proof}
    \else
    \fi

    \ifarxiv
    \begin{lemma}\label{lem:multiaccuracy-linear}
        If $p_1,...,p_T$ is $(\cL, \cC, \alpha)$-cross calibrated, then for any $\ell \in \cL$ and any selection of benchmark policies $\{c_a\in\cC\}_{a\in\cA}$:
        \[
        \left| \sum_{t=1}^T (\ell(c_{a_t}(x_t), p_t) - \ell(c_{a_t}(x_t), y_t)) \right| \leq L|\cA|^2 \alpha\left(\frac{T}{|\cA|^2}\right)
        \]
    \end{lemma}
    \fi
    \newcommand{\prooflemmultiaccuracylinear}{
        Using linearity of $\ell$, we can compute:
        \begin{align*}
            \left| \sum_{t=1}^T (\ell(c_{a_t}(x_t), p_t) - \ell(c_{a_t}(x_t), y_t)) \right| &= \left| \sum_{t=1}^T \sum_{a,b \in \cA} \1[a_t=a,c_{a}(x_t)=b] (\ell(b, p_t) - \ell(b, y_t)) \right| \\
            &= \Bigg| \sum_{a,b\in\cA} \Bigg(\ell\left(b, \sum_{t=1}^T \1[a_t = a,c_{a}(x_t)=b] p_t\right) \\ &- \ell\left(b, \sum_{t=1}^T \1[a_t = a,c_{a}(x_t)=b] y_t\right)\Bigg) \Bigg| \\
            &\leq L \sum_{a,b\in\cA} \left\|  \sum_{t=1}^T \1[a_t = a,c_{a}(x_t)=b](p_t - y_t) \right\|_\infty \\
            &\leq L \sum_{a,b\in\cA} \alpha(T^\ell(a,b))
        \end{align*}
        where the first inequality follows from $L$-Lipschitzness of $\ell$, and the second inequality follows from $(\cL,\cC,\alpha)$-decision cross calibration. By concavity of $\alpha$ and the fact that $\sum_{a,b\in\cA} T^\ell(a,b) = \sum_{a,b\in\cA} \sum_{t=1}^T \1[a_t=a,c_{a}(x_t)=b] = T$, this expression is at most:
        \begin{align*}
            L|\cA|^2 \alpha\left(\frac{T}{|\cA|^2}\right)
        \end{align*}
    }
    \ifarxiv
    \begin{proof}
        \prooflemmultiaccuracylinear
    \end{proof}
    \else
    \fi


\newcommand{\proofthmdecisionreglinear}{
     Fix any $\ell\in\cL$. Applying Lemmas \ref{lem:decisioncalibration-linear} and \ref{lem:multiaccuracy-linear}, and the fact that $a_t = \BR^\ell(p_t)$, we can conclude that the decision swap regret to any selection of benchmark policies $\{c_a\in\cC\}_{a\in\cA}$ is bounded by:
    \begin{align*}
        \frac{1}{T} \sum_{t=1}^T (\ell(a_t, y_t) - \ell(c_{a_t}(x_t), y_t)) 
        &\leq \frac{1}{T} \sum_{t=1}^T (\ell(a_t, p_t) - \ell(c_{a_t}(x_t), p_t)) + \frac{L|\cA|\beta(T/|\cA|) + L|\cA|^2 \alpha\left(\frac{T}{|\cA|^2}\right)}{T} \\
        &\leq \frac{L|\cA|\beta(T/|\cA|) + L|\cA|^2 \alpha\left(\frac{T}{|\cA|^2}\right)}{T}
    \end{align*}
}
\ifarxiv
We can now complete the proof of Theorem \ref{thm:decisionreg-linear}.
\begin{proof}
    \proofthmdecisionreglinear
\end{proof}
\fi

\subsection{Algorithm for Decision Calibration and Decision Cross Calibration}

We now turn to the algorithmic problem of producing predictions that are decision calibrated and decision cross-calibrated. We use the \textsc{Unbiased-Prediction} algorithm of \citet{noarov2023highdimensional}, which makes conditionally unbiased predictions in the online  setting. The algorithm and its guarantees are presented in full in Appendix \ref{app:unbiased-prediction}; we defer  further details to \citet{noarov2023highdimensional}.

We will instantiate \textsc{Unbiased-Prediction} to make predictions that simultaneously achieve decision calibration and decision cross calibration; we will refer to this instantiation as \textsc{Decision-Swap}. 
Our guarantees will directly inherit from the guarantees of \textsc{Unbiased-Prediction}. We state them below as expected error bounds over randomized predictions $\pi_t\in\Delta\cP$. 

\ifarxiv
\begin{theorem}\label{thm:unbiasedalg}
    Consider a convex outcome space $\cY \subseteq [0,1]^d$ and the prediction space $\cP=\cY$. Let $\cL$ be a collection of loss functions $\ell: \cA \times \cY \to [0,1]$ that are linear and $L$-Lipschitz in the second argument. Let $\cC$ be a collection of policies $c: \cX \to \cA$. There is an instantiation of \textsc{Unbiased-Prediction} \citep{noarov2023highdimensional}---which we call \textsc{Decision-Swap}---producing predictions $\pi_1,...,\pi_T \in \Delta \cP$ satisfying, for any sequence of outcomes $y_1,...,y_T \in \cY$:
    \begin{itemize}
    \item For any $\ell \in \cL, a\in\cA$:
    \[
    \E_{\substack{p_t\sim\pi_t,\\ t\in[T]}}\left[ \left\| \sum_{t=1}^T \1[\BR^\ell(p_t) = a] (p_t - y_t) \right\|_\infty \right] \leq O\left( \ln(d|\cA||\cL||\cC|T) + \sqrt{\ln(d|\cA||\cL||\cC|T) \E_{\substack{p_t\sim\pi_t,\\ t\in[T]}}[T^\ell(a)]} \right)
    \]
    \item For any $\ell\in\cL, a,b\in\cA, c\in\cC$:
    \[
    \E_{\substack{p_t\sim\pi_t,\\ t\in[T]}}\left[ \left\| \sum_{t=1}^T \1[\BR^\ell(p_t) = a,c(x_t)=b] (p_t - y_t) \right\|_\infty \right] \leq O\left( \ln(d|\cA||\cL||\cC|T) + \sqrt{\ln(d|\cA||\cL||\cC|T) \E_{\substack{p_t\sim\pi_t,\\ t\in[T]}}[T^\ell(a, b)]}\right)
    \]
    \end{itemize}
\end{theorem}

Substituting the above bounds into Theorem \ref{thm:decisionreg-linear}, we arrive at the following corollary bounding the decision swap regret of predictions output by \textsc{Decision-Swap} when agents best respond

\begin{corollary}\label{cor:decisionreg-linear}
    Consider a convex outcome space $\cY \subseteq [0,1]^d$ and a prediction space $\cP=\cY$. Let $\cL$ be a family of loss functions $\ell: \cA \times \cY \to [0,1]$ that are linear and $L$-Lipschitz in the second argument. Let $\cC$ be a collection of policies $c: \cX \to \cA$. The sequence of predictions $\pi_1,...,\pi_T\in\Delta\cP$
    output by \textsc{Decision-Swap} achieves $(\cL, \cC, \eps)$-decision swap regret (in expectation over the randomized predictions), where 
    \[
    \eps \leq O\left(L |\cA| \sqrt{\frac{\ln(d|\cA||\cL||\cC|T)}{T}} \right)
    \]
    when agents choose actions $a_t = \BR^\ell(p_t)$.
\end{corollary}
\fi

\section{Approximation of Non-Linear Multivariate Functions}
\label{sec:basis}

In this section, we extend our results from Section \ref{sec:linear} 
to $d$-dimensional non-linear loss functions by constructing basis functions whose linear combinations can uniformly approximate functions from broad classes of various loss functions. More specifically, we will consider basis representations of outcomes $y$; that is, for a family of losses $\cL$, our goal is to construct a basis that spans the family of parameterized functions $\cF=\{f_a^\ell(y) = \ell(a, y)\}$ for every $a\in\cA$. 

\begin{definition}[$(n,\lambda)$-basis]\label{def:basis}
    Fix a family of functions $\cF=\{f: \cY \rightarrow$ $\bR\}$. We say that $\cS=\left\{s_i: \cY \rightarrow[0,1]\right\}_{i=0}^n$ $(n,\lambda)$-spans $\cF$, or is a $(n,\lambda)$-basis for $\cF$ if for every $f\in\cF$, there exists $\left\{r_i \in \mathbb{R}\right\}_{i=0}^n$ such that:
    \[
    f(y) = r_0+\sum_{i=1}^{n} r_i s_i(y)
    \]
    and the function $r_0+\sum_{i=1}^{n} r_i s_i(y)$ is $\lambda$-Lipschitz in $s_i(y)$ in the $L_\infty$ norm. That is, for any $y,y' \in \cY$:
     \[
     \left| \sum_{i=1}^{n} r_i s_i(y) - \sum_{i=1}^{n} r_i s_i(y') \right| \le \lambda\max_{1 \le i \le n} \left|s_i(y)-s_i(y')\right|
     \]
\end{definition}

We can also refer to an \textit{approximate} basis, which we define following 
\citet{gopalan2024omnipredictors}:

\begin{definition}[$(n,\lambda,\delta)$-approximate basis]\label{def:approximate-basis}
     Fix a family of functions $\cF=\{f: \cY \rightarrow$ $\bR\}$. A basis $\cS=\left\{s_i: \cY \rightarrow[0,1]\right\}_{i=0}^n$ is said to $(n,\lambda,\delta)$-approximately span, or be a $(n,\lambda,\delta)$-approximate basis for, the family $\cF$ if for every $f \in \cF$, there exists coefficients $\left\{r_i \in \mathbb{R}\right\}_{i=0}^n$ such that: 
     \[
     \left|r_0+\sum_{i=1}^{n} r_i s_i(y) -f(y)\right| \leq \delta
     \]
     and the function $r_0+\sum_{i=1}^{n} r_i s_i(y)$ is $\lambda$-Lipschitz in $s_i(y)$ in the $L_\infty$ norm. 

     
     Given a basis $\cS$, we will denote by $s(y)$ the basis vector evaluated at $y$, i.e. $s(y) = (s_1(y),...,s_n(y))$. For simplicity, $\cS$ is also said to $\delta$-approximately span, or be a $\delta$-approximate basis for, the family $\cF$ when it is only necessary to state the approximation error $\delta$. 
     
\end{definition}

\ifarxiv
We observe that the Lipschitz constant $\lambda$ of a linear function is related to the magnitude of its coefficients:

\begin{lemma}\label{lem:lip-constant}
    Suppose $\sum_{i=1}^{n} \left|r_i\right|\leq \lambda$. Then, for any $y,y'\in[0,1]^n$:
    \[
    \left| \sum_{i=1}^{n} r_i y_i - \sum_{i=1}^{n} r_i y_i' \right| \le \lambda\max_{1 \le i \le n} \left|y-y'\right|
    \]
\end{lemma}
\fi

\newcommand{\prooflemlipconstant}{
    We have that:
    \[
    \left| \sum_{i=1}^{n} r_i y_i - \sum_{i=1}^{n} r_i y_i' \right| = \left| \sum_{i=1}^{n} r_i (y_i -y_i') \right| \leq \sum_{i=1}^{n} \left| r_i \right| \cdot \max_{1 \le i \le n} \left|y-y'\right|  \leq \lambda \max_{1 \le i \le n} \left|y-y'\right|
    \]
}
\ifarxiv
\begin{proof}
    \prooflemlipconstant
\end{proof}
\else
\fi

\subsection{Approximate Bases for Some Common Loss Families}

Here, we state approximation guarantees for several notable loss families. We summarize the results in Table \ref{tab:basis}. We will introduce each function class in the table and discuss the construction of the corresponding basis.

\begin{table}[H]
        \centering
        \begin{tabular}{|c|c|c|c|}
            \hline
            Function Class & Size of Basis & Lipschitz Constant & Error \\ \hline
            $\cF^1_{\text{cvx}}$ (1-dimensional convex 1-Lipschitz) & $O\left(\frac{\ln^{4/3}(1/\delta)}{\delta^{2/3}}\right)$ & $O(1)$ & $\delta$ \\
            $\cF^d_L$ ($d$-dimensional $L$-Lipschitz) & $O\left(\left(\frac{L}{\delta}\right)^d\right)$ & $1$ & $\delta$ \\
            $\cF^d_p$ ($d$-dimensional $L_p$ loss) & $pd$ & $O(p^{p+1}d)$ & 0 \\
            $\cF^d_{\Omega_{\beta,g}}$ (monomials of degree $\beta$ over $g$) & $d^\beta$ & $A(Rd)^\beta)$ & 0 \\
            $\cF^d_{\Omega_{\text{exp},g}}$ (exponential functions over $g$) & $O((1/\delta)^{\ln d})$ & $Ae^{Rd}$ & $\delta$ \\
            $\cF^d_{\text{Leon}}$ (Leontiff functions) & $O\left( \frac{d^4c^{d}\ln^{3d+1}(\frac{1}{\delta})}{\delta^{(6d+8)/(d+2)}} \right)$ & $O\Big( \frac{Ld^{\frac{5}{2}}c^{d}\ln^{2d+\frac{1}{2}}(\frac{1}{\delta})}{\delta^{(5d+6)/(d+2)}} \Big)$ & $3L\delta$ \\ \hline
        \end{tabular}
        \caption{Basis for Specific Function Classes}
        \label{tab:basis}
\end{table}

\paragraph{1-dimensional convex 1-Lipschitz functions}
First, let $\cF_\convex^1$ denote the family of convex, 1-Lipschitz functions over $\cY = [0,1]$. \citet{gopalan2024omnipredictors} give an approximate basis for $\cF_\convex^1$.

\ifarxiv
\begin{theorem}[\citet{gopalan2024omnipredictors}]\label{thm:basis-convex-1d}
    For $\delta>0$ sufficiently small, there exist a $\left( O\left( \frac{\ln^{4/3}(1/\delta)}{\delta^{2/3}} \right), O(1), \delta \right)$
    -approximate basis for $\cF_\convex^1$.
\end{theorem}
\fi

\ifarxiv
\begin{remark}
    \cite{gopalan2024omnipredictors} consider basis functions whose range are $[-1,1]$ rather than $[0,1]$. However, these ranges are equivalent up to scaling and shifting of the basis functions. Since we consider $\cY = [0,1]^d$ throughout this paper, we require the range of the basis functions to be $[0,1]$ for simplicity.
\end{remark}
\fi

\paragraph{$L$-Lipschitz functions}
Now turning to the high-dimensional setting, we begin by observing that for any $L$-Lipschitz function over the domain $\cY=[0,1]^d$, it is enough to predict a distribution over a label space discretized to precision $\delta/L$. This transforms the problem to a multi-class prediction problem with $(L/\delta)^d$ labels. Thus, the family of all $L$-Lipschitz functions $\cF^d_{L}$ admits a simple (approximate) basis: indicator functions over sufficiently many gridpoints of $\cY$.  

\ifarxiv
\begin{proposition}\label{prop:basis-lipschitz}
    The basis $\mathcal{S} = \{ \1[y_1,...,y_d = \delta_{i_1},...,\delta_{i_d}] \}_{i_1,...,i_d\in \{1,\cdots,L/\delta\}}$ is a $O\left(\left(\frac{L}{\delta}\right)^d, 1, \delta\right)$-approximate basis for the family of all $L$-Lipschitz functions $\cF^d_L$ over the domain $[0,1]^d$.
\end{proposition}
\fi

\newcommand{\proofpropbasislipschitz}{
    For any $f\in\cF^d_L$, define a piecewise constant function $f_\delta(y_1,\cdots,y_d) = f(\frac{\delta}{L}\cdot \Round(\frac{Ly_1}{\delta}), \cdots, \frac{\delta}{L}\cdot\Round(\frac{Ly_d}{\delta}))$. Here $\Round(x)$ is defined as the closest integer to $x$ in $\{\delta/L, 2\delta/L,\cdots,1\}$ for convenience in downstream analysis. Due to the $L$-Lipschitzness of $f$, we have $|f(y) - f_\delta(y)| \le \delta$ for any $y \in [0,1]^d$. Therefore, it suffices to $\delta$-approximate $f$ at the grids $\{(\delta i_1,\cdots,\delta i_d) \mid i_1,\cdots,i_d \in \{1,\cdots,L/\delta\}\}$. We have that:
    \[
    f_\delta(y_1,...,y_d)  = \sum_{i_1,...,i_d\in \{1,\cdots,L/\delta\}} f_\delta(i_1,...,i_d) \1[y_1,...,y_d = \delta_{i_1},...,\delta_{i_d}]
    \]
    Thus $\mathcal{S} = \{ \1[y_1,...,y_d = \delta_{i_1},...,\delta_{i_d}] \}_{i_1,...,i_d\in \{1,\cdots,L/\delta\}}$ is a basis for every $f_\delta$, and so is a $\delta$-approximate basis for $\cF^d_L$. Moreover, $|\mathcal{S}| = \left(\frac{L}{\delta}\right)^d$.

    It remains to compute the Lipschitz constant given by the basis transformation. Since $\1[y_1,...,y_d = \delta_{i_1},...,\delta_{i_d}] =1$ at exactly one configuration of $i_1,...,i_d$, we have that for any $y, y'\in[0,1]^d, y\neq y'$:
    \[
    \max_{i_1,...,i_d\in \{1,\cdots,L/\delta\}} \left|\1[y_1,...,y_d = \delta_{i_1},...,\delta_{i_d}] -  \1[y_1',...,y_d' = \delta_{i_1},...,\delta_{i_d}]\right| = 1
    \]
    Now, $f_\delta$ has range between 0 and 1, and so $|f_\delta(y_1,...,y_d) - f_\delta(y_1',...,y_d')| \leq 1$. Therefore, for any $y,y'\in[0,1]^d$
    \[
    |f_\delta(y_1,...,y_d) - f_\delta(y_1',...,y_d')| \leq \max_{i_1,...,i_d\in \{1,\cdots,L/\delta\}} \left|\1[y_1,...,y_d = \delta_{i_1},...,\delta_{i_d}] -  \1[y_1',...,y_d' = \delta_{i_1},...,\delta_{i_d}]\right|
    \]
    This completes the proof.
}
\ifarxiv
\begin{proof}
    \proofpropbasislipschitz
\end{proof}
\else
\fi

\ifarxiv
One might hope that for losses with additional structure beyond Lipschitzness, it is possible to find a better approximation. For the remainder of this section, we derive bases spanning some common families of loss functions. We begin with the observation that some familiar losses admit reasonably sized bases. 
\fi


\paragraph{$L_p$ losses}
Consider $L_p$ losses over $\cY = [0,1]^d$: $\ell(a, y) = \sum_{i=1}^d(a_i-y_i)^p$ for $p\in\mathbb{N}$ (here, the action---which we typically think of as the ``prediction" in a learning task---is also a $d$-dimensional vector). Recall that we can define the family $\cF^d_p$ by parameterizing each loss $\ell$ by $a$ --- i.e. $\cF^d_p$ consists of functions $f^\ell_a(y) = \ell(a,y)$ for every $\ell$ and $a$. $f^\ell_a$ has simple underlying linear structure that allows us to articulate a basis. Observe that $f^\ell_a$ expands as follows: 
\begin{align*}
    f^\ell_a = \sum_{i=1}^d (a_i-y_i)^p
    = \sum_{i=1}^d \sum_{k=0}^p {p\choose k} a_i^{p-k} (-y_i)^k
\end{align*}
Thus, the functions $s_{i,k}(y) = (-y_i)^k$ form a basis for $\cF^d_p$, where $r_{i,k}(a) = {p\choose k} a_i^{p-k}$. Therefore, the family of $L_p$ losses can be exactly described using at most $pd$ basis functions.

\ifarxiv
Formally, we have:
\begin{proposition}\label{prop:basis-lp-loss}
    The basis $\mathcal{S} = \{s_{i,k}: [0,1]^d\to[-1,1] \mid s_{i,k}(y) = (-y_i)^k \}_{i\in[d], k\in[p]}$ is a $\left(pd, O(p^{p+1}d)\right)$-basis for the family $\cF^d_p$ corresponding to $L_p$ losses\ifarxiv \footnote{It is common to take $p$ to be a small constant, e.g. $p=2$ gives the $L_2$ loss.}\fi.
\end{proposition}
\fi

\newcommand{\proofpropbasislploss}{
    The basis construction follows from the binomial expansion of $L_p$ losses. The Lipschitz constant follows from the fact that $\sum_{i=1}^d\sum_{k=0}^p \left| {p\choose k} a_i^{p-k} \right| = O(p^{p+1}d)$, and Lemma \ref{lem:lip-constant}.
}
\ifarxiv
\begin{proof}
    \proofpropbasislploss
\end{proof}
\else
\fi

\paragraph{Functions with a linear aggregation of coordinate‐wise nonlinearities}
We next consider general functional forms of losses that---like in the case of $L_p$ losses above---offer underlying linear structure. Let $\cF^d_{\Omega,g}$ denote the family of functions $f:[0,1]^d\to\R$ where for every $f\in\cF^d_{\Omega,g}$, there exists $\{r_i\in\R\}_{i=1}^d$, $g: [0,1] \to \R$, and $\omega\in\Omega$, $\omega:\R\to\R$, such that:
\[
f(y) = \omega\left(\sum_{i=1}^d r_i g(y_i)\right)
\]
That is, $f$ is the composition of a function $\omega$ with a basis transformation $g$ of $y$. 
\ifarxiv
If $g$ is the identity function, then $f$ is a transformation $\omega$ of a linear function of $y$.
\fi
Below we show how to bound the approximate dimension of families $\cF^d_{\Omega,g}$ for several broad classes of $\Omega$, which we will see capture well-known losses in economics. In what follows, we assume that every $|r_i|\leq R$ and $|\sum_{i=1}^d r_i g(y_i)| \leq C$ for constants $R, C$.

Consider the family $\cF^d_{\Omega_{\beta,g}}$ where $\Omega_{\beta,g}$ is the collection of \textit{monomials} of degree $\beta$ with an underlying basis representation given by $g$ --- i.e. for every $f\in \cF^d_{\Omega_{\beta,g}}$, there is a coefficient $A\in\R^+$ such that $f$ can be written as:
\[
f(y) = A \left(\sum_{i=1}^d r_i g(y_i)\right)^\beta
\]
where $\beta \in \mathbb{Z}^+$ is an integer exponent. Most straightfowardly, when $g$ is the identity function, the class contains losses of the form $(\ell_L(a, y))^\beta$, where $\ell_L$ is a linear in $y$. It also includes special cases of the Constant Elasticity of Substitution function \citep{ces}, given by $f(y) = \left( \sum_{i=1}^d a_i y_i^\rho \right)^{v/\rho}$---in particular, when $v/\rho$ is integer-valued.

We show that $\cF^d_{\Omega_{\beta,g}}$ admits a succinct basis. The key idea is that each function in the family can be decomposed into a polynomial function of the underlying basis $g$, and so a basis for $\cF^d_{\Omega_{\beta,g}}$ suffices to include polynomials of $g$.

\ifarxiv
\begin{proposition}\label{prop:basis-mono}
    Let $\Omega_{\beta,g}$ be the collection of monomials of degree $\beta$, $\omega(x) = Ax^\beta$ for $A\in\R^+$. Then,
    $\mathcal{S} = \{g(y_{i_1}) \cdots g(y_{i_\beta})\}_{i_1,...,i_\beta \in [d]}$ is a $\left(d^\beta,A (Rd)^\beta\right)$-basis for the family
    $\cF^d_{\Omega_{\beta,g}}$. 
\end{proposition}
\fi

\newcommand{\proofpropbasismono}{
    We can expand any $f\in \cF^d_{\Omega_{\beta,g}}$ as:
    \begin{align*}
        f(y) = A \left(\sum_{i=1}^d r_i g(y_i)\right)^\beta 
        = A \left(\sum_{i_1,...,i_\beta \in [d]} r_{i_1} \cdots r_{i_\beta} \cdot g(y_{i_1}) \cdots g(y_{i_\beta}) \right)
    \end{align*}
    It follows that the set of functions $\{g(y_{i_1}) \cdots g(y_{i_\beta})\}_{i_1,...,i_\beta \in [d]}$ spans any $f\in \cF^d_{\Omega_{\beta,g}}$. In particular, there are $d^\beta$ such functions. Moreover, we have that $A \sum_{i_1,...,i_\beta \in [d]} \left| r_{i_1} \cdots r_{i_\beta} \right| \leq A \sum_{i_1,...,i_\beta \in [d]} |r_{i_1}| \cdots |r_{i_\beta}| \leq A (Rd)^\beta$. Thus, combined with Lemma \ref{lem:lip-constant}, this proves the proposition.
}
\ifarxiv
\begin{proof}
    \proofpropbasismono
\end{proof}
\else
\fi

Next we consider the family $\cF^d_{\Omega_{\text{exp},g}}$ where $\Omega_{\text{exp},g}$ is a family of exponential functions $\omega(x) = -Ae^x$. That is, for every $f \in \cF^d_{\Omega_{\text{exp},g}}$, there exists $A\in\R^+$ such that we can write:
\[
f(y) = -A\exp\left( \sum_{i=1}^d r_i g(y_i) \right)
\]
This function family includes the well-known Cobb-Douglas production function in economics\footnote{\label{fn:myfootnote}Both the Cobb-Douglas production function and the Leontief production function discussed later are typically defined as utility functions; we give equivalent formulations as loss functions.}, given by $f(y) = - A \prod_{i=1}^d y_i^{\alpha_i} = - A \exp\left( \sum_{i=1}^d \alpha_i \ln y_i \right)$, with $\alpha_i\in(0,1)$ \citep{cobbdouglas}. In this case, $g(y_i) = \ln y_i$.

\begin{remark}
    As stated, the functions we consider might have range outside of $[0,1]$, but we can handle these functions nonetheless by introducing a constant offset, i.e. a coefficient $r_0$, so that they have range between $[0,1]$. Note that this offset preserves any $(n,\lambda, \delta)$-approximation, since $r_0$ does not affect the Lipschitz constant $\lambda$. 
\end{remark}

Like above, we will take advantage of the underlying basis $g$ and construct a basis for $f$ consisting of polynomials of $g$. However, since $\omega$ is no longer a polynomial function itself, we will first appeal to a Taylor approximation of $\omega$. We will show that the Taylor series of $\omega$ converge sufficiently rapidly, allowing us to construct a basis of reasonable size.
\ifarxiv
As a reminder, the $k^{th}$ degree Taylor expansion of a function $h(x)$ at the point $a$ is given by $h_k(x) = \sum_{n=0}^k \frac{h^{(n)}(a)}{n!} (x-a)^n$. The following lemma states the approximation error of a finite Taylor expansion.

\begin{lemma}[Taylor's Theorem]\label{lem:taylors}
    Suppose $h^{(k)}$ exists and is continuous on the interval between $x$ and $a$. Then, $h(x) - h_k(x) = \frac{h^{(k+1)}(z)}{(k+1)!} (x-a)^{k+1}$ for some value $z$ between $x$ and $a$, where $h_k(x)$ is the $k^{th}$ degree Taylor expansion of the function $h(x)$ at the point $a$. 
\end{lemma}
\fi

\ifarxiv
\begin{proposition}\label{prop:basis-exp}
    Let $\Omega_{\text{exp},g}$ be a collection of exponential functions $\omega(x) = -Ae^x$ for $A\in\R^+$. Define the basis: $$\mathcal{S} = \cup_{n=1}^{k} \{ g(y_{i_1}) \cdots g(y_{i_n})\}_{i_1,...,i_n\in[d]}$$ Then, $\mathcal{S}$ is a $\left(O( (1/\delta)^{\ln d}), Ae^{Rd}, \delta \right)$-approximate basis for the family $\cF^d_{\Omega_{\text{exp},g}}$, for $k=O(\ln(1/\delta))$ and $\delta$ sufficiently small (at most a fixed constant).
\end{proposition}
\fi

\newcommand{\proofpropbasisexp}{
    Recall that for any $f\in\cF^d_{\Omega_{\text{exp},g}}$, we can write $f(y) = \omega\left(\sum_{i=1}^d r_i g(y_i)\right) = -A\exp\left(\sum_{i=1}^d r_i g(y_i)\right)$ for some $A\in\R^+$. We can thus approximate $f$ using a finite Taylor expansion of $\omega$. In particular, let $\hat{f}$ be the approximation to $f$ given by the $k^{th}$ degree Taylor expansion of $\omega$ at $a=0$:
    \begin{align*}
        \hat{f}(y) &\coloneqq \omega_k\left(\sum_{i=1}^d r_i g(y_i)\right) \\
        &= \sum_{n=0}^k \frac{\omega^{(n)}(0)}{n!} \left(\sum_{i=1}^d r_i g(y_i)\right)^n \\
        &= \sum_{n=0}^k \frac{\omega^{(n)}(0)}{n!} \sum_{i_1,...,i_n\in[d]} r_{i_1} \cdots r_{i_n} \cdot g(y_{i_i}) \cdots g(y_{i_n}) \\
        &= \sum_{n=0}^k \sum_{i_1,...,i_n\in[d]} \frac{\omega^{(n)}(0)}{n!} r_{i_1} \cdots r_{i_n} \cdot g(y_{i_i}) \cdots g(y_{i_n})
    \end{align*}
    Therefore, the set of functions $\mathcal{S} = \cup_{n=1}^k \{ g(y_{i_1}) \cdots g(y_{i_n})\}_{i_1,...,i_n\in[d]}$ approximately spans $\Omega_{\text{exp},g}$ and thus $\cF^d_{\Omega_{\text{exp},g}}$. By Lemma \ref{lem:taylors}, the approximation error is:
    \begin{align*}
        |f(y) - \hat{f}(y)| = \left|\omega\left(\sum_{i=1}^d r_i g(y_i)\right) - \omega_k\left(\sum_{i=1}^d r_i g(y_i)\right) \right| = \left| \frac{\omega^{(k+1)}(z)}{(k+1)!}\left(\sum_{i=1}^d r_i g(y_i)\right)^{k+1} \right|
    \end{align*}
    for some value $z$ between $0$ and $\sum_{i=1}^d r_i g(y_i)$. Now, using the assumption that $|\sum_{i=1}^d r_i g(y_i)| \leq C$, and the fact that $|\omega^{(k+1)}(z)| = Ae^z \leq Ae^C$, we can bound:
    \begin{align*}
        |f(y) - \hat{f}(y)| &\leq \frac{Ae^C}{(k+1)!}C^{k+1} \\
        &\leq \frac{Ae^C}{\sqrt{2\pi(k+1)}\left(\frac{k+1}{e}\right)^{k+1}}C^{k+1} \\ 
        &= \frac{Ae^C}{\sqrt{2\pi(k+1)}\left(\frac{k+1}{Ce}\right)^{k+1}}
    \end{align*}
    where in the second step we use Stirling's inequality: $n! \geq \sqrt{2\pi n}\left(\frac{n}{e}\right)^n$. It remains to verify that $k$ is large enough so that the approximation error is at most $\delta$, i.e. $|f(y) - \hat{f}(y)| \leq \delta$. Observe that we have the following:
    \begin{align*}
        & \frac{Ae^C}{\sqrt{2\pi(k+1)}\left(\frac{k+1}{Ce}\right)^{k+1}} \leq \delta \\
        \iff & \sqrt{2\pi(k+1)}\left(\frac{k+1}{Ce}\right)^{k+1} \geq \frac{Ae^C}{\delta} \\
        \iff & \frac{1}{2}(\ln(k+1)+\ln(2\pi)) + (k+1)(\ln(k+1)-\ln(Ce)) \geq \ln A + C + \ln\left(\frac{1}{\delta}\right) 
    \end{align*}
    Now, the LHS is bounded below by $(k+1)(\ln(k+1)-\ln(Ce)) \geq \Omega(k\ln k)$. Thus taking $k=O(\ln(1/\delta))$ suffices to bound the approximation error by $\delta$. So, we can conclude that the size of the basis is $|\mathcal{S}| = d+d^2+...+d^k = O(d^k) = O\left( d^{\ln (1/\delta)} \right) = O\left( (1/\delta)^{\ln d} \right)$.

    By Lemma \ref{lem:lip-constant}, it remains to compute the magnitude of coefficients. We have that:
    \begin{align*}
        \sum_{n=0}^k \sum_{i_1,...,i_n\in[d]} \left|\frac{\omega^{(n)}(0)}{n!} r_{i_1} \cdots r_{i_n}\right| \leq \sum_{n=0}^k A\cdot \frac{(Rd)^n}{n!} \leq \sum_{n=0}^\infty A\cdot \frac{(Rd)^n}{n!} = Ae^{Rd}        
    \end{align*}
    Here, we first use the assumption that each $|r_{i_j}| \leq R$ and the fact that $|\omega^{(n)}(0)| = Ae^0 = A$ for all $n$. Then, we use the Taylor expansion for $e^{Rd}$. This completes the proof.
}
\ifarxiv
\begin{proof}
    \proofpropbasisexp
\end{proof}
\else
\fi

\paragraph{Leontief production functions}
Next we consider $\cF_\Leontief^d$, the family of Leontief production functions
that are $L$-Lipschitz on $[0,1]^d$ --- i.e., for every $f \in \cF_\Leontief^d$, there exists $a_1,\cdots,a_d \ge 1/L$ such that we can write:
\[
    f(y) = -\min_{1 \le j \le d}\{y_j/a_j\}
\]


\ifarxiv
\begin{proposition}\label{thm:basis-Leontief}
    For any $\delta>0$, there exists a $\left( O\left( \frac{d^4c^{d}\ln^{3d+1}(1/\delta)}{\delta^{(6d+8)/(d+2)}} \right), O\Big( \frac{Ld^{5/2}c^{d}\ln^{2d+1/2}(1/\delta)}{\delta^{(5d+6)/(d+2)}} \Big), 3L\delta \right)$-approximate basis for $\cF_\Leontief^d$, where $c$ is an universal constant.
\end{proposition} 
The remainder of the section is dedicated to proving Proposition \ref{thm:basis-Leontief}. 
\fi

\newcommand{\proofthmbasisLeontief}{
We directly extend techniques developed by \citet{gopalan2024omnipredictors} for 1-dimensional convex 1-Lipschitz functions to the $d$-dimensional setting. Our approach closely follows that of \citet{gopalan2024omnipredictors} and consists of constructing bases for several intermediate function classes. First, we consider a discretization of the function class. We then represent any discretized function using point indicator functions, which we show can be represented using MReLU functions (a multivariate extension of ReLU functions). Then, we construct a basis for MReLU functions using indicator functions of hyperrectangles (a multivariate extension of interval indicators). Finally, we construct a basis for hyperrectangles using a low-rank approximate factorization of the identity matrix. 

In fact, this approach gives us a general template to represent any non-linear Lipschitz multivariate function. For Leontief functions in particular, we will be able to find an upper bound to the sum of the coefficients in the representation that is better than, as we will see, the naive bound of $(1/\delta)^d$. We proceed with the argument below; missing proofs can be found in Appendix \ref{app:basis}.

For ease of notation in the proof, we will rewrite every $f\in\cF_\Leontief^d$ as the analogous \textit{utility} function (the negative loss function). First, observe that we can write a Leontief utility function in the following form:
\[
f = L \cdot \min_{1 \le j \le d}\left\{\frac{1}{a_jL} \cdot y_j\right\}
\]
where the labels $y_1,\cdots,y_d$ and their coefficients $\frac{1}{a_1L},\cdots,\frac{1}{a_dL}$ are all in $[0,1]$.
    
For notational convenience, we assume $\delta = 1/m$ where $m \in \bZ^+$. Our first step is to discretize and scale both the labels and their coefficients by multiplying them by $m$ and then rounding to integers. 
Let $\Round(x)$ represent the closest integer to $x$ in $\{1,\cdots,m-1\}$. For each $j \in \{1,\cdots,d\}$, let $b_j = \Round(\frac{1}{a_jL\delta}) \in \{1,\cdots,m-1\}$, $z_j = \Round(\frac{y_j}{\delta}) \in \{1,\cdots,m-1\}$. Define $g: \{1,\cdots,m-1\}^d \to \{1,\cdots,(m-1)^2\}$ as: 
\begin{align*}
    g(z_1,\cdots,z_d) = \min_{1 \le j \le d}\{b_jz_j\}
\end{align*}
And denote the family of $g$ as $\cG_\Leontief^d = \{g(z_1,\cdots,z_d) : b_1,\cdots,b_d \in \{1,\cdots,m-1\}\}$, which can be viewed as a scaled and discretized version of Leontief utility functions. We show that $g$ is close to the original Leontief utility function $f$ after scaling it back.
\begin{lemma}
    For any $y_1,\cdots,y_d \in [0,1]$ and Leontief utility function $f \in \cF_\Leontief^d$, the corresponding $g \in \cG_\Leontief^d$ constructed as described above satisfies:
    \[
    \left|f(y_1,\cdots,y_d) - L\delta^2 g\left(\Round\left(\frac{y_1}{\delta}\right),\cdots,\Round\left(\frac{y_d}{\delta}\right)\right)\right| \le 2L\delta
    \]
\end{lemma}
\begin{proof}
    We note that for any $j \in \{1,\cdots,d\}$:
    \begin{align*}
    \left|\frac{1}{a_jL\delta} \cdot\frac{y_j}{\delta} - b_j\cdot\Round(\frac{y_j}{\delta})\right| 
    &\le \left|\frac{1}{a_jL\delta} \cdot\frac{y_j}{\delta} - b_j\cdot\frac{y_j}{\delta}\right| + \left|b_j \cdot\frac{y_j}{\delta} - b_j\cdot\Round(\frac{y_j}{\delta})\right| \\
    &= \left|\frac{1}{a_jL\delta} - b_j\right| \cdot\frac{y_j}{\delta} + b_j \cdot \left|\frac{y_j}{\delta} - \Round(\frac{y_j}{\delta})\right| \\
    &\le 1 \cdot m + m \cdot 1 \\
    &= 2m
\end{align*}
Due to the 1-Lipschitzness of the $\min$ function, we have:
\begin{align*}
    |f(y_1,\cdots,y_d) - L\delta^2 g(\Round(\frac{y_1}{\delta}),\cdots,\Round(\frac{y_d}{\delta}))| &= \left|L\min_{1 \le j \le d}\{\frac{1}{a_jL} \cdot y_j\} - L\min_{1 \le j \le d}\left\{\delta b_j\cdot\delta\Round(\frac{y_j}{\delta})\right\}\right| \\
    &= L\delta^2 \left|\min_{1 \le j \le d}\{\frac{1}{a_jL\delta} \cdot \frac{y_j}{\delta}\} - \min_{1 \le j \le d}\left\{b_j\cdot \Round(\frac{y_j}{\delta})\right\}\right| \\
    &\le L\delta^2 \cdot 2m \\
    &= 2L\delta
\end{align*}
\end{proof}
    
We will construct a 
$m$-approximate basis for $\cG_\Leontief^d$, and show that it is also a 
$3L\delta$-approximate basis for the Leontief utility functions $\cF_\Leontief^d$. 


As a starting point, any $g \in \cG_\Leontief^d$ can by straightforwardly represented by point indicator functions:
\begin{align*}
    g(z_1,\cdots,z_d) &= \sum_{i_1,\cdots,i_d \in \{1,\cdots,m-1\}} g(i_1,\cdots,i_d) \1[z_1=i_1,\cdots,z_d=i_d]
\end{align*}

Now we represent point indicator functions using MReLU functions. For any thresholds $i_1,\cdots,i_d \in \bZ$, the function $\MReLU_{i_1,\cdots,i_d}: \{1,\cdots,m-1\}^d \to \bZ$ is defined as: 
\[
\MReLU_{i_1,\cdots,i_d}(z_1,\cdots,z_d) = \max\{z_1-i_1, \cdots, z_d-i_d, 0\}
\]
MReLU functions are 1-Lipschitz due to the 1-Lipschitzness of the $\max$ function. When $d=1$, this reduces to the well known $\ReLU$ function, $\ReLU_i(z) = \max\{z-i,0\}$. 

Any indicator function $\1[z_1=i_1,\cdots,z_d=i_d]$ is a linear combination of MReLU functions whose thresholds are adjacent to $(i_1,\cdots,i_d)$. 

\begin{lemma}\label{lem:indicator-to-relu}
Suppose that $z_1,\cdots,z_d \in \{1,\cdots,m-1\}$, $i_1,\cdots,i_d \in \{1,\cdots,m-1\}$, we have:
\begin{align*}
    \1[z_1=i_1,\cdots,z_d=i_d] &= \sum_{\sigma_1,\cdots,\sigma_d \in \{0,1\}} (-1)^{d+\sigma_1+\cdots+\sigma_d} \MReLU_{i_1+\sigma_1,\cdots,i_d+\sigma_d}(z_1,\cdots,z_d) \\
    &\hspace{1em}+ \sum_{\sigma_1,\cdots,\sigma_d \in \{0,1\}} (-1)^{1+\sigma_1+\cdots+\sigma_d} \MReLU_{i_1-\sigma_1,\cdots,i_d-\sigma_d}(z_1,\cdots,z_d)
\end{align*}
\end{lemma}
For example, when $d=1$, Lemma \ref{lem:indicator-to-relu} gives us:
\begin{align*}
    \1[z=i] = \ReLU_{i+1}(z) - 2\ReLU_{i}(z) + \ReLU_{i-1}(z)
\end{align*}
And when $d=2$, we have:
\begin{align*}
    \1[z_1=i_1,z_2=i_2] &= \MReLU_{i_1+1,i_2+1}(z_1,z_2) - \MReLU_{i_1+1,i_2}(z_1,z_2) - \MReLU_{i_1,i_2+1}(z_1,z_2) \\ &\hspace{1em}+ \MReLU_{i_1-1,i_2}(z_1,z_2) + \MReLU_{i_1,i_2-1}(z_1,z_2) - \MReLU_{i_1-1,i_2-1}(z_1,z_2)
\end{align*}

Plugging Lemma \ref{lem:indicator-to-relu} into the above expression for $g$, we derive that:
\begin{align*}
    g(z_1,\cdots,z_d) &= \sum_{i_1,\cdots,i_d \in \{1,\cdots,m-1\}} g(i_1,\cdots,i_d) \Bigg[ \sum_{\sigma_1,\cdots,\sigma_d \in \{0,1\}} (-1)^{d+\sigma_1+\cdots+\sigma_d} \MReLU_{i_1+\sigma_1,\cdots,i_d+\sigma_d}(z_1,\cdots,z_d) \\
    &\hspace{13em}+ \sum_{\sigma_1,\cdots,\sigma_d \in \{0,1\}} (-1)^{1+\sigma_1+\cdots+\sigma_d} \MReLU_{i_1-\sigma_1,\cdots,i_d-\sigma_d}(z_1,\cdots,z_d) \Bigg]
\end{align*}

Rearranging terms, we derive a representation for $g$ based on $\MReLU$ functions.

\begin{corollary}\label{cor:representation-via-relu}
    For any $g \in \cG_\Leontief^d$, $z_1,\cdots,z_d \in \{1,\cdots,m-1\}$, we have
    \begin{align*}
        g(z_1,\cdots,z_d) = \sum_{i_1,\cdots,i_d \in \{0,\cdots,m\}} c_{i_1,\cdots,i_d}^g \MReLU_{i_1,\cdots,i_d}(z_1,\cdots,z_d)
    \end{align*}
    where
    \begin{align*}
        c_{i_1,\cdots,i_d}^g &= \sum_{\substack{\sigma_1,\cdots,\sigma_d \in \{0,1\} :\\ i_1-\sigma_1,\cdots,i_d-\sigma_d \in \{1,\cdots,m-1\}}} (-1)^{d+\sigma_1+\cdots+\sigma_d} g(i_1-\sigma_1,\cdots,i_d-\sigma_d) \\
        &\hspace{1em}+ \sum_{\substack{\sigma_1,\cdots,\sigma_d \in \{0,1\}:\\ i_1+\sigma_1,\cdots,i_d+\sigma_d \in \{1,\cdots,m-1\}}} (-1)^{1+\sigma_1+\cdots+\sigma_d} g(i_1+\sigma_1,\cdots,i_d+\sigma_d) \\
        &= \sum_{\substack{\sigma_1,\cdots,\sigma_d \in \{0,1\} :\\ i_1-\sigma_1,\cdots,i_d-\sigma_d \in \{1,\cdots,m-1\}}} (-1)^{d+\sigma_1+\cdots+\sigma_d} \min\{b_1(i_1-\sigma_1),\cdots,b_d(i_d-\sigma_d)\} \\
        &\hspace{1em}+ \sum_{\substack{\sigma_1,\cdots,\sigma_d \in \{0,1\}:\\ i_1+\sigma_1,\cdots,i_d+\sigma_d \in \{1,\cdots,m-1\}}} (-1)^{1+\sigma_1+\cdots+\sigma_d} \min\{b_1(i_1+\sigma_1),\cdots,b_d(i_d+\sigma_d)\}
    \end{align*}

\end{corollary} 

Furthermore, the sum of the absolute values of the coefficients $c_{i_1,\cdots,i_d}^g$ in the representation above is upper bounded by $12^{d+1}m^3$.

\begin{lemma}\label{lem:mrelu-coeffs}
For $c_{i_1,\cdots,i_d}^g$ as defined above, we have:
\begin{align*}
    \sum_{i_1,\cdots,i_d \in \{0,\cdots,m\}} \left|c_{i_1,\cdots,i_d}^g\right| \le 12^{d+1}m^3
\end{align*}
\end{lemma}

Thus, we have a basis consisting of MReLU functions. However, this basis has size $(m+1)^d$, which we can already obtain for all Lipschitz functions. 
With some more work, we can find a smaller basis for MReLU functions. To do this, we first turn to representing differences between adjacent MReLU functions --- i.e. MReLU functions with thresholds that differ by 1 in one coordinate. 

Let the interval function $\bI_{a,b}(z) = \1[z \in [a,b]]$ be the indicator for the interval $[a,b]$ in one dimension. Observe that for $d=1$,
\begin{align*}
    \ReLU_i(z)-\ReLU_{i+1}(z) = \bI_{i+1, m}(z)
\end{align*}
As an extension to higher dimensions, we define \textit{hyperrectangle functions}, i.e. indicator functions for hyperrectangles, to be products of interval functions over a multi-dimensional space. For instance, $\bI_{1,m}(z_1)\bI_{1,m}(z_1-z_2)\bI_{1,m}(z_1-z_3)$ corresponds to a hyperrectangle in the space of $(z_1,z_1-z_2,z_1-z_3)$. The following lemma states the difference between adjacent MReLU functions over higher dimensions as a hyperrectangle function.

\begin{lemma}\label{lem:relu-to-hyperrectangle}
Suppose that $z_1,\cdots,z_d \in \{1,\cdots,m\}$, $i_1,\cdots,i_d \in \{0,\cdots,m\}$. For $d \ge 2$, we have:
\begin{align*}
    \MReLU_{i_1,i_2,\cdots,i_d}(z_1,\cdots,z_d) - \MReLU_{i_1+1,i_2,\cdots,i_d}(z_1,\cdots,z_d) = \bI_{i_1+1,m}(z_1)\prod_{j=2}^d \bI_{i_1-i_j+1,m}(z_1-z_j)
\end{align*}
This can be equivalently stated as:
\begin{align*}
    \MReLU_{i_1,i_2,\cdots,i_d}(z_1,\cdots,z_d) - \MReLU_{i_1+1,i_2,\cdots,i_d}(z_1,\cdots,z_d) = \bI_{i_1+1,m}(z_1)\prod_{j=2}^d \bI_{i_1-i_j+1+m,2m}(z_1-z_j+m)
\end{align*}
where all the intervals are in $[1,\cdots,2m]$ for convenience.
\end{lemma}

Ultimately, this will allow us to discretize our basis of MReLU functions, while filling in the remaining gaps using hyperrectangle functions. More specifically, consider a set of discrete MReLU functions:
\[
\mathcal{R}_s = \{\MReLU_{i_1s,\cdots,i_ds}(z_1,\cdots,z_d) : i_1,\cdots,i_d \in \{1,\cdots,m/s\}\}
\]
for some discretization parameter $s$. Observe that the size of the discretized set is $(m/s)^d$.
By Lemma \ref{lem:relu-to-hyperrectangle}, any MReLU function that is not included in this set can simply be represented as a sum of hyperrectangle functions that ``step" from one MReLU function to another --- this requires at most $d \cdot s$ steps. As a result, a basis for MReLU functions will suffice to include a discretized subset \textit{and} hyperrectangle functions. Later, we will determine $s$ to strike a balance between the number of discrete MReLU functions we include in the basis, and the error coming from approximating hyperrectangle functions.

Our task thus reduces to constructing an approximate basis for
hyperrectangle functions on $\{1,\cdots,2m\}^d$, i.e., the set $$\left\{\prod_{j=1}^d \bI_{a_j,b_j}(w_j) : a_j,b_j,w_j \in \{1,\cdots,2m\}\right\}$$ We will use the set of \textit{dyadic} hyperrectangles. In 1 dimensions, these are intervals of the form $[a2^h,(a+1)2^h-1]$ where $a,h$ are integer. In $d$ dimensions, these are products of dyadic intervals: $$\prod_{j=1}^d [a_j2^{h_j},(a_j+1)2^{h_j}-1]$$ We refer to indicator functions for dyadic hyperrectangles as dyadic hyperrectangle functions. Notice that every interval on $\{1,\cdots,2m\}$ can be written as the disjoint union of at most $2\log_2 m$ dyadic intervals, which gives us the following: 

\begin{lemma}
    Every hyperrectangle on $\{1,\cdots,2m\}^d$ can be written as the disjoint union of at most $2^d\log_2^d m$ dyadic hyperrectangles. This quantity can be re-written as $c^d\ln^d m$ for a universal constant $c$.
\end{lemma}

In the next step, we will view indicators of dyadic hyperrectangles as standard basis vectors of dimension $N$ --- i.e. columns of the $N\times N$ identity matrix --- where $N$ is essentially the number of dyadic hyperrectangles with certain edge lengths (see Appendix \ref{app:dyadic-hyperrectangle-basis} for a more detailed discussion of the construction). We will use the following lemma to obtain a low-rank approximate factorization of the identity matrix, and consequently an approximate basis for dyadic hyperrectangles. 

\begin{lemma}\citep{alon2009perturbed}\label{lem:perturbed-identity-upper}
    There exists $c>0$ such that the following holds for all $\mu>0$ and $N$, there exists an $N \times K$ matrix $V$ where $K=c \ln N / \mu^2$ such that:
    $$
    \left(V V^T\right)_{i j}= \begin{cases}1 & \text { if } i=j \\ \leq \mu & \text { if } i \neq j\end{cases}
    $$
\end{lemma}

\begin{lemma}\label{lem:dyadic-hyperrectangle-basis}
    For any $h_1,\cdots,h_d \in \{0,\cdots,\ln(2m)\}$, let $N(h_1,\cdots,h_d) = \prod_{j=1}^d\frac{2m}{2^{h_j}}$, the number of dyadic hyperrectangles with edge lengths $2^{h_1},...,2^{h_d}$. Let $\mu = \frac{\delta^2}{12^{d+1}dm^{d/(d+2)}2^d\ln^d m}$. Let $V^{(h_1,\cdots,h_d)}$ denote the matrix given by Lemma \ref{lem:perturbed-identity-upper} with $N(h_1,\cdots,h_d)$ rows and $K(h_1,\cdots,h_d)$ columns. Let $V^{(h_1,\cdots,h_d)}_{ik}$ denote its element on the $i$-th row and $k$-th column. For $k \in \{1,\cdots,K(h_1,\cdots,h_d)\}$, define the function $\nu^{(h_1,\cdots,h_d)}_k: \{1,\cdots,2m\}^d \to [-1,1]$ as:
\begin{align*}
    \nu_k^{(h_1,\cdots,h_d)}(w_1,\cdots,w_d) = V^{(h_1,\cdots,h_d)}_{ik} \text{  where  } i \text{  satisfies that  } (w_1,\cdots,w_d) \in \text{HyperRec}^{(h_1,\cdots,h_d)}_i
\end{align*}
and define their shifted version $\tilde\nu^{(h_1,\cdots,h_d)}_k: \{1,\cdots,2m\}^d \to [0,1]$ as:
\begin{align*}
    \tilde\nu^{(h_1,\cdots,h_d)}_k(w_1,\cdots,w_d) = \frac{1+\nu(w_1,\cdots,w_d)}{2}
\end{align*}
(we require basis functions with range $[0,1]$). Then, 
$$\mathcal{W} = \cup_{h_1,\cdots,h_d \in \{0,\cdots,\ln(2m)\}} W^{(h_1,\cdots,h_d)}$$ is a $(n, \lambda, \mu)$-approximate basis for all dyadic hyperrectangle functions, where 
$$n = O\left( d^3m^{(6d+8)/(d+2)}24^{2d}\ln^{3d+1}m \right)$$ and $$\lambda = O\left( d^{3/2}m^{(3d+4)/(d+2)}24^{d}\ln^{d+1/2}m \right)$$
\end{lemma}

Putting everything together, we have that the set $\mathcal{W}$ combined with $\mathcal{R}_s$ (the set of discrete MReLU functions) gives us the desired approximate basis for MReLU functions, and thus Leontief functions. We present the complete proof in Appendix \ref{app:leontif-prop}.

}
\ifarxiv
\proofthmbasisLeontief
\else
\fi

\section{Decision Swap Regret for Non-Linear Losses}\label{sec:convex}

In this section we present our main result---an algorithm making predictions that guarantee low decision swap regret for losses that admit approximate basis representations. The bulk of the work in proving Theorem \ref{thm:decisionreg-convex} has already been completed; what remains is to combine results established in previous sections. First, we rely on a basis $\cS$ to (approximately) represent a loss $\ell$ by a higher-dimensional linear loss  $\hat{\ell}\in\hat{\cL}$. Because $\hat{\ell}$ is linear in the higher-dimensional representation $s(y)$, we can use the \textsc{Decision-Swap} algorithm to make predictions of $s(y)$. Then, best responding according to $\hat{\ell}$ guarantees vanishing decision swap regret under $\hat{\ell}$. But, $\hat{\ell}$ pointwise approximates $\ell$, and so for both the actual loss suffered and the loss achieved by the benchmark policy, we can switch $\ell$ and $\hat{\ell}$ without much change in loss, implying vanishing decision swap regret under $\ell$.
Now, recall that the guarantees of \textsc{Decision-Swap} depend on the size of the loss family and thus only apply to finite families of losses. We will thus use a net-based argument: rather than running \textsc{Decision-Swap} on $\hat{\cL}$, which might have infinite size, we run \textsc{Decision-Swap} on a $\gamma$-\textit{cover} of $\hat{\cL}$, which has finite size and uniformly approximates any $\hat{\ell}\in\hat{\cL}$ up to $\gamma$. 

\begin{definition}[$\gamma$-Cover]
    Let $\cL$ be a family of loss functions $\ell:\cA\times\cY\to[0,1]$. We say that $\cL_\gamma$ is a $\gamma$-cover of $\cL$ if for every $\ell\in\cL$, there exists $\ell_\gamma\in\cL_\gamma$ such that for every $a\in\cA, y\in\cY$, $|\ell(a, y) - \ell_\gamma(a, y)|\leq \gamma$.
\end{definition}

We discretize linear losses by ``snapping" their coefficients to the closest discrete point. 

\begin{lemma}[Discretization of Linear Losses]\label{lem:discretization-linear}
    Consider $\cY = [0,1]^d$. Let $\cL$ be a family of losses $\ell:\cA\times\cY\to[0,1]$ that are linear and $\lambda$-Lipschitz (in the $L_\infty$ norm) in the second argument. 
    Then there exists a family of losses $\cL_\gamma$ that is a $(d\gamma)$-cover of $\cL$ with $|\cL_\gamma| \leq (2\lambda/\gamma)^{d|\cA|}$.
\end{lemma}
\newcommand{\discretizationlemma}{
\begin{proof}
    Recall that for any $\ell\in\cL$ and any $a\in\cA$, we can write $\ell(a, y) = \sum_{i=1}^d r^\ell_i(a) \cdot y_i$ for some collection of coefficients $\{r_i^\ell\}_{i\in[d]}$. Moreover, by $\lambda$-Lipschitzness of $\ell$ and the fact that the range is bounded between $[0,1]^d$, we have that $|r_i^\ell(a)|\leq \lambda$ for all $i\in[d]$. So, for any $\ell$, we can construct a discretized loss $\ell_\gamma(a, y) = \sum_{i=1}^d r^\ell_{\gamma,i}(a)\cdot y_i$, where each $r^\ell_{\gamma,i}(a)$ is obtained via rounding each coefficient $r^\ell_i(a)$ to the nearest multiple of $\gamma$, e.g. the nearest value in $\{-\left\lfloor{\lambda/\gamma}\right\rfloor \gamma,...,-\gamma,0,\gamma,...,\left\lfloor{\lambda/\gamma}\right\rfloor \gamma\}$. Thus since $|r^\ell_i(a)-r^\ell_{\gamma,i}(a)|\leq \gamma$ for all $i\in[d], a\in\cA$, we have that:
    \begin{align*}
       |\ell(a, y) - \ell_\gamma(a, y)| &= \left|\sum_{i=1}^d r^\ell_i(a) \cdot y_i - \sum_{i=1}^d r^\ell_{\gamma,i}(a)\cdot y_i \right| \\ &= \left|\sum_{i=1}^d (r^\ell_i(a) - r^\ell_{\gamma,i}(a)) \cdot y_i \right| \\ &\leq \sum_{i=1}^d \left|r^\ell_i(a) - r^\ell_{\gamma,i}(a)\right| \cdot y_i \\ &\leq d\gamma 
    \end{align*}
    where in the last inequality we use the fact that each $y_i\in [0,1]$. Thus, $\cL_\gamma = \{\ell_\gamma\}_{\ell\in\cL}$ is a $(d\gamma)$-cover of $\cL$. Finally, since each $r^\ell_{\gamma,i}(a)$ take one of at most $2\lambda/\gamma$ values, we have that $|\cL_\gamma| \leq (2\lambda/\gamma)^{d|\cA|}$.
\end{proof}
}
\ifarxiv
\discretizationlemma
\fi

The exponential dependence on $|\cA|$ and $d$ will not be prohibitive---remember that \textsc{Decision-Swap} gives error bounds that depend only logarithmically on the size of the loss family. Next we state the result.

\begin{algorithm}[H]
    \KwIn{Family of losses $\cL=\{\ell:\cA\times[0,1]^d\to[0,1]\}$, $(n,\gamma, \delta)$-approximate basis $\cS$ for $\cL$, collection of policies $\cC$, \textsc{Decision-Swap} algorithm}
    \KwOut{Sequence of predictions $\hat{p}_1,...,\hat{p}_T \in [0,1]^n$} 
    \vspace{.5em}

    Set $\gamma=\frac{1}{2n\sqrt{T}}$
    
    Construct $\hat{\cL}=\{\hat{\ell}: \cA\times[0,1]^n\to[0,1]\}_{\ell\in\cL}$, the corresponding family of linear losses given by the basis $\cS$
    
    Construct $\hat{\cL}_\gamma$, the $(n\gamma)$-cover of $\hat{\cL}$ given by Lemma \ref{lem:discretization-linear}\;

    Instantiate a copy of $\textsc{Decision-Swap}$ with loss family $\hat{\cL}_\gamma$ and collection of policies $\cC$\;

    \For{$t=1$ \KwTo $T$}{
        Receive $x_t$\;
        
        Let $\pi_t = \textsc{Decision-Swap}_t(\{x_r\}_{r=1}^{t}, \{s(y_r)\}_{r=1}^{t-1})$, the distribution over predictions output by \textsc{Decision-Swap} on round $t$ given contexts $\{x_r\}_{r=1}^{t}$ and outcomes $\{s(y_r)\}_{r=1}^{t-1})$\;

        Predict $\hat{p}_t \sim \pi_t$\;
        
        Observe $y_t$\;
    }
    
    \caption{Decision Swap Regret for Non-Linear Losses}
    \label{alg:convex}
\end{algorithm}

\begin{theorem}\label{thm:decisionreg-convex}
    Let $\cY=[0,1]^d$ be the outcome space and $\cC$ be a collection of policies $c:\cX\to\cA$. Let $\cL$ be a family of loss functions $\ell:\cA\times\cY\to[0,1]$. Suppose $\mathcal{S}$ is a $(n,\lambda,\delta)$-approximate basis for $\cL$. Let $\hat{\cL} = \{\hat{\ell}\}_{\ell\in\cL}$ be the family of linear approximations to $\cL$ given by $\mathcal{S}$. Let $\hat{\cL}_\gamma$ be the $(n\gamma)$-cover of $\hat{\cL}$ given by Lemma \ref{lem:discretization-linear}, for $\gamma = \frac{1}{2n\sqrt{T}}$. Algorithm \ref{alg:convex} produces predictions $\hat{p}_1,...,\hat{p}_T \in [0,1]^n$ that has $(\cL,\cC,\eps)$-decision swap regret for 
    \[
    \eps \leq O\left(\lambda |\cA| \sqrt{\frac{n|\cA|\ln(n\lambda |\cA||\cC|T)}{T}} \right)  + 2\delta
    \]
    when choosing actions $a_t = \BR^{\hat{\ell}_\gamma}(\hat{p}_t)$ for a nearby $\hat{\ell}_\gamma\in\hat{\cL}_\gamma$.
\end{theorem}
\newcommand{\decisionregconvexthm}{


    Fix any $\ell\in\cL$. By $(n,\lambda,\delta)$-approximation, there is an $\hat{\ell}\in\hat{\cL}$ such that $\hat{\ell}_\gamma$ is a linear and $\lambda$-Lipschitz function of $[0,1]^n$, and for any $a\in\cA, y\in\cY$, $\left| \hat{\ell}(a, s(y)) - \ell(a, y) \right| \leq \delta$. Furthermore, since $\hat{\cL}_\gamma$ is a $(n\gamma)$-cover of $\hat{\cL}$, there exists $\hat{\ell}_\gamma\in\hat{\cL}_\gamma$ satisfying $\left| \hat{\ell}_\gamma(a, s(y)) - \hat{\ell}(a, s(y)) \right| \leq n\gamma$, and therefore satisfying:
    \[
    \left| \hat{\ell}_\gamma(a, s(y)) - \ell(a, y) \right| \leq \delta + n\gamma
    \]
    
    We now turn to analyzing the guarantees of the sequence of predictions $\hat{p}_1,...,\hat{p}_T \in [0,1]^n$ output by Algorithm \ref{alg:convex}. Fix any assignment of policies $\{c_a\}_{a\in\cA}$. We first analyze the decision swap regret as measured by $\hat{\ell}_\gamma$ when the agent chooses actions $a_t = \BR^{\hat{\ell}_\gamma}(\hat{p}_t)$. Since $\hat{\ell}_\gamma$ is linear in $[0,1]^n$, Corollary \ref{cor:decisionreg-linear} bounds the agent's decision swap regret---specifically, with respect to the sequence of outcomes under the transformation $s$. We have that:
    \begin{align*}
        \frac{1}{T} \sum_{t=1}^T \E_{\hat{p}_t\sim\pi_t}\left[\hat{\ell}_\gamma(a_t, s(y_t)) - \hat{\ell}_\gamma(c_{a_t}(x_t), s(y_t))\right]
        &\leq O\left(\lambda |\cA| \sqrt{\frac{\ln(n|\cA||\hat{\cL}_\gamma||\cC|T)}{T}} \right) \\
        &\leq O\left(\lambda |\cA| \sqrt{\frac{\ln(n|\cA||\cC|(\lambda/\gamma)^{n|\cA|}T)}{T}} \right) \\
        &\leq O\left(\lambda |\cA| \sqrt{\frac{n|\cA|\ln(n|\cA||\cC|(\lambda/\gamma)T)}{T}} \right)       
    \end{align*}
    The first inequality applies Corollary \ref{cor:decisionreg-linear}. The second inequality uses Lemma \ref{lem:discretization-linear} to establish that $|\hat{\cL}_\gamma| \leq (2\lambda/\gamma)^{n|\cA|}$.

    Using this, we can now derive the decision swap regret as measured by $\ell$. We simply substitute $\ell$ for $\hat{\ell}_\gamma$ and incur the approximation factor $\delta+n\gamma$ for both the sequence of realized losses and the benchmark sequence:
    \begin{align*}
        \frac{1}{T} \sum_{t=1}^T \E_{\hat{p}_t\sim\pi_t}\left[ \ell(a_t, y_t) - \ell(c_{a_t}(x_t), y_t) \right] 
        &\leq \frac{1}{T} \sum_{t=1}^T \E_{\hat{p}_t\sim\pi_t}\left[ \hat{\ell}_\gamma(a_t, s(y_t)) - \hat{\ell}_\gamma(c_{a_t}(x_t), s(y_t)) \right] + 2\delta + 2n\gamma \\
        &\leq O\left(\lambda |\cA| \sqrt{\frac{n|\cA|\ln(n|\cA||\cC|(\lambda/\gamma)T)}{T}} \right)  + 2\delta + 2n\gamma
    \end{align*}
    For our setting of $\gamma = \frac{1}{2n\sqrt{T}}$, this expression is at most:
    \[
    O\left(\lambda |\cA| \sqrt{\frac{n|\cA|\ln(n\lambda |\cA||\cC|T)}{T}} \right)  + 2\delta
    \]
    This proves the theorem.

}
\ifarxiv
\begin{proof}
  \decisionregconvexthm  
\end{proof}
\fi

\ifarxiv
We can now instantiate our decision swap regret guarantees for specific loss families. We will denote by $\cL_\cF$ the loss family corresponding to a function family $\cF$---i.e. $\cL_\cF=\{\ell(a,y)=f_a(y)\}$ for every $f_a\in\cF$.

\begin{corollary}\label{cor:regret-specific-loss}
    Let $\cY=[0,1]^d$ be the outcome space and $\cC$ be a collection of policies $c:\cX\to\cA$. Let $\cL$ be a family of loss functions $\ell:\cA\times\cY\to[0,1]$. Let $\hat{\cL} = \{\hat{\ell}\}_{\ell\in\cL}$ be the family of linear approximations to $\cL$ given by an $(n,\lambda,\delta)$-approximate basis $\mathcal{S}$. Let $\hat{\cL}_\gamma$ be the $(n\gamma)$-cover of $\hat{\cL}$ given by Lemma \ref{lem:discretization-linear}, for $\gamma = \frac{1}{2n\sqrt{T}}$. Suppose agents choose actions $a_t = \BR^{\hat{\ell}_\gamma}(\hat{p}_t)$ for a nearby $\hat{\ell}_\gamma\in\hat{\cL}_\gamma$. Then, for the loss families below, Algorithm \ref{alg:convex} produces predictions that has $(\cL,\cC,\eps)$-decision swap regret for the following $\eps$:
    \begin{itemize}
        \item Let $\cL = \cL_{\cF^1_{\text{cvx}}}$ be the family of convex, 1-Lipschitz functions over $\cY = [0,1]$. Then for $\delta = \frac{1}{T^{3/8}}$, 
        \[
        \eps \leq O\left(\frac{|\cA|\ln T \sqrt{|\cA|\ln(|\cA||\cC|T\ln T)}}{T^{3/8}} \right)
        \]

        \item Let $\cL = \cL_{\cF^d_L}$ be the family of $L$-Lipschitz losses. Then for $\delta=\frac{L^{d/(d+2)}}{2^{d/(d+2)}T^{1/(d+2)}}$, 
        \[
        \eps \leq O\left(\frac{|\cA| L^{d/(d+2)} \sqrt{d|\cA|\ln(L |\cA||\cC|T)}}{T^{1/(d+2)}} \right)
        \]

        \item Let $\cL = \cL_{\cF^d_p}$ be the family of $L_p$ losses. Then, 
        \[
        \eps \leq O\left(p^{p+2}d |\cA| \sqrt{\frac{d|\cA|\ln(pd|\cA||\cC|T)}{T}} \right)
        \]

        \item Let $\cL = \cL_{\cF^d_{\Omega_{\beta,g}}}$ be the family of loss functions that are monomials of degree $\beta$ over an underlying basis representation $g$. Then, 
        \[
        \eps \leq O\left(d^\beta |\cA| \sqrt{\frac{\beta d^\beta|\cA|\ln(d |\cA||\cC|T)}{T}} \right)
        \]
        
        \item Let $\cL = \cL_{\cF^d_{\Omega_{\text{exp},g}}}$ be the family of loss functions that are exponential functions over an underlying basis representation $g$. Then for $\delta = \frac{e^{2Rd/(\ln d+2)}}{T^{1/(\ln d+2)}}$, 
        \[
        \eps \leq O\left(\frac{|\cA| e^{Rd/(\ln d+2)} \sqrt{d|\cA|\ln(|\cA||\cC|T)}}{T^{1/(\ln d+2)}} \right)
        \]

        \item Let $\cL = \cL_{\cF^d_{\text{Leon}}}$ be the family of Leontif loss functions. Then, 
        for some universal constant $c$ and $\delta = O\left(\frac{dc^d}{T^{(d+2)/(18d+24)}}\right)$
        \[
        \eps \le O\left(\frac{|\cA| d^{3/2}c^{d}\sqrt{|\cA|\ln( Ld|\cA||\cC|T)}}{T^{(d+2)/(18d+24)}}\right)
        \]
    \end{itemize}
\end{corollary}
\fi

\newcommand{\seccontsactionspaces}{
The decision swap regret bounds we have derived depend on $|\cA|$, the number of actions. This is acceptable in games that have finite action spaces (as often studied in the ``downstream regret'' literature) --- but the omniprediction literature often takes $\cA = \cY$ --- i.e. actions correspond to (continuous) labe predictions. Here we extend our results to continuous but bounded action spaces $\cA = [0,1]^m$ (this is without loss, since taking $\cA$ to be any bounded space $[-C, C]^m$ will only introduce a constant factor to our analyses). Our results for finite action spaces translate to continuous action spaces as long as losses are additionally Lipschitz in the actions, and agents choose actions from a $\theta$-net of $\cA$. This has implications for decision tasks where the action space is the outcome space (e.g. omniprediction). For these tasks, if we assume losses to be Lipschitz in the outcomes, then it is generally reasonable to assume that they are also Lipschitz in the actions (most loss functions considered in machine learning are Lipschitz in both the prediction and the label --- e.g. squared loss satisfies this assumption). 

\begin{definition}[$\theta$-Net]
    Let $\cS$ be an arbitrary continuous space. We say that $\cS_\theta \subset \cS$ is a $\theta$-net of $\cS$ if for every $s\in\cS$, there exists $\tilde{s}\in\cS_\theta$ such that $\|s - \tilde{s}\|_\infty \leq \theta$. Observe that when $\cS = [0,1]^m$, we can always obtain a $\theta$-net of $\cS$ of size $|\cS_\theta| \leq 1/\theta^m$ by discretizing each coordinate into multiples of $\theta$. 
\end{definition}

The idea is straightforward. Given a family of losses $\cL = \{\ell:\cA \times \cY \to [0,1]\}$, we can construct a corresponding family of losses $\cL_\theta = \{ \ell|_{\cA_\theta}: \cA_\theta \times \cY \to [0,1]\}_{\ell\in\cL}$, where $\ell|_{\cA_\theta}$ is the restriction of $\ell$ to the a $\theta$-net $\cA_\theta$ of $\cA$. Since $\cA_\theta$ is finite, we can invoke our previous result for finite loss families (Theorem \ref{thm:decisionreg-convex}) on $\cL_\theta$, when agents choose actions from $\cA_\theta$. We can think of agents as ``snapping" their actions to a nearby action in $\cA_\theta$. There is just one complication: our previous result only lets us compare our performance to benchmark policies that suggest actions in $\cA_\theta$. However, we show that benchmark policies that instead suggest actions in $\cA$ are not much more powerful when losses are Lipschitz in the actions.


\begin{theorem}\label{thm:continuous-action-space}
    Let $\cA=[0,1]^m$ and $\cA_\theta$ be a $\theta$-net of $\cA$. Let $\cL$ be a family of losses $\ell:\cA\times\cY\to[0,1]$. Let $\cC$ be a collection of policies $c:\cX\to\cA$ and $\tilde{\cC} = \{\tilde{c}: \cX\to\cA_\theta\}_{c\in\cC}$ be the corresponding collection of policies with ``snapped" outputs. If a sequence of predictions $p_1,...,p_T$ has $(\cL_\theta, \tilde{\cC}, \eps')$ decision swap regret with respect to a sequence of outcomes $y_1,...,y_T\in\cY$ and actions $\tilde{a}_t,...,\tilde{a}_T\in\cA_\theta$, then it has $(\cL, \cC, \eps)$ decision swap regret for $\eps \leq \eps' + \theta$ with respect to the same sequence of outcomes and actions. 
\end{theorem}
\begin{proof}
    Fix any $\ell\in\cL$ and any assignment of benchmark policies $\{c_{\tilde{a}}\in\cC\}_{\tilde{a}\in\cA_\theta}$. On the same sequence of outcomes and actions, we observe that the actions suggested by $\{c_{\tilde{a}}\in\cC\}_{\tilde{a}\in\cA_\theta}$ do not perform too much better than the actions suggested by the corresponding benchmark policies $\{\tilde{c}_{\tilde{a}}\in\tilde{\cC}\}_{\tilde{a}\in\cA_\theta}$. Specifically, for all $t\in[T]$, we have that:
    \begin{align*}
        \ell(\tilde{c}_{\tilde{a}_t}(x_t), y_t) - \ell(c_{\tilde{a}}(x_t), y_t) \leq \|\tilde{c}_{\tilde{a}_t}(x_t) - c_{\tilde{a}_t}(x_t)\|_\infty \leq \theta
    \end{align*}
    where the first inequality follows by 1-Lipschitzness in the first argument, and the second inequality follows by the fact that for any $c\in\cC$, $\tilde{c}\in\tilde{\cC}$ snaps its output to the nearest point in the $\theta$-net. 

    Given this observation, we can compute the decision swap regret incurred by $\ell$ against an assignment of benchmark policies in $\cC$. For all $t\in[T]$, since $\tilde{a}_t, \tilde{c}_{\tilde{a}_t}(x_t) \in \cA_\theta$, we have that $\ell(\tilde{a}_t, y_t) = \ell|_{\cA_\theta}(\tilde{a}_t, y_t)$ and $\ell(\tilde{c}_{\tilde{a}_t}(x_t), y_t) = \ell|_{\cA_\theta}(\tilde{c}_{\tilde{a}_t}(x_t), y_t)$. Thus, as a consequence of the earlier observation and $(\cL_\theta, \tilde{\cC}, \eps')$ decision swap regret, we have:
    \begin{align*}
        \frac{1}{T} \sum_{t=1}^T (\ell(\tilde{a}_t, y_t) - \ell(c_{\tilde{a}_t}(x_t), y_t))
        \leq \frac{1}{T} \sum_{t=1}^T (\ell|_{\cA_\theta}(\tilde{a}_t, y_t) - \ell|_{\cA_\theta}(\tilde{c}_{\tilde{a}_t}(x_t), y_t)) + \theta
        \leq \eps' + \theta
    \end{align*}
    which proves the theorem.
\end{proof}

Invoking Theorem \ref{thm:decisionreg-convex}, this allows us to state decision swap regret guarantees for agents who choose nearby actions in $\cA_\theta$.

\begin{corollary}\label{cor:decisionreg-convex-continuous-action}
    Let $\cA=[0,1]^m$ be the action space and $\cY=[0,1]^d$ be the outcome space. Let $\cL$ be a family of loss functions $\ell:\cA\times\cY\to[0,1]$. Suppose $\mathcal{S}$ is a $(n,\lambda,\delta)$-approximate basis for $\cL$. Let $\hat{\cL} = \{\hat{\ell}\}_{\ell\in\cL}$ be the family of linear approximations to $\cL$ given by $\mathcal{S}$. Let $\hat{\cL}_\gamma$ be the $(n\gamma)$-cover of $\hat{\cL}$ given by Lemma \ref{lem:discretization-linear}, for $\gamma = \frac{1}{2n\sqrt{T}}$. Let $\cA_\theta$ be a $\theta$-net of $\cA$ for $\theta = \left( \frac{\lambda^2 n}{T} \right)^{1/(3m+2)}$. Let $\cL_\theta = \{ \ell|_{\cA_\theta}: \cA_\theta \times \cY \to [0,1]\}_{\ell\in\cL}$ be the restriction of the loss functions to $\cA_\theta$. Let $\cC$ be a collection of policies $c:\cX\to\cA$ and $\tilde{\cC} = \{\tilde{c}: \cX\to\cA_\theta\}_{c\in\cC}$ be the corresponding collection of policies with ``snapped" outputs. Running Algorithm \ref{alg:convex} with inputs $\cL_\theta$, $\cS$ and $\tilde{\cC}$ produces predictions $\hat{p}_1,...,\hat{p}_T \in [0,1]^n$ that has $(\cL,\cC,\eps)$-decision swap regret for:
    \[
        \eps \leq O\left( \frac{\lambda^{2/(3m+2)} n^{1/(3m+2)} \sqrt{\ln(n\lambda|\cC|T)}}{T^{1/(3m+2)}} \right)  + 2\delta
    \]
    when choosing actions $\tilde{a}_t \in\cA_\theta$ that is the closest point to $\BR^{\hat{\ell}_\gamma}(\hat{p}_t)$ for a nearby $\hat{\ell}_\gamma\in\hat{\cL}_\gamma$. When the action space is $m = 1$ dimensional (the standard setting for omniprediction), this gives a bound of:
        \[
        \eps \leq O\left( \frac{\lambda^{2/5} n^{1/5} \sqrt{\ln(n\lambda|\cC|T)}}{T^{1/5}} \right)  + 2\delta
    \]
\end{corollary}
\begin{proof}
    By Theorem \ref{thm:decisionreg-convex}, running Algorithm \ref{alg:convex} with inputs $\cL_\theta$, $\cS$ and $\tilde{\cC}$ produces predictions $\hat{p}_1,...,\hat{p}_T$ that has $(\cL_\theta, \tilde{\cC}, \eps')$-decision swap regret for:
    \begin{align*}
        \eps' &\leq O\left(\lambda |\cA_\theta| \sqrt{\frac{n|\cA_\theta|\ln(n\lambda |\cA_\theta||\cC|T)}{T}} \right) + 2\delta \\
        &\le O\left(\lambda \left(\frac{1}{\theta^m}\right)^{\frac{3}{2}} \sqrt{\frac{n\ln(n\lambda|\cC|T/\theta^m)}{T}} \right) + 2\delta
    \end{align*}

    By Theorem \ref{thm:continuous-action-space}, the sequence of predictions $\hat{p}_1,...,\hat{p}_T$ has $(\cL,\cC,\eps)$-decision swap regret for:
    \begin{align*}
        \eps &\le \eps' + \theta \\
        &\le O\left(\lambda \left(\frac{1}{\theta^m}\right)^{\frac{3}{2}} \sqrt{\frac{n\ln(n\lambda|\cC|T/\theta^m)}{T}} \right) + 2\delta + \theta
    \end{align*}

    For our setting of $\theta = \left(\frac{\lambda^2 n}{T}\right)^{1/(3m+2)}$, this expression is at most:
    \[
        O\left( \frac{\lambda^{2/(3m+2)} n^{1/(3m+2)} \sqrt{\ln(n\lambda|\cC|T)}}{T^{1/(3m+2)}} \right)  + 2\delta
    \]
    
\end{proof}

}

\newcommand{\seconlinetobatch}{
We use an online-to-batch reduction to obtain guarantees for the offline version of the problem, in which features and labels are drawn i.i.d. from a fixed joint distribution $\cD$, rather than selected sequentially by an adversary. In the reduction, we will feed a sequence of i.i.d. draws from $\cD$ to Algorithm \ref{alg:convex}, our algorithm for the online adversarial setting, to produce prediction functions $p_t:\cX\to\cP$ mapping contexts to predictions (we note that Algorithm \ref{alg:convex} produces an implicit prediction function at every round). Our (randomized) predictor will simply be the uniform mixture of these prediction functions. 

\begin{theorem}\label{thm:batch}
    Let $\cY=[0,1]^d$ be the outcome space and $\cC$ be a collection of policies $c:\cX\to\cA$. Let $\cL$ be a family of loss functions $\ell:\cA\times\cY\to[0,1]$. Suppose $\mathcal{S}$ is a $(n,\lambda,\delta)$-approximate basis for $\cL$. Let $\hat{\cL} = \{\hat{\ell}\}_{\ell\in\cL}$ be the family of linear approximations to $\cL$ given by $\mathcal{S}$. Let $\hat{\cL}_\gamma$ be the $(n\gamma)$-cover of $\hat{\cL}$ given by Lemma \ref{lem:discretization-linear}, for $\gamma = \frac{1}{2n\sqrt{T}}$. We run Algorithm \ref{alg:convex} with a sample $D=\left\{\left(x_t, y_t\right)\right\}_{t=1}^T$ that is drawn i.i.d. from $\cD$ to obtain a sequence of prediction functions $\left\{p_t\right\}_{t=1}^T$, where $\left\{p_t(x_t)\right\}_{t=1}^T$ has $(\cL,\cC,\eps)$-decision swap regret with respect to $\{y_t\}_{t=1}^T$. From this, we construct a single randomized predictor $\pi: \cX \to \Delta \cP$, where $\pi(x)$ distributes uniformly on $\left\{p_t\right(x)\}_{t=1}^T$. With probability $1-\kappa$, the randomized predictor $\pi$ satisfies that, for any loss function $\ell \in \cL$ and any assignment of policies $\{c_a \in \cC\}_{a \in \cA}$:
    \[
    \E_{\substack{(x,y) \sim \cD \\ p \sim \pi(x)}} \left[ \ell(k^\ell(p), y) - \ell(c_{k^\ell(p)}(x), y) \right] 
    \leq \eps + O\left( \sqrt{\frac{n|\cA|\ln(n \lambda |\cC| T / \kappa)}{T}} \right)
    + 4\delta
    \]
    when the action selection rule is $k^\ell(p) = \BR^{\hat{\ell}_\gamma}(p)$ for a nearby $\hat{\ell}_\gamma\in\hat{\cL}_\gamma$.
\end{theorem}

\begin{proof}
Fix a  loss function $\ell \in \cL$ and any assignment of policies $\{c_a \in \cC\}_{a \in \cA}$. 
We note that
\begin{align*}
    \E_{(x,y) \sim \cD, p \sim \pi(x)} \left[ \ell(k^\ell(p), y) - \ell(c_{k^\ell(p)}(x), y) \right]
    &= \frac{1}{T} \sum_{t=1}^T \E_{(x,y) \sim \cD} \left[ \ell(k^\ell(p_t(x)), y) - \ell(c_{k^\ell(p_t(x))}(x), y) \right]
\end{align*}
is the distributional analogue of the empirical decision swap regret,
\begin{align*}
    \frac{1}{T} \sum_{t=1}^T \left( \ell(k^\ell(p_t(x_t)), y_t) - \ell(c_{k^\ell(p_t(x_t))}(x_t), y_t) \right)
\end{align*}
which is bounded by Theorem \ref{thm:decisionreg-convex}.

We will show that the distance between the empirical decision swap regret and the distributional version is bounded by bounding the following terms:
\begin{align*}
    & \left| \frac{1}{T}\sum_{t=1}^T \ell(k^\ell(p_t(x_t)), y_t) - \frac{1}{T}\sum_{t=1}^T \E_{(x,y) \sim \cD} \left[ \ell(k^\ell(p_t(x)), y) \right] \right| \\
    & \left| \frac{1}{T}\sum_{t=1}^T \ell(c_{k^\ell(p_t(x_t))}(x_t), y_t) - \frac{1}{T}\sum_{t=1}^T \E_{(x,y) \sim \cD} \left[ \ell(c_{k^\ell(p_t(x))}(x), y) \right] \right|
\end{align*}

\begin{lemma}\label{lem:online-to-batch-action}
    For any $\kappa>0$, with probability at least $1-\kappa/2$, for any loss function $\ell \in \cL$:
    \begin{align*}
     \left| \frac{1}{T}\sum_{t=1}^T \ell(k^\ell(p_t(x_t)), y_t) - \frac{1}{T}\sum_{t=1}^T \E_{(x,y) \sim \cD} \left[ \ell(k^\ell(p_t(x)), y) \right] \right| 
     \le O\left( \sqrt{\frac{n|\cA|\ln(n \lambda T / \kappa)}{T}} \right)
     + 2\delta
    \end{align*}
\end{lemma}

\begin{proof}
    We first prove that, with probability at least $1-\kappa/2$, for any $\hat{\ell}_\gamma \in \hat{\cL}_\gamma$:
    \begin{align*}
     \left| \frac{1}{T}\sum_{t=1}^T \hat{\ell}_\gamma(\BR^{\hat{\ell}_\gamma}(p_t(x_t)), y_t) - \frac{1}{T}\sum_{t=1}^T \E_{(x,y) \sim \cD} \left[ \hat{\ell}_\gamma(\BR^{\hat{\ell}_\gamma}(p_t(x)), y) \right] \right| 
     \le O\left( \sqrt{\frac{n|\cA|\ln(n \lambda T / \kappa)}{T}} \right)
    \end{align*}
    
    We define the filtration $\cF_t = \sigma\left(\left\{\left(x_s, y_s, p_s\right)\right\}_{s=1}^t\right)$. For a fixed $\hat{\ell}_\gamma \in \hat{\cL}_\gamma$, consider the following sequence adapted to $\cF_t$,
    \[
    W_t = W_{t-1} + \hat{\ell}_\gamma(\BR^{\hat{\ell}_\gamma}(p_t(x_t)), y_t) - \E_{(x,y) \sim \cD} \left[ \hat{\ell}_\gamma(\BR^{\hat{\ell}_\gamma}(p_t(x)), y) \right].
    \]
    
    Since $W_{t-1} \in \cF_{t-1}$ and
    \begin{align*}
        \E \left[ \hat{\ell}_\gamma(\BR^{\hat{\ell}_\gamma}(p_t(x_t)), y_t) \mid \cF_{t-1} \right] = \E \left[ \E_{(x,y) \sim \cD} \left[ \hat{\ell}_\gamma(\BR^{\hat{\ell}_\gamma}(p_t(x)), y) \right] \mid \cF_{t-1} \right]
    \end{align*}
    we have $\E[W_t \mid \cF_{t-1}] = W_{t-1}$, so $\{W_t\}_{t=1}^T$ is a martingale.
    
    We also have $|W_t-W_{t-1}| \le 1$, so by applying Azuma-Hoeffding’s inequality (Lemma \ref{lem:azuma}), we get that with probability $1-\frac{\kappa}{2|\hat{\cL}_\gamma|}$:
    \begin{align*}
     \left| \frac{1}{T} \sum_{t=1}^T \hat{\ell}_\gamma(\BR^{\hat{\ell}_\gamma}(p_t(x_t)), y_t) - \frac{1}{T} \sum_{t=1}^T \E_{(x,y) \sim \cD} \left[ \hat{\ell}_\gamma(\BR^{\hat{\ell}_\gamma}(p_t(x)), y) \right] \right| \le \sqrt{\frac{2\ln \left(2|\hat{\cL}_\gamma|/\kappa\right)}{T}}
     \le \sqrt{\frac{2\ln \left(2(2\lambda/\gamma)^{n|\cA|}/\kappa\right)}{T}}
    \end{align*}

    For our setting of $\gamma = \frac{1}{2n\sqrt{T}}$, this expression is at most:
    \[
         O\left( \sqrt{\frac{n|\cA|\ln(n \lambda T / \kappa)}{T}} \right)
    \]

    Taking the union bound over every $\hat{\ell}_\gamma \in \hat{\cL}_\gamma$, we get that with probability $1-\kappa/2$, for any $\hat{\ell}_\gamma \in \hat{\cL}_\gamma$:
    \[
        \left| \frac{1}{T} \sum_{t=1}^T \hat{\ell}_\gamma(\BR^{\hat{\ell}_\gamma}(p_t(x_t)), y_t) - \frac{1}{T} \sum_{t=1}^T \E_{(x,y) \sim \cD} \left[ \hat{\ell}_\gamma(\BR^{\hat{\ell}_\gamma}(p_t(x)), y) \right] \right| \le O\left( \sqrt{\frac{n|\cA|\ln(n \lambda T / \kappa)}{T}} \right)
    \]

    Since $\hat{\cL}_\gamma$ is the $(n\gamma)$-cover of $\hat{\cL}$ and $\hat{\cL}$ is an $\delta$-approximate basis for $\cL$, we get that when the above inequality holds, for any $\ell \in \cL$:
    \begin{align*}
        & \left| \sum_{t=1}^T \ell(k^{\ell}(p_t(x_t)), y_t) - \sum_{t=1}^T \E_{(x,y) \sim \cD} \left[ \ell(k^{\ell}(p_t(x)), y) \right] \right| \\
        &= \left| \sum_{t=1}^T \ell(\BR^{\hat{\ell}_\gamma}(p_t(x_t)), y_t) - \sum_{t=1}^T \E_{(x,y) \sim \cD} \left[ \ell(\BR^{\hat{\ell}_\gamma}(p_t(x)), y) \right] \right| \\
        &\le \left| \sum_{t=1}^T \ell(\BR^{\hat{\ell}_\gamma}(p_t(x_t)), y_t) - \sum_{t=1}^T \hat{\ell}_\gamma(\BR^{\hat{\ell}_\gamma}(p_t(x_t)), y_t) \right| \\
        &+ \left| \sum_{t=1}^T \E_{(x,y) \sim \cD} \left[ \hat{\ell}_\gamma(\BR^{\hat{\ell}_\gamma}(p_t(x)), y) \right] - \sum_{t=1}^T \E_{(x,y) \sim \cD} \left[ \ell(\BR^{\hat{\ell}_\gamma}(p_t(x)), y) \right] \right| \\
        &+ \left| \sum_{t=1}^T \hat{\ell}_\gamma(\BR^{\hat{\ell}_\gamma}(p_t(x_t)), y_t) - \sum_{t=1}^T \E_{(x,y) \sim \cD} \left[ \hat{\ell}_\gamma(\BR^{\hat{\ell}_\gamma}(p_t(x)), y) \right] \right| \\
        &\le (n\gamma + \delta) + (n\gamma + \delta) + O\left( \sqrt{\frac{n|\cA|\ln(n \lambda T / \kappa)}{T}} \right) \\
        &= O\left( \sqrt{\frac{n|\cA|\ln(n \lambda T / \kappa)}{T}} \right) + 2\delta
    \end{align*}
    where we use the triangle inequality.
    
\end{proof}

\begin{lemma}\label{lem:online-to-batch-benchmark}
    For any $\kappa>0$, with probability at least $1-\kappa/2$, for any loss function $\ell \in \cL$ and any assignment of policies $\{c_a \in \cC\}_{a\in\cA}$:
    \begin{align*}
     \left| \frac{1}{T}\sum_{t=1}^T \ell(c_{k^\ell(p_t(x_t))}(x_t), y_t) - \frac{1}{T}\sum_{t=1}^T \E_{(x,y) \sim \cD} \left[ \ell(c_{k^\ell(p_t(x))}(x), y) \right] \right| \le O\left( \sqrt{\frac{n|\cA|\ln(n \lambda |\cC| T / \kappa)}{T}} \right) + 2\delta
    \end{align*}
\end{lemma}

\begin{proof}
    The proof is similar to the proof of Lemma \ref{lem:online-to-batch-action}. We first prove that, with probability at least $1-\kappa/2$, for any $\hat{\ell}_\gamma \in \hat{\cL}_\gamma$ and any assignment of policies $\{c_a \in \cC\}_{a\in\cA}$:
    \begin{align*}
     \left| \frac{1}{T}\sum_{t=1}^T \hat{\ell}_\gamma(c_{\BR^{\hat{\ell}_\gamma}(p_t(x_t))}(x_t), y_t) - \frac{1}{T}\sum_{t=1}^T \E_{(x,y) \sim \cD} \left[ \hat{\ell}_\gamma(c_{\BR^{\hat{\ell}_\gamma}(p_t(x))}(x), y) \right] \right| \le O\left( \sqrt{\frac{n|\cA|\ln(n \lambda |\cC| T / \kappa)}{T}} \right)
    \end{align*}
    
    For a fixed $\hat{\ell}_\gamma \in \hat{\cL}_\gamma$ and assignment of policies $\{c_a \in \cC\}_{a\in\cA}$, consider the following sequence adapted to $\cF_t$,
    \[
    S_t = S_{t-1} + \hat{\ell}_\gamma(c_{\BR^{\hat{\ell}_\gamma}(p_t(x_t))}(x_t), y_t) - \E_{(x,y) \sim \cD} \left[ \hat{\ell}_\gamma(c_{\BR^{\hat{\ell}_\gamma}(p_t(x))}(x), y) \right]
    \]
    
    Since $S_{t-1} \in \cF_{t-1}$ and
    \begin{align*}
        \E \left[ \hat{\ell}_\gamma(c_{\BR^{\hat{\ell}_\gamma}(p_t(x_t))}(x_t), y_t) \mid \cF_{t-1} \right] = \E \left[ \E_{(x,y) \sim \cD}  \left[ \hat{\ell}_\gamma(c_{\BR^{\hat{\ell}_\gamma}(p_t(x))}(x), y) \right] \mid \cF_{t-1} \right]
    \end{align*}
    we have $\E[S_t \mid \cF_{t-1}] = S_{t-1}$, so $\{S_t\}_{t=1}^T$ is a martingale.
    
    We also have $|S_t-S_{t-1}| \le 1$, so by applying Azuma-Hoeffding’s inequality (Lemma \ref{lem:azuma}), we get that with probability $1-\frac{\kappa}{2|\hat{\cL}_\gamma||\cC|^{|\cA|}}$,
    \begin{align*}
        \left| \frac{1}{T}\sum_{t=1}^T \hat{\ell}_\gamma(c_{\BR^{\hat{\ell}_\gamma}(p_t(x_t))}(x_t), y_t) - \frac{1}{T}\sum_{t=1}^T \E_{(x,y) \sim \cD} \left[ \hat{\ell}_\gamma(c_{\BR^{\hat{\ell}_\gamma}(p_t(x))}(x), y) \right] \right| &\le \sqrt{\frac{2\ln \left(2|\hat{\cL}_\gamma||\cC|^{|\cA|}/\kappa\right)}{T}} \\
        &\le \sqrt{\frac{2\ln \left(2(2\lambda/\gamma)^{n|\cA|}|\cC|^{|\cA|}/\kappa\right)}{T}}
    \end{align*}

    For our setting of $\gamma = \frac{1}{2n\sqrt{T}}$, this expression is at most:
    \[
        O\left( \sqrt{\frac{n|\cA|\ln(n \lambda |\cC| T / \kappa)}{T}} \right)
    \]

    Taking the union bound over every $\hat{\ell}_\gamma \in \hat{\cL}_\gamma$ and every assignment of policies $\{c_a \in \cC\}_{a\in\cA}$, we get that with probability $1-\kappa/2$, for any $\hat{\ell}_\gamma \in \hat{\cL}_\gamma$ and any assignment of policies $\{c_a \in \cC\}_{a\in\cA}$:
    \begin{align*}
        \left| \frac{1}{T}\sum_{t=1}^T \hat{\ell}_\gamma(c_{\BR^{\hat{\ell}_\gamma}(p_t(x_t))}(x_t), y_t) - \frac{1}{T}\sum_{t=1}^T \E_{(x,y) \sim \cD} \left[ \hat{\ell}_\gamma(c_{\BR^{\hat{\ell}_\gamma}(p_t(x))}(x), y) \right] \right| \le O\left( \sqrt{\frac{n|\cA|\ln(n \lambda |\cC| T / \kappa)}{T}} \right)
    \end{align*}

    Since $\hat{\cL}_\gamma$ is the $(n\gamma)$-cover of $\hat{\cL}$ and $\hat{\cL}$ is an $\delta$-approximate basis for $\cL$, we get that when the above inequality holds, for any $\ell \in \cL$ and any assignment of policies $\{c_a \in \cC\}_{a\in\cA}$:
    \begin{align*}
        & \left| \frac{1}{T}\sum_{t=1}^T \ell(c_{k^\ell(p_t(x_t))}(x_t), y_t) - \frac{1}{T}\sum_{t=1}^T \E_{(x,y) \sim \cD} \left[ \ell(c_{k^\ell(p_t(x))}(x), y) \right] \right| \\
        &= \left| \frac{1}{T}\sum_{t=1}^T \ell(c_{\BR^{\hat{\ell}_\gamma}(p_t(x_t))}(x_t), y_t) - \frac{1}{T}\sum_{t=1}^T \E_{(x,y) \sim \cD} \left[ \ell(c_{\BR^{\hat{\ell}_\gamma}(p_t(x))}(x), y) \right] \right| \\
        &\le \left| \frac{1}{T}\sum_{t=1}^T \ell(c_{\BR^{\hat{\ell}_\gamma}(p_t(x_t))}(x_t), y_t) - \frac{1}{T}\sum_{t=1}^T \hat{\ell}_\gamma(c_{\BR^{\hat{\ell}_\gamma}(p_t(x_t))}(x_t), y_t) \right| \\
        &+ \left| \frac{1}{T}\sum_{t=1}^T \E_{(x,y) \sim \cD} \left[ \hat{\ell}_\gamma(c_{\BR^{\hat{\ell}_\gamma}(p_t(x))}(x), y) \right] - \frac{1}{T}\sum_{t=1}^T \E_{(x,y) \sim \cD} \left[ \ell(c_{\BR^{\hat{\ell}_\gamma}(p_t(x))}(x), y) \right] \right| \\
        &+ \left| \frac{1}{T}\sum_{t=1}^T \hat{\ell}_\gamma(c_{\BR^{\hat{\ell}_\gamma}(p_t(x_t))}(x_t), y_t) - \frac{1}{T}\sum_{t=1}^T \E_{(x,y) \sim \cD} \left[ \hat{\ell}_\gamma(c_{\BR^{\hat{\ell}_\gamma}(p_t(x))}(x), y) \right] \right| \\
        &\le (n\gamma+\delta) + (n\gamma+\delta) + O\left( \sqrt{\frac{n|\cA|\ln(n \lambda |\cC| T / \kappa)}{T}} \right) \\
        &\le O\left( \sqrt{\frac{n|\cA|\ln(n \lambda |\cC| T / \kappa)}{T}} \right) + 2\delta
    \end{align*}
    where we use the triangle inequality.
    
\end{proof}

We then combine Lemma \ref{lem:online-to-batch-action}, and Lemma \ref{lem:online-to-batch-benchmark} with the triangle inequality to derive
\begin{align*}
    & \frac{1}{T} \sum_{t=1}^T \E_{(x,y) \sim \cD} \left[ \ell(k^\ell(p_t(x)), y) - \ell(c_{k^\ell(p_t(x))}(x), y) \right] \\
    &\le  \frac{1}{T} \sum_{t=1}^T \left( \ell(k^\ell(p_t(x_t)), y_t) - \ell(c_{k^\ell(p_t(x_t))}(x_t), y_t) \right) \\
    &+ \left| \frac{1}{T}\sum_{t=1}^T \ell(k^\ell(p_t(x_t)), y_t) - \frac{1}{T}\sum_{t=1}^T \E_{(x,y) \sim \cD} \left[ \ell(k^\ell(p_t(x)), y) \right] \right| \\
    &+ \left| \frac{1}{T}\sum_{t=1}^T \ell(c_{k^\ell(p_t(x_t))}(x_t), y_t) - \frac{1}{T}\sum_{t=1}^T \E_{(x,y) \sim \cD} \left[ \ell(c_{k^\ell(p_t(x))}(x), y) \right] \right| \\
    &\le \eps + O\left( \sqrt{\frac{n|\cA|\ln(n \lambda |\cC| T / \kappa)}{T}} \right) + 4\delta
\end{align*}

This finishes the proof of Theorem \ref{thm:batch}.
\end{proof}

}

\ifarxiv
\subsection{Continuous Action Spaces}\label{sec:cont-action}
\seccontsactionspaces

\subsection{Online to Batch Conversion}\label{sec:online-to-batch}
\seconlinetobatch
\fi

\subsection{Offline Omniprediction}\label{sec:offline-omniprediction}

With the extension to continuous action spaces and the online-to-batch reduction, we obtain sample complexity bounds for offline omniprediction, as long as there exists a reasonably sized basis for the losses $\cL$ and randomized predictors are allowed.

\begin{theorem}\label{thm:omniprediction-batch}
    Let $\cA=[0,1]^d$ be the action space and $\cY=[0,1]^d$ be the outcome space. Let $\cC$ be a collection of policies $c:\cX\to\cA$. Let $\cL$ be a family of loss functions $\ell:\cA\times\cY\to[0,1]$. Suppose $\mathcal{S}$ is a $(n,\lambda,\delta)$-approximate basis for $\cL$, where $n$ and $\lambda$ scales polynomially with $1/\delta$. For any $\eps$, there exists an algorithm that returns an $(\cL,\cC,\eps)$-Omnipredictor with probability $1-e^{-O(poly(1/\eps))}$. The time and sample complexity of the algorithm scales polynomially with $1/\eps$.
\end{theorem}

\newcommand{\proofofflineomniprediction}{
    Let $\cA_\theta$ be a $\theta$-net of $\cA$ for $\theta = \left( \frac{\lambda^2 n}{T} \right)^{1/(3d+2)}$. Let $\cL_\theta = \{ \ell|_{\cA_\theta}: \cA_\theta \times \cY \to [0,1]\}_{\ell\in\cL}$ be the restriction of the loss functions to $\cA_\theta$. Let $\tilde{\cC} = \{\tilde{c}: \cX\to\cA_\theta\}_{c\in\cC}$ be the collection of policies with ``snapped" outputs. According to Corollary \ref{cor:decisionreg-convex-continuous-action}, running Algorithm \ref{alg:convex} with inputs $\cL_\theta$, $\cS$ and $\tilde{\cC}$ produces predictions $\hat{p}_1,...,\hat{p}_T \in [0,1]^n$ that has $(\cL,\cC,\eps')$-decision swap regret for:
    \[
        \eps' \leq O\left( \frac{\lambda^{2/(3d+2)} n^{1/(3d+2)} \sqrt{\ln(n\lambda|\cC|T)}}{T^{1/(3d+2)}} \right) + 2\delta
    \]

    By Theorem \ref{thm:batch}, we can feed i.i.d. examples to this algorithm and obtain a randomized predictor $\pi$ such that, for any loss function $\ell \in \cL$ and any assignment of policies $\{c_a \in \cC\}_{a \in \cA}$:
    \begin{align*}
        & \E_{\substack{(x,y) \sim \cD \\ p \sim \pi(x)}} \left[ \ell(k^\ell(p), y) - \ell(c_{k^\ell(p)}(x), y) \right] \\
        &\leq \eps' + O\left( \sqrt{\frac{n|\cA_\theta|\ln(n \lambda |\cC| T / \kappa)}{T}} \right) + 4\delta \\
        &\le O\left( \frac{\lambda^{2/(3d+2)} n^{1/(3d+2)} \sqrt{\ln(n\lambda|\cC|T)}}{T^{1/(3d+2)}} \right) + O\left( \sqrt{\frac{n\left(\frac{T}{\lambda^2 n}\right)^{d/(3d+2)}\ln(n \lambda |\cC| T / \kappa)}{T}} \right) + 6\delta \\
        &= \underbrace{ O\left( \frac{\lambda^{2/(3d+2)} n^{1/(3d+2)} \sqrt{\ln(n\lambda|\cC|T)}}{T^{1/(3d+2)}} \right) }_{:=\tau_1} + \underbrace{ O\left( \frac{n^{(d+1)/(3d+2)}}{\lambda^{d/(3d+2)}T^{(d+1)/(3d+2)}}\sqrt{\ln(n \lambda |\cC| T / \kappa)} \right) }_{:=\tau_2} + 6\delta
    \end{align*}
    where $T$ is the number of i.i.d. examples, $n$ and $\lambda$ may depend on $\delta$ for different loss families.

    For any $\eps$, we can choose $(\delta,\kappa,T)$ such that $\tau_1+\tau_2+6\delta$, is at most $\eps$. We first present the steps for constructing $(\delta,\kappa,T)$, and then provide their explicit expressions for 1-dimensional convex 1-Lipschitz losses $\cL_{\cF^1_{\text{cvx}}}$ and $d$-dimensional $L$-Lipschitz losses $\cL_{\cF^d_L}$ in the proof of Corollary \ref{cor:omniprediction-batch-convex} and \ref{cor:omniprediction-batch-lipschitz}. 
    
    We can first set $\delta = \eps / 18$. Since $n$ and $\lambda$ scales polynomially with $1/\delta$, $\tau_1$ scales polynomially with $1/\eps$ and $1/T$ now. We can select $T$ such that $T$ scales polynomially with $1/\eps$ and $\tau_1$ is at most $\eps / 3$. Then we can select $\kappa$ such that $\ln(1/\kappa)$ scales polynomially with $1/\eps$ and $\tau_2$ is at most $\eps / 3$. As a result, $\tau_1+\tau_2+6\delta \le \eps$, the predictor $\pi$ is an $(\cL,\cC,\eps)$-Omnipredictor.

    The time complexity is the product of the sample complexity and the per round time complexity of Algorithm \ref{alg:convex}, which scales polynomially with $\delta$, $|\cA_\theta|$, $|\cC|$, as indicated by Theorem \ref{thm:unbiased-prediction-algorithm} from \cite{noarov2023highdimensional}. Therefore, the time complexity also scales polynomially with $1/\eps$.
}

\ifarxiv
\begin{proof}
    \proofofflineomniprediction
\end{proof}
\fi

Here we give sample complexity bounds for key loss families of interest: 1-dimensional convex 1-Lipschitz losses $\cL_{\cF^1_{\text{cvx}}}$ and $L$-Lipschitz losses $\cL_{\cF^d_L}$ in $d$ dimensions.


\begin{corollary}\label{cor:omniprediction-batch-convex}
    Let $\cA=[0,1]$ be the action space and $\cY=[0,1]$ be the outcome space. Let $\cC$ be a collection of policies $c:\cX\to\cA$. Consider the family of convex 1-Lipschitz losses $\cL_{\cF^1_{\text{cvx}}}$. For any $\eps$, there exists an algorithm that returns an $(\cL_{\cF^1_{\text{cvx}}},\cC,\eps)$-Omnipredictor with probability $1-e^{-O(1/\eps^2)}$. The sample complexity of the algorithm is $O\left(\frac{\ln^{4/3}(1/\delta)\ln^{5/2}(|\cC|/\eps)}{\eps^{17/3}}\right)$.
\end{corollary}

\newcommand{\proofcoromnipredictionbatchconvex}{
    Since we have a $\left (\frac{\ln^{4/3}(1/\delta)}{\delta^{2/3}},O(1),\delta \right)$-approximate basis for 1-dimensional convex 1-Lipschitz losses $\cL_{\cF^1_{\text{cvx}}}$, the predictor $\pi$ described in the proof of Theorem \ref{thm:batch} satisfies:
    \begin{align*}
        & \E_{\substack{(x,y) \sim \cD \\ p \sim \pi(x)}} \left[ \ell(k^\ell(p), y) - \ell(c_{k^\ell(p)}(x), y) \right] \\
        &\le O\left( \frac{\left(\frac{\ln^{4/3}(1/\delta)}{\delta^{2/3}}\right)^{1/5} \sqrt{\ln(\frac{\ln^{4/3}(1/\delta)}{\delta^{2/3}}|\cC|T)}}{T^{1/5}} \right) + O\left( \frac{\left(\frac{\ln^{4/3}(1/\delta)}{\delta^{2/3}}\right)^{2/5}\sqrt{\ln(\frac{\ln^{4/3}(1/\delta)}{\delta^{2/3}} |\cC| T / \kappa)}}{T^{2/5}} \right) + 6\delta
    \end{align*}

    By setting $T = O\left(\frac{\ln^{4/3}(1/\delta)\ln^{5/2}(|\cC|/\eps)}{\eps^{17/3}}\right)$, $\kappa=e^{-O(1/\eps^2)}$ and $\delta = \eps/18$, we can verify that this expression is at most $\eps$ in this case.
}
\ifarxiv
\begin{proof}
    \proofcoromnipredictionbatchconvex
\end{proof}
\fi

\begin{corollary}\label{cor:omniprediction-batch-lipschitz}
    Let $\cA=[0,1]^d$ be the action space and $\cY=[0,1]^d$ be the outcome space. Let $\cC$ be a collection of policies $c:\cX\to\cA$. Consider the family of $L$-Lipschitz losses $\cL_{\cF^d_L}$. For any $\eps$, there exists an algorithm that returns an $(\cL_{\cF^d_L},\cC,\eps)$-Omnipredictor with probability $1-e^{-O(1/\eps^{2d})}$. The sample complexity of the algorithm is $O\left(\frac{L^d d^{3d/2+1} \ln^{3d/2+1} (L |\cC| / \eps) }{\eps^{4d+2}}\right)$.
\end{corollary}

\newcommand{\proofcoromnipredictionbatchlipschitz}{
    Since we have a $\left ((L/\delta)^d,1,\delta \right)$-approximate basis for $L$-Lipschitz losses $\cL_{\cF^d_L}$ in $d$ dimensions, the predictor $\pi$ described in the proof of Theorem \ref{thm:batch} satisfies:
    \begin{align*}
        & \E_{\substack{(x,y) \sim \cD \\ p \sim \pi(x)}} \left[ \ell(k^\ell(p), y) - \ell(c_{k^\ell(p)}(x), y) \right] \\
        &\le O\left( \frac{(L/\delta)^{d/(3d+2)} \sqrt{\ln((L/\delta)^d|\cC|T)}}{T^{1/(3d+2)}} \right) + O\left( \frac{(L/\delta)^{d(d+1)/(3d+2)}}{T^{(d+1)/(3d+2)}}\sqrt{\ln((L/\delta)^d |\cC| T / \kappa)} \right) + 6\delta
    \end{align*}

    By setting $T = O\left(\frac{L^d d^{3d/2+1} \ln^{3d/2+1} (L |\cC| / \eps) }{\eps^{4d+2}}\right)$, $\kappa=e^{-O(1/\eps^{2d})}$ and $\delta = \eps/18$, we can verify that this expression is at most $\eps$ in this case.
}
\ifarxiv
\begin{proof}
     \proofcoromnipredictionbatchlipschitz
\end{proof}
\fi

\bibliographystyle{ACM-Reference-Format}
\bibliography{main}

\newpage
\appendix

\section{Contextual Swap Regret}

Here we elucidate the connection between our definition of decision swap regret and the following definition of \textit{contextual swap regret} appearing in the literature (e.g. \citet{garg2024oracle,lin2024persuading}).  

\begin{definition}[$(\cL, \cC, \eps)$-Contextual Swap Regret]
    Let $\cL$ be a family of loss functions and $\cC$ be a collection of policies $c:\cX\to\cA$. A sequence of predictions $p_1,...,p_T$ has $(\cL, \cC, \eps)$-contextual swap regret with respect to outcomes $y_1,...,y_T$ if for any loss function $\ell\in\cL$, any policy $c\in\cC$, and any swap function $\phi: \cA\times\cA \to \cA$, choosing actions $a_t = \BR^\ell(p_t(x_t))$ achieves:
    \[
    \frac{1}{T} \sum_{t=1}^T \left( \ell(a_t, y_t) - \ell(\phi(a_t, c(x_t)), y_t) \right) \leq \eps
    \]
\end{definition}

Like decision swap regret, contextual swap regret is a strengthening of both omniprediction and swap regret, since benchmarks can depend on the actions taken and the suggested actions of a fixed policy. We show that making predictions that achieve low decision swap regret (over a slightly richer class) is only more general than making predictions that achieve low contextual swap regret.

\begin{proposition}
    Let $\cL$ be a family of loss functions and $\cC$ be a collection of policies $c:\cX\to\cA$. Let $\Phi$ be the collection of all swap functions $\phi: \cA\to\cA$. Let $\cC_\Phi = \{\phi \circ c: \cX\to\cA \}_{\phi\in\Phi, c\in\cC}$. If a sequence of predictions $p_1,...,p_T$ has $(\cL,\cC_\Phi,\eps)$ decision swap regret with respect to outcomes $y_1,...,y_T$, then it has $(\cL,\cC,\eps)$ contextual swap regret with respect to the same outcomes. 
    This implies that when $\cC_\Phi = \cC$, i.e., the collection of policies $\cC$ is closed under any swap function, decision swap regret is stronger than contextual swap regret.
\end{proposition}
\begin{proof}
    Fix any $\ell\in\cL$. By $(\cL,\cC_\Phi,\eps)$ decision swap regret, we have that for any assignment of $\{(\phi\circ c)_a \in \cC_\Phi\}_{a\in\cA}$, or written equivalently, $\{\phi_a \circ c_a \in \cC_\Phi\}_{a\in\cA}$: 
    \[
    \frac{1}{T} \sum_{t=1}^T \left( \ell(a_t, y_t) - \ell(\phi_{a_t}(c_{a_t}(x_t)), y_t) \right) \leq \eps
    \]
    In particular, this implies that for any assignment $\{\phi_a \circ c \in \cC_\Phi\}_{a\in\cA}$ where $c_a = c$ for some $c\in\cC$, we have that:
    \[
    \frac{1}{T} \sum_{t=1}^T \left( \ell(a_t, y_t) - \ell(\phi_{a_t}(c(x_t)), y_t) \right) \leq \eps
    \]
    Since the output of $\phi_{a_t}(c(x_t))$ depends both on $a_t$ and $c(x_t)$, this expression is exactly the contextual swap regret as measured by $\ell$, which completes the proof. 
\end{proof}

\ifarxiv
\section{Proofs from Section \ref{sec:basis}}\label{app:basis}

\subsection{Proof of Lemma \ref{lem:indicator-to-relu}}

\begin{proof}
    When $d=1$, we can verify the lemma by considering three cases: $z \ge i+1$, $z \le i-1$, $z=i$.

    If $z \ge i+1$, then $\ReLU_{i+1}(z) - 2\ReLU_{i}(z) + \ReLU_{i-1}(z) = (z-i-1) - 2(z-i) + (z-i+1) = 0$.

    If $z \le i-1$, then $\ReLU_{i+1}(z) - 2\ReLU_{i}(z) + \ReLU_{i-1}(z) = 0 - 0 + 0 = 0$.

    If $z = i$, then $\ReLU_{i+1}(z) - 2\ReLU_{i}(z) + \ReLU_{i-1}(z) = 0 - 0 + 1 = 1$.
    
    For $d \ge 2$, we will verify the lemma in four different cases based on the values of $z_1-i_1,\cdots,z_d-i_d$.
    
    \paragraph{Case 1}
    Suppose there exists $j \in \{1,\cdots,d\}$ such that $z_j-i_j \le -1$. Due to the symmetry across the $d$ coordinates, we can assume $z_1-i_1 \le -1$ without loss of generality. Then for any $\sigma_1 \in \{0,1\}$, $z_1-i_1-\sigma_1 \le -1-\sigma_1 \le 0$. So $z_1-i_1-\sigma_1$ can be omitted from the calculation of $\max\{z_1-i_1-\sigma_1, \cdots, z_d-i_d-\sigma_d, 0\}$. This implies that
    \[
    \MReLU_{i_1+\sigma_1,\cdots,i_d+\sigma_d}(z_1,\cdots,z_d) = \MReLU_{i_2+\sigma_2,\cdots,i_d+\sigma_d}(z_2,\cdots,z_d)
    \]
    
    Similarly, $z_1-i_1+\sigma_1 \le -1+\sigma_1 \le 0$, $z_1-i_1+\sigma_1$ can be omitted from the calculation of $\max\{z_1-i_1+\sigma_1, \cdots, z_d-i_d+\sigma_d, 0\}$ which implies that
    \[
    \MReLU_{i_1-\sigma_1,\cdots,i_d-\sigma_d}(z_1,\cdots,z_d) = \MReLU_{i_2-\sigma_2,\cdots,i_d-\sigma_d}(z_2,\cdots,z_d)
    \]

    Therefore,
    \begin{align*}
        & \sum_{\sigma_1,\cdots,\sigma_d \in \{0,1\}} (-1)^{d+\sigma_1+\cdots+\sigma_d} \MReLU_{i_1+\sigma_1,\cdots,i_d+\sigma_d}(z_1,\cdots,z_d) \\
        &= \sum_{\sigma_1 \in \{0,1\}} (-1)^{\sigma_1} \sum_{\sigma_2,\cdots,\sigma_d \in \{0,1\}} (-1)^{d+\sigma_2+\cdots+\sigma_d} \MReLU_{i_2+\sigma_2,\cdots,i_d+\sigma_d}(z_2,\cdots,z_d) \\
        &= 0
    \end{align*}
    and
    \begin{align*}
        & \sum_{\sigma_1,\cdots,\sigma_d \in \{0,1\}} (-1)^{1+\sigma_1+\cdots+\sigma_d} \MReLU_{i_1-\sigma_1,\cdots,i_d-\sigma_d}(z_1,\cdots,z_d) \\
        &= \sum_{\sigma_1 \in \{0,1\}} (-1)^{\sigma_1} \sum_{\sigma_2,\cdots,\sigma_d \in \{0,1\}} (-1)^{1+\sigma_2+\cdots+\sigma_d} \MReLU_{i_2-\sigma_2,\cdots,i_d-\sigma_d}(z_2,\cdots,z_d) \\
        &= 0
    \end{align*}

    So we have verified the lemma when there exists $j \in \{1,\cdots,d\}$ such that $z_j-i_j \le -1$. Now it suffices to verify the lemma when $z_j-i_j \ge 0$ for any $j \in \{1,\cdots,d\}$.

    \paragraph{Case 2}
    If there exists $j, j' \in \{1,\cdots,d\}$, such that $z_j - i_j \neq z_{j'} - i_{j'}$. Without loss of generality, suppose that $z_j - i_j \le z_{j'} - i_{j'} - 1$, then for any $\sigma_j,\sigma_{j'} \in \{0,1\}$, 
    \begin{align*}
        z_j - i_j - \sigma_j &\le z_{j'} - i_{j'} - 1 - \sigma_j \\
        &\le z_{j'} - i_{j'} - 1 \\
        &\le z_{j'} - i_{j'} - \sigma_{j'}
    \end{align*}
    and
    \begin{align*}
        z_j - i_j + \sigma_j &\le z_{j'} - i_{j'} - 1 + \sigma_j \\
        &\le z_{j'} - i_{j'} \\
        &\le z_{j'} - i_{j'} + \sigma_{j'}
    \end{align*}
    Same as Case 1, $z_j-i_j-\sigma_j$ can be omitted from the  calculation of $\max\{z_1-i_1-\sigma_1, \cdots, z_d-i_d-\sigma_d, 0\}$, and $z_j-i_j+\sigma_1$ can be omitted from the calculation of $\max\{z_1-i_1+\sigma_1, \cdots, z_d-i_d+\sigma_d, 0\}$. So we can verify the lemma based on the same argument as in Case 1.

    Now it suffices to verify the lemma when $z_1 - i_1 = \cdots = z_d - i_d \ge 0$.
    
    \paragraph{Case 3}
    If $z_1 - i_1 = \cdots = z_d - i_d = s \ge 1$, then for any $\sigma_1,\cdots,\sigma_d \in \{0,1\}$,
    \begin{align*}
        \MReLU_{i_1+\sigma_1,\cdots,i_d+\sigma_d}(z_1,\cdots,z_d) &= \max\{z_1 - i_1 - \sigma_1, \cdots, z_d - i_d - \sigma_d, 0\} \\
        &= s - \min\{\sigma_1,\cdots,\sigma_d\} \\
        &= s - \1[\sigma_1=\cdots=\sigma_d=1]
    \end{align*}
    and
    \begin{align*}
        \MReLU_{i_1-\sigma_1,\cdots,i_d-\sigma_d}(z_1,\cdots,z_d) &= \max\{z_1 - i_1 + \sigma_1, \cdots, z_d - i_d + \sigma_d, 0\} \\
        &= s + \max\{\sigma_1,\cdots,\sigma_d\} \\
        &= s + 1 - \1[\sigma_1=\cdots=\sigma_d=0]
    \end{align*}
    Therefore,
    \begin{align*}
        & \sum_{\sigma_1,\cdots,\sigma_d \in \{0,1\}} (-1)^{d+\sigma_1+\cdots+\sigma_d} \MReLU_{i_1+\sigma_1,\cdots,i_d+\sigma_d}(z_1,\cdots,z_d) \\
        &= \sum_{\sigma_1,\cdots,\sigma_d \in \{0,1\}} (-1)^{d+\sigma_1+\cdots+\sigma_d} \left( s - \1[\sigma_1=\cdots=\sigma_d=1] \right) \\
        &= -(-1)^{2d} + s \sum_{\sigma_1,\cdots,\sigma_d \in \{0,1\}} (-1)^{d+\sigma_1+\cdots+\sigma_d} \\
        &= -1
    \end{align*}    
    and
    \begin{align*}
        & \sum_{\sigma_1,\cdots,\sigma_d \in \{0,1\}} (-1)^{1+\sigma_1+\cdots+\sigma_d} \MReLU_{i_1-\sigma_1,\cdots,i_d-\sigma_d}(z_1,\cdots,z_d) \\
        &= \sum_{\sigma_1,\cdots,\sigma_d \in \{0,1\}} (-1)^{1+\sigma_1+\cdots+\sigma_d} \left( s + 1 - \1[\sigma_1=\cdots=\sigma_d=0] \right) \\
        &= -(-1)^1 + (s+1) \sum_{\sigma_1,\cdots,\sigma_d \in \{0,1\}} (-1)^{1+\sigma_1+\cdots+\sigma_d} \\
        &= 1
    \end{align*}    
    The sum of the two equations above is 0, so we have verified the lemma when $z_1 - i_1 = \cdots = z_d - i_d \ge 1$. Now the only case to check is when $z_1 - i_1 = \cdots = z_d - i_d = 0$.

    \paragraph{Case 4}
    If $z_1 - i_1 = \cdots = z_d - i_d = 0$, then for any $\sigma_1,\cdots,\sigma_d \in \{0,1\}$,
    \begin{align*}
        \MReLU_{i_1+\sigma_1,\cdots,i_d+\sigma_d}(z_1,\cdots,z_d) &= \max\{z_1 - i_1 - \sigma_1, \cdots, z_d - i_d - \sigma_d, 0\} \\
        &= \max\{- \sigma_1, \cdots, - \sigma_d, 0\} \\
        &= 0
    \end{align*}
    and
    \begin{align*}
        \MReLU_{i_1-\sigma_1,\cdots,i_d-\sigma_d}(z_1,\cdots,z_d) &= \max\{z_1 - i_1 + \sigma_1, \cdots, z_d - i_d + \sigma_d, 0\} \\
        &= \max\{\sigma_1,\cdots,\sigma_d\} \\
        &= 1 - \1[\sigma_1=\cdots=\sigma_d=0]
    \end{align*}
    Therefore,
    \begin{align*}
        \sum_{\sigma_1,\cdots,\sigma_d \in \{0,1\}} (-1)^{d+\sigma_1+\cdots+\sigma_d} \MReLU_{i_1+\sigma_1,\cdots,i_d+\sigma_d}(z_1,\cdots,z_d) = 0
    \end{align*}     
    and
    \begin{align*}
        & \sum_{\sigma_1,\cdots,\sigma_d \in \{0,1\}} (-1)^{1+\sigma_1+\cdots+\sigma_d} \MReLU_{i_1-\sigma_1,\cdots,i_d-\sigma_d}(z_1,\cdots,z_d) \\
        &= \sum_{\sigma_1,\cdots,\sigma_d \in \{0,1\}} (-1)^{1+\sigma_1+\cdots+\sigma_d} \left( 1 - \1[\sigma_1=\cdots=\sigma_d=0] \right) \\
        &= -(-1)^1 + \sum_{\sigma_1,\cdots,\sigma_d \in \{0,1\}} (-1)^{1+\sigma_1+\cdots+\sigma_d} \\
        &= 1
    \end{align*}    
    The sum of the two equations above is 1, so we have verified the lemma when $z_1 - i_1 = \cdots = z_d - i_d = 0$. This finishes the proof of Lemma \ref{lem:indicator-to-relu}.
\end{proof}

\subsection{Proof of Lemma \ref{lem:mrelu-coeffs}}

Since the magnitude of each single $c_{i_1,\cdots,i_d}^g$ is at most $2^{d+1}m^2$, it suffices to show that the number of non-zero $c_{i_1,\cdots,i_d}^g$ is bounded by $6^{d+1}m$. We first give a sufficient condition for $c_{i_1,\cdots,i_d}^g = 0$.

\begin{lemma}\label{lem:zero-coef-relu}
    For any $g \in \cG_\Leontief^d$, $i_1,\cdots,i_d \in \{0,\cdots,m\}$, if there exist $i_{j_1},i_{j_2} \in \{2,\cdots,m-2\}$, and $|b_{j_1}i_{j_1}-b_{j_2}i_{j_2}| \ge \max\{b_{j_1},b_{j_2}\}$, then $c_{i_1,\cdots,i_d}^g = 0$.
\end{lemma}


\begin{proof}
     By definition, $c_{i_1,\cdots,i_d}^g$ involves two summations over $\sigma_1,\cdots,\sigma_d$. Such summation can be calculated by summing over two parts --- $(\sigma_{j_1},\sigma_{j_2})$, and the other $\sigma_j$ where $j \notin \{j_1,j_2\}$. In other words, we have
    \begin{align*}
        c_{i_1,\cdots,i_d}^g &= \sum_{\substack{\sigma_j \in \{0,1\}, j \notin \{j_1,j_2\},\\ i_j-\sigma_j \in \{1,\cdots,m-1\}}} \sum_{\sigma_{j_1},\sigma_{j_2} \in \{0,1\}} (-1)^{d+\sigma_1+\cdots+\sigma_d} \min\{b_1(i_1-\sigma_1),\cdots,b_d(i_d-\sigma_d)\} \\
        &\hspace{1em}+ \sum_{\substack{\sigma_j \in \{0,1\}, j \notin \{j_1,j_2\},\\ i_j+\sigma_j \in \{1,\cdots,m-1\}}} \sum_{\sigma_{j_1},\sigma_{j_2} \in \{0,1\}} (-1)^{1+\sigma_1+\cdots+\sigma_d} \min\{b_1(i_1+\sigma_1),\cdots,b_d(i_d+\sigma_d)\}
    \end{align*}
    
    We will show that for any fixed $\{\sigma_j: j \notin \{j_1,j_2\}\}$ included in the summation, summing over $(\sigma_{j_1},\sigma_{j_2})$ leads to zeroes, i.e.,
    \begin{align*}
       & \sum_{\sigma_{j_1},\sigma_{j_2} \in \{0,1\}} (-1)^{d+\sigma_1+\cdots+\sigma_d} \min\{b_1(i_1-\sigma_1),\cdots,b_d(i_d-\sigma_d)\} = 0, \\
       & \sum_{\sigma_{j_1},\sigma_{j_2} \in \{0,1\}} (-1)^{1+\sigma_1+\cdots+\sigma_d} \min\{b_1(i_1+\sigma_1),\cdots,b_d(i_d+\sigma_d)\} = 0
    \end{align*}
    
    More specifically, we will show that:
    \begin{align*}
       & \text{When } \sigma_{j_1} = 0, \sum_{\sigma_{j_2} \in \{0,1\}} (-1)^{d+\sigma_1+\cdots+\sigma_d} \min\{b_1(i_1-\sigma_1),\cdots,b_d(i_d-\sigma_d)\} = 0, \\
       & \text{When } \sigma_{j_1} = 1, \sum_{\sigma_{j_2} \in \{0,1\}} (-1)^{d+\sigma_1+\cdots+\sigma_d} \min\{b_1(i_1-\sigma_1),\cdots,b_d(i_d-\sigma_d)\} = 0, \\
       & \text{When } \sigma_{j_2} = 0, \sum_{\sigma_{j_1} \in \{0,1\}} (-1)^{1+\sigma_1+\cdots+\sigma_d} \min\{b_1(i_1+\sigma_1),\cdots,b_d(i_d+\sigma_d)\} = 0, \\
       & \text{When } \sigma_{j_2} = 1, \sum_{\sigma_{j_1} \in \{0,1\}} (-1)^{1+\sigma_1+\cdots+\sigma_d} \min\{b_1(i_1+\sigma_1),\cdots,b_d(i_d+\sigma_d)\} = 0
    \end{align*}

    Without loss of generality, we assume that $b_{j_1}i_{j_1} \le b_{j_2}i_{j_2}$. For $\sigma_{j_1} = 0$, we compare the two terms in the summation that correspond to $\sigma_{j_2} \in \{0,1\}$, i.e.,
    \[
        (-1)^{d+\sigma_1+\cdots+\sigma_{j_2}+\cdots+\sigma_d} \min\{b_1(i_1-\sigma_1),\cdots,b_{j_2}(i_{j_2}-\sigma_{j_2}),\cdots,b_d(i_d-\sigma_d)\}
    \]
    where $(-1)^{d+\sigma_1+\cdots+\sigma_{j_2}+\cdots+\sigma_d}$ determines the sign and $\min\{b_1(i_1-\sigma_1),\cdots,b_{j_2}(i_{j_2}-\sigma_{j_2}),\cdots,b_d(i_d-\sigma_d)\}$ determines the magnitude.

    For the magnitude, the only difference is one term involves $b_{j_2}(i_{j_2}-1)$ and the other involves $b_{j_2}i_{j_2}$. By the assumption of Lemma \ref{lem:zero-coef-relu}, $b_{j_1}i_{j_1}$ is always no larger than $b_{j_2}(i_{j_2}-\sigma_{j_2})$, so $b_{j_2}(i_{j_2}-\sigma_{j_2})$ can be omitted from the calculation of $\min\{b_1(i_1-\sigma_1),\cdots,b_{j_2}(i_{j_2}-\sigma_{j_2}),\cdots,b_d(i_d-\sigma_d)\}$. As a result, the magnitude of the two terms are equal.

    Since we fix $\{\sigma_j: j \notin \{j_1,j_2\}\}$, $\sigma_{j_1} = 0$, $\sigma_{j_2} \in \{0,1\}$, the signs of the two terms are opposite. Therefore, the sum of the terms equals to zero.
    \[
    \text{When } \sigma_{j_1} = 0, \sum_{\sigma_{j_2} \in \{0,1\}} (-1)^{d+\sigma_1+\cdots+\sigma_d} \min\{b_1(i_1-\sigma_1),\cdots,b_d(i_d-\sigma_d)\} = 0
    \]

    For the other cases, we can prove using similar arguments. Then summing over $\{\sigma_j: j \notin \{j_1,j_2\}\}$ finishes the proof.
    
\end{proof}

Now we bound the number of non-zero $c_{i_1,\cdots,i_d}^g$ and obtain an upper bound for the sum of the absolute values of the coefficients $c_{i_1,\cdots,i_d}^g$.

\begin{proofof}{Lemma \ref{lem:mrelu-coeffs}}
    Without loss of generality, we assume that $b_1 \le b_2 \le \cdots \le b_d$. We can always permute the variables to obtain this ordering. 

    Now we count the combinations of $i_1,\cdots,i_d$ whose corresponding $c_{i_1,\cdots,i_d}^g$ can be non-zero. This is done by counting $i_1,\cdots,i_d$ that violate the conditions in Lemma \ref{lem:zero-coef-relu}.

    \paragraph{Case 1} 
    For any $j \in \{1,\cdots,d\}$, $i_j \in \{0,1,m-1,m\}$, this corresponds to $4^d$ combinations of $i_1,\cdots,i_d$.
    
    \paragraph{Case 2}
    There exists $j \in \{1,\cdots,d\}$, such that $i_j \in \{2,\cdots,m-2\}$. Suppose $j_0$ is the smallest among such $j$.

    Then for $j < j_0$, $i_j \in \{0,1,m-1,m\}$, there are $4^{j_0-1}$ combinations.

    For $j > j_0$, $c_{i_1,\cdots,i_d}^g = 0$ if $b_j \in \{2,\cdots,m-2\}$ and $|\frac{b_{j_0}}{b_{j}}i_{j_0}-i_{j}| \ge 1$. So, $c_{i_1,\cdots,i_d}^g$ being non-zero requires $i_{j}$ to be in $(\frac{b_{j_0}}{b_{j}}i_{j_0} - 1, \frac{b_{j_0}}{b_{j}}i_{j_0} + 1) \cup \{0,1,m-1,m-2\}$, which include at most 6 integers. There are $6^{d-j_0}$ combinations.

    All together, the number of nonzero $c_{i_1,\cdots,i_d}^g$ is bounded by $4^d + \sum_{j_0=1}^d 4^{j_0-1} \cdot m \cdot 6^{d-j_0} \le 6^{d+1}m$.

\end{proofof}

\subsection{Proof of Lemma \ref{lem:relu-to-hyperrectangle}}

\begin{proof}
    For $d = 1$, we can verify the result by considering two cases: (1) $z \le i$, (2) $z \ge i+1$.

    If $z \le i$, then $\ReLU_i(z)-\ReLU_{i+1}(z) = 0 - 0 = 0$.

    If $z \ge i+1$, then $\ReLU_i(z)-\ReLU_{i+1}(z) = (z-i) - (z-i-1) = 1$.

    For $d \ge 2$, it is equivalent to prove 
    \begin{align*}
    & \MReLU_{i_1,i_2,\cdots,i_d}(z_1,\cdots,z_d) - \MReLU_{i_1+1,i_2,\cdots,i_d}(z_1,\cdots,z_d) \\
    &= \left[ \MReLU_{i_1,i_2,\cdots,i_{d-1}}(z_1,\cdots,z_{d-1}) - \MReLU_{i_1+1,i_2,\cdots,i_{d-1}}(z_1,\cdots,z_{d-1}) \right] \bI_{i_1-i_d+1,m}(z_1-z_d)
    \end{align*}
    
    We can verify this by considering two cases: (1) $z_1-z_d \ge i_1-i_d+1$, (2) $z_1-z_d \le i_1-i_d$.

    If $z_1-z_d \ge i_1-i_d+1$, then $z_1-i_1 > z_1-i_1-1 \ge z_d-i_d$. So $z_d-i_d$ can be omitted from the calculation of $\max\{z_1-i_1,\cdots,z_d-i_d,0\}$ and $\max\{z_1-i_1-1,\cdots,z_d-i_d,0\}$. This indicates that:
    \begin{align*}
        & \MReLU_{i_1,i_2,\cdots,i_d}(z_1,\cdots,z_d) - \MReLU_{i_1+1,i_2,\cdots,i_d}(z_1,\cdots,z_d) \\
        &= \MReLU_{i_1,i_2,\cdots,i_{d-1}}(z_1,\cdots,z_{d-1}) - \MReLU_{i_1+1,i_2,\cdots,i_{d-1}}(z_1,\cdots,z_{d-1})
    \end{align*}

    If $z_1-z_d \le i_1-i_d$, then $z_1-i_1-1 < z_1-i_1 \le z_d-i_d$. As a result, $z_1-i_1$ can be omitted from the calculation of $\max\{z_1-i_1,\cdots,z_d-i_d,0\}$, $z_1-i_1-1$ can be omitted from the calculation of $\max\{z_1-i_1-1,\cdots,z_d-i_d,0\}$. This indicates that:
     \begin{align*}
        & \MReLU_{i_1,i_2,\cdots,i_d}(z_1,\cdots,z_d) - \MReLU_{i_1+1,i_2,\cdots,i_d}(z_1,\cdots,z_d) \\
        &= \MReLU_{i_2,\cdots,i_d}(z_2,\cdots,z_d) - \MReLU_{i_2,\cdots,i_d}(z_2,\cdots,z_d) \\
        &= 0
    \end{align*}

    This finishes the proof of Lemma \ref{lem:relu-to-hyperrectangle}.
\end{proof}

\subsection{Proof of Lemma \ref{lem:dyadic-hyperrectangle-basis}}\label{app:dyadic-hyperrectangle-basis}

Dyadic hyperrectangles and the identity matrix can be related in the following manner. For any fixed $h_1,\cdots,h_d \in \{0,\cdots,\ln(2m)\}$, the set $$\left\{\prod_{j=1}^d [a_j2^{h_j},(a_j+1)2^{h_j}-1]\right\}$$ are all the dyadic hyperrectangles whose edge lengths are $2^{h_1},\cdots,2^{h_d}$, and there are $N(h_1,\cdots,h_d) = \prod_{j=1}^d\frac{2m}{2^{h_j}}$ such dyadic hyperrectangles. We can enumerate these dyadic hyperrectangles as $$\left\{\text{HyperRec}^{(h_1,\cdots,h_d)}_i\right\}_{i=1}^{N(h_1,\cdots,h_d)}$$ The indicator function for $\text{HyperRec}_i^{(h_1,\cdots,h_d)}$ outputs 1 when $(w_1,\cdots,w_d)$ falls into $\text{HyperRec}_i^{(h_1,\cdots,h_d)}$ and outputs 0 when $(w_1,\cdots,w_d)$ falls into any other $\text{HyperRec}_{i'}^{(h_1,\cdots,h_d)}$, where $i' \in \{1,\cdots,N\} \setminus i$. Intuitively, the indicator function for $\text{HyperRec}_i$ can be equivalently represented by the $i$-th standard basis vector of length $N(h_1,\cdots,h_d)$. And all the indicator functions for these $N(h_1,\cdots,h_d)$ dyadic hyperrectangles can be equivalently represented by the rows of the $N(h_1,\cdots,h_d) \times N(h_1,\cdots,h_d)$ identity matrix. Lemma \ref{lem:perturbed-identity-upper} indicates that the rows of the $N \times N$ identity matrix can be $\mu$-approximately spanned by the $K = c\ln N / \mu^2$ columns of $V$. This can then be translated to an approximate basis for dyadic hyperrectangle functions.

\begin{proofof}{Lemma \ref{lem:dyadic-hyperrectangle-basis}}
By Lemma \ref{lem:perturbed-identity-upper}, $W^{(h_1,\cdots,h_d)} = \{\tilde\nu^{(h_1,\cdots,h_d)}_k\}_{k=1}^{K(h_1,\cdots,h_d)}$ is a $(K(h_1,\cdots,h_d),\sqrt{K(h_1,\cdots,h_d)},\mu)$-approximate basis for the indicator functions for the set $\{\text{HyperRec}_i^{(h_1,\cdots,h_d)}\}_{i=1}^{N(h_1,\cdots,h_d)}\}$, which are all dyadic hyperrectangles whose edge lengths are $2^{h_1},\cdots,2^{h_d}$. Therefore, $\cup_{h_1,\cdots,h_d \in \{0,\cdots,\ln(2m)\}} W^{(h_1,\cdots,h_d)}$ is a $\mu$-approximate basis for all dyadic hyperrectangle functions. 

We have that the size of this basis is:
\begin{align*}
    \sum_{h_1,\cdots,h_d \in \{0,\cdots,\ln(2m)\}} K(h_1,\cdots,h_d) &= \sum_{h_1,\cdots,h_d \in \{0,\cdots,\ln(2m)\}} O \left( \frac{\ln N(h_1,\cdots,h_d)}{\mu^2} \right) \\
    &= \sum_{h_1,\cdots,h_d \in \{0,\cdots,\ln(2m)\}} O\left( \ln\left(\prod_{j=1}^d\frac{2m}{2^{h_j}}\right) 12^{2d}d^2m^{4+2d/(d+2)}4^d\ln^{2d}m \right) \\
    &= \sum_{h_1,\cdots,h_d \in \{0,\cdots,\ln(2m)\}} O\left( \sum_{j=1}^d(\ln(2m)-h_j) d^2m^{4+2d/(d+2)}24^{2d}\ln^{2d}m \right) \\
    &= O\left( d\ln^{d+1}(2m)\cdot d^2m^{(6d+8)/(d+2)}24^{2d}\ln^{2d}m \right) \\
    &= O\left( d^3m^{(6d+8)/(d+2)}24^{2d}\ln^{3d+1}m \right)
\end{align*}

The Lipschitz constant of this basis is:
\begin{align*}
    \max_{h_1,\cdots,h_d \in \{0,\cdots,\ln(2m)\}} \sqrt{K(h_1,\cdots,h_d)} &= \sum_{h_1,\cdots,h_d \in \{0,\cdots,\ln(2m)\}} O \left( \sqrt{\frac{\ln N(h_1,\cdots,h_d)}{\mu^2}} \right) \\
    &= \max_{h_1,\cdots,h_d \in \{0,\cdots,\ln(2m)\}} O\left( \sqrt{\ln\left(\prod_{j=1}^d\frac{2m}{2^{h_j}}\right) 12^{2d}d^2m^{4+2d/(d+2)}4^d\ln^{2d}m} \right) \\
    &= \max_{h_1,\cdots,h_d \in \{0,\cdots,\ln(2m)\}} O\left( \sqrt{\sum_{j=1}^d(\ln(2m)-h_j) d^2m^{4+2d/(d+2)}24^{2d}\ln^{2d}m} \right) \\
    &\le O\left( \sqrt{d\ln(2m)} dm^{(3d+4)/(d+2)}24^{d}\ln^{d}m \right) \\
    &= O\left( d^{3/2}m^{(3d+4)/(d+2)}24^{d}\ln^{d+1/2}m \right)
\end{align*}
\end{proofof}

\subsection{Proof of Proposition \ref{thm:basis-Leontief}}

\begin{proofof}{Proposition \ref{thm:basis-Leontief}}\label{app:leontif-prop}

By Lemma \ref{lem:dyadic-hyperrectangle-basis}, $\mathcal{W}$ is a $\left( O\left( d^3m^{(6d+8)/(d+2)}24^{2d}\ln^{3d+1}m \right), O\left( d^{3/2}m^{(3d+4)/(d+2)}24^{d}\ln^{d+1/2}m \right), \mu \right)$-approximate basis for all dyadic hyperrectangle functions. As pointed out by our previous discussion, this is also a $\left( O\left( d^3m^{(6d+8)/(d+2)}24^{2d}\ln^{3d+1}m \right), (2\ln m)^d O\left( d^{3/2}m^{(3d+4)/(d+2)}24^{d}\ln^{d+1/2}m \right), \mu(2\ln m)^d \right)$-approximate basis, i.e., $\left( O\left( d^3m^{(6d+8)/(d+2)}24^{2d}\ln^{3d+1}m \right),  O\left( d^{3/2}m^{(3d+4)/(d+2)}48^{d}\ln^{2d+1/2}m \right), \frac{\delta^2}{12^{d+1}dm^{d/(d+2)}} \right)$-approximate basis for all hyperrectangle functions. Recall that Lemma \ref{lem:relu-to-hyperrectangle} allows us to represent differences between adjacent MReLU functions as hyperrectangle functions on the space of $(z_j,z_1-z_j,\cdots,z_d-z_j)$ for $j \in \{1,\cdots,d\}$. There are $d$ such spaces. So we obtain a $\Big( O\left( d^4m^{(6d+8)/(d+2)}24^{2d}\ln^{3d+1}m \right),  O\Big( d^{3/2}m^{(3d+4)/(d+2)}48^{d}\ln^{2d+1/2}m \Big)$,
$ \frac{\delta^2}{12^{d+1}dm^{d/(d+2)}} \Big)$-approximate basis for differences between adjacent MReLU functions. 

Combined with the discretized set of MReLU functions, $$\mathcal{R}_s = \{\MReLU_{i_1s,\cdots,i_ds}(z_1,\cdots,z_d) : i_1,\cdots,i_d \in \{1,\cdots,m/s\}\}$$ for $s = \lceil m^{d/(d+2)} \rceil$, we obtain a $\Big( O\left( d^4m^{(6d+8)/(d+2)}24^{2d}\ln^{3d+1}m \right),  O\Big( d^{5/2}m^{(4d+4)/(d+2)}48^{d}\ln^{2d+1/2}m \Big)$,
$ \frac{\delta^2}{12^{d+1}} \Big)$-approximate basis for all MReLU functions. This completes the proof.
\end{proofof}

\fi

\section{Making Predictions for Naive Decision Makers under Smooth Best Response}

In this section we show how to make more interpretable predictions that are better suited to simple/naive decision makers. Our techniques produce predictions of the coefficients of basis functions in a high dimensional space, and to translate these into actions, a downstream decision maker must first produce an approximate representation of their loss function in this higher dimensional space, and then use it to best respond. We show how to produce predictions in the original label space so that a downstream decision maker can act straightforwardly using a smoothed version of best response --- i.e. that chooses a distribution over actions that is weighted towards higher utility actions, but is not a point-mass on the single best response action. 

For concreteness, we instantiate  ``smooth best response'' to take the form of quantal response, which is a common model of bounded rationality studied in the behavioral economics literature \citep{luce1959individual, mcfadden1976quantalchoice, mckelvey1995quantalresponse, anderson2002logit, goeree2002quantal}. Quantal response is motivated by the hypothesis that agents do not behave perfectly rationally and thus do not always perfectly best respond; instead, their choices are \textit{noisy}. This noise is captured by modeling agents who randomly select their actions according to some distribution. The quantal response function chooses actions with probability inversely proportional to the exponential loss. We note that the results we present below hold in general for any response function that is Lipschitz in the predictions and an approximate best response (the bounds we give depend on both the Lipschitz constant and the approximation factor), and our restriction to quantal response is purely for concreteness. 

\begin{definition}[Smooth Best Response]\label{def:smoothBR}
    The smooth best response to a prediction $p$ according to a loss function $\ell: \cA \times \cY \to [0,1]$ is the distribution $q^\ell(p) \in \Delta\cA$ where the weight placed on action $a$ is:
    \[
    q^\ell(a, p) \propto \exp(-\eta \cdot \ell(a, p))
    \]
    given some smoothness parameter $\eta > 0 $. Observe that as $\eta \to \infty$, the smooth best response converges to an exact best response.
\end{definition}

Next we give analogues of some definitions to accommodate for randomized actions. First we extend the definition of decision swap regret, which we state as an expectation over the random action.

\begin{definition}[$(\cL, \cC, \eps)$-Decision Swap Regret] 
Let $\cL$ be a family of loss functions and $\cC$ be a collection of policies $c: \cX \to \cA$. For a sequence of outcomes $y_1,...,y_T\in\cY$, we say $p_1,...,p_T\in\cP$ has $(\cL, \cC, \eps)$ decision swap regret if for any loss function $\ell \in \cL$, there is a randomized action selection rule $k^\ell: \cP \to \Delta\cA$ such that choosing actions $q_t = k^\ell(p_t(x_t))$ achieves:
\[
\frac{1}{T} \sum_{t=1}^T \E_{a\sim q_t}\left[ \ell(a, y_t) - \ell(c_{a}(x_t), y_t) \right] \leq \eps
\]
for any assignment of policies $\{c_a \in \cC\}_{a\in\cA}$.
\end{definition}



Next, we give the smooth analogues of decision calibration and decision cross calibration. At first blush, it might seem that extending decision calibration and decision cross calibration to distributions over actions would involve conditioning on the entire smooth response distribution. This would be unfortunate, as achieving such a notion would require regret bounds scaling exponentially with the cardinality of the action space. However, it turns out that conditioning on the probability of playing each action \emph{marginally} will be enough. 

\begin{definition}[$(\cL, \beta)$-Smooth Decision Calibration]
    Let $\cL$ be a family of loss functions $\ell: \cA \times \cY \to [0,1]$. Let $\beta:\R\to\R$. We say that a sequence of predictions $p_1,...,p_T$ is $(\cL, \beta)$-smooth decision calibrated with respect to a sequence of outcomes $y_1,...,y_T$ if, for every $\ell \in \cL$ and $a\in\cA$:
    \[
    \left\| \sum_{t=1}^T q^\ell(a, p_t) (p_t - y_t) \right\|_\infty \leq \beta(T^\ell_q(a))
    \]
    where $q^\ell(a, p_t)$ is the weight placed on action $a$ under the smooth best response distribution $q^\ell(p_t)$, and $T^\ell_q(a)=\sum_{t=1}^T q^\ell(a,p_t)$.
\end{definition}

We similarly extend decision cross calibration:

\begin{definition}[$(\cL, \cC, \alpha)$-Smooth Decision Cross Calibration]
    Let $\cL$ be a family of loss functions $\ell:\cA\times\cY\to[0,1]$ and $\cC$ be a collection of policies $c: \cX \to \cA$. Let $\alpha:\R\to\R$. We say that a sequence of predictions $p_1,...,p_T$ is $(\cL, \cC, \alpha)$-smooth decision cross calibrated with respect to a sequence of outcomes $y_1,...,y_T \in \cY$ if for every $\ell\in\cL$, $a,b\in\cA$, $c\in\cC$:
    \[
    \left\| \sum_{t=1}^T q^\ell(a, p_t) \1[c(x_t)=b] (p_t - y_t) \right\|_\infty \leq \alpha(T^\ell_q(a, b))
    \]
    where $q^\ell(a, p_t)$ is the weight placed on action $a$ under the distribution $q^\ell(p_t)$, and $T^\ell_q(a, b)=\sum_{t=1}^T q^\ell(a,p_t) \1[c(x_t)=b]$.

\end{definition}

With these extended definitions in hand, we can now state decision swap regret guarantees for linear losses when agents smoothly best respond. 

\begin{theorem}[Decision Swap Regret for Linear Losses Under Smooth Best Response]\label{thm:decisionreg-linear-smooth}
    Consider an outcome space $\cY \subseteq [0,1]^d$. Let $\cL$ be any collection of loss functions $\ell: \cA \times \cY \to [0,1]$ that are linear and $L$-Lipschitz in the second argument. Let $\cC$ be a collection of policies $c: \cX \to \cA$. If the sequence of predictions $p_1,...,p_T$ is $(\cL, \beta)$-smooth decision calibrated and $(\cL, \cC, \alpha)$-smooth decision cross calibrated, then it has $(\cL, \cC, \eps)$-decision swap regret where $$\eps \leq \frac{L|\cA|\beta(T/|\cA|) + L|\cA|^2 \alpha\left(T/|\cA|^2\right)}{T} + \frac{\ln|\cA|+1}{\eta}$$
    when agents play their smooth best response $q^\ell(p_t)$ with smoothness parameter $\eta$.
\end{theorem}

As stated, smooth best response only increases an agent's regret: observe that as $\eta$ goes to $\infty$---i.e. as the smooth best response converges to an exact best response---the term $\frac{\ln|\cA|+1}{\eta}$ goes to 0, and we recover the bound on decision swap regret under exact best response (Theorem \ref{thm:decisionreg-linear}). However, to give decision swap regret guarantees later on, we will need a form of agent response that, unlike the exact best response, is Lipschitz in our predictions---which smooth best response satisfies. 

The argument follows a similar structure as the argument when agents exactly best respond (Theorem \ref{thm:decisionreg-linear}): we show that smooth decision calibration ensures that our self-assessed expected losses for \textit{smooth} responses are (on average) accurate, while smooth decision cross calibration ensures that our self-assessed expected losses for the actions recommended by the benchmarks are (on average) accurate. Now that agents smoothly, rather than exactly, best respond, it is no longer the case that their actions are optimal given our predictions. However, we show that the smooth best response is \textit{approximately} optimal, and so the expected loss under smooth best response is still competitive (up to an additional approximation factor) against any choice of benchmark policies. 

\begin{lemma}\label{lem:decisioncalibration-linear-smooth}
    If $p_1,..,p_T$ is $(\cL,\beta)$-smooth decision calibrated, then for any $\ell \in \cL$:
    \[
    \left| \sum_{t=1}^T \E_{a\sim q^\ell(p_t)}[\ell(a, p_t) - \ell(a, y_t)] \right| \leq L|\cA|\beta(T/|\cA|)
    \]
\end{lemma}
\begin{proof}
    Expanding out the expectation and using linearity of $\ell$, we have that:
    \begin{align*}
        \left| \sum_{t=1}^T \E_{a\sim q^\ell(p_t)}[\ell(a, p_t) - \ell(a, y_t)] \right| &= \left| \sum_{t=1}^T \sum_{a\in\cA} q^\ell(a, p_t)(\ell(a, p_t) - \ell(a, y_t)) \right| \\
        &= \left|  \sum_{a\in\cA} (\ell(a,  \sum_{t=1}^T q^\ell(a, p_t) p_t) - \ell(a, \sum_{t=1}^T q^\ell(a, p_t) y_t)) \right| \\
        &\leq L \sum_{a\in\cA} \left\|  \sum_{t=1}^T q^\ell(a, p_t) (p_t - y_t) \right\|_\infty \\
        &\leq L \sum_{a\in\cA} \beta(T^\ell_q(a))
    \end{align*}
    where the first inequality follows by $L$-Lipschitzness of $\ell$, the second inequality follows from $(\cL, \beta)$-smooth decision calibration. By concavity of $\beta$ and the fact that $\sum_{a\in\cA} T^\ell_q(a) = \sum_{t=1}^T \sum_{a\in\cA} q^\ell(a, p_t)=T$, this expression is at most:
    \[
    L |\cA|\beta(T/|\cA|)
    \]
\end{proof}

\begin{lemma}\label{lem:multiaccuracy-linear-smooth}
    If $p_1,...,p_T$ is $(\cL, \cC, \alpha)$-smooth decision cross calibrated, then for any $\ell\in\cL$ and any choice of benchmark policies $\{c_a\in\cC\}_{a\in\cA}$:
    \[
    \left| \sum_{t=1}^T \E_{a\sim q^\ell(p_t)}[\ell(c_a(x_t), p_t) - \ell(c_a(x_t), y_t)] \right| \leq L|\cA|^2 \alpha\left(\frac{T}{|\cA|^2}\right)
    \]
\end{lemma}
    \begin{proof}
        Expanding out the expectation and using linearity of $\ell$, we can compute:
        \begin{align*}
            \left| \sum_{t=1}^T \E_{a\sim q^\ell(p_t)}[\ell(c_a(x_t), p_t) - \ell(c_a(x_t), y_t)] \right| &= \left| \sum_{t=1}^T \sum_{a \in \cA} q^\ell(a,p_t)(\ell(c_a(x_t), p_t) - \ell(c_a(x_t), y_t)) \right| \\
            &= \left| \sum_{t=1}^T \sum_{a,b \in \cA} q^\ell(a,p_t)\1[c_{a}(x_t)=b] (\ell(b, p_t) - \ell(b, y_t)) \right| \\
            &= \Bigg| \sum_{a,b\in\cA} \Bigg(\ell\left(b, \sum_{t=1}^T q^\ell(a,p_t) \1[c_{a}(x_t)=b] p_t\right) \\ & - \ell\left(b, \sum_{t=1}^T q^\ell(a,p_t) \1[c_{a}(x_t)=b] y_t\right)\Bigg) \Bigg| \\
            &\leq L \sum_{a,b\in\cA} \left\|  \sum_{t=1}^T q^\ell(a,p_t) \1[c_{a}(x_t)=b](p_t - y_t) \right\|_\infty \\
            &\leq L \sum_{a,b\in\cA} \alpha(T^\ell_q(a,b))
        \end{align*}
        where the first inequality follows from $L$-Lipschitzness of $\ell$, and the second inequality follows from $(\cL,\cC,\alpha)$-smooth decision cross calibration. By concavity of $\alpha$ and the fact that $$\sum_{a,b\in\cA} T^\ell_q(a,b) = \sum_{t=1}^T \sum_{a\in\cA} q^\ell(a,p_t) \sum_{b\in\cA} \1[c_{a}(x_t)=b] = T$$ this expression is at most:
        \begin{align*}
            L|\cA|^2 \alpha\left(\frac{T}{|\cA|^2}\right)
        \end{align*}
    \end{proof}



To complete the proof, we need an additional fact that controls the difference in expected loss when agents smoothly---rather than exactly---best respond.

\begin{lemma}\label{lem:smoothBR}
    Fix a loss $\ell$ and a prediction $p\in\cY$. Let $a^* = \BR^\ell(p)$. Then:
    \[
    \E_{a\sim q^\ell(p)}[\ell(a, p)] \leq \ell(a^*, p) + \frac{\ln|\cA|+1}{\eta}
    \]
    where $q^\ell(p)$ is the distribution given by smooth best response with smoothness parameter $\eta$.
\end{lemma}
\begin{proof}
    Fix any $x$. We see that:
    \begin{align*}
    \Pr[\ell(a, p)\geq x] &\leq \frac{\Pr[\ell(a,p)\geq x]}{\Pr[\ell(a,p)=\ell(a^*,p)]}
    \leq \frac{|\cA|\exp(-\eta\cdot x)}{\exp(-\eta\cdot \ell(a^*, p))}
    = |\cA|\exp(-\eta(x-\ell(a^*,p)))
    \end{align*}
    To obtain the second inequality, notice that for any action satisfying $\ell(a,p)\geq x$, its weight is proportional to at most $\exp(-\eta\cdot x)$, and so we arrive at the numerator by summing over every action. 
    
    Now, plugging in $x=\ell(a^*,p) + \frac{1}{\eta}(\ln|\cA|+c)$, we have:
    \begin{align*}
        \Pr[\eta(\ell(a,p)-\ell(a^*,p)) \geq \ln|\cA|+c] &= \Pr[\ell(a,p) \geq \ell(a^*,p) + \frac{1}{\eta}(\ln|\cA|+c)] \\
        &\leq |\cA|\exp(-\ln|\cA|-c) \\
        &= \exp(-c)
    \end{align*}
    Let $Z = \eta(\ell(a,p)-\ell(a^*,p))$. Since $Z$ is a non-negative random variable, we can write:
    \begin{align*}
        \E[Z] = \int_0^\infty \Pr[Z\geq x] dx = \int_{c=-\ln|\cA|}^\infty \Pr[Z\geq \ln|\cA|+c] dc \leq \int_{c=-\ln|\cA|}^0 1 dc + \int_0^\infty \exp(-c) dc \leq \ln|\cA| + 1
    \end{align*}
    The claim then follows. 
\end{proof}

Now we prove Theorem \ref{thm:decisionreg-linear-smooth}.

\begin{proofof}{Theorem \ref{thm:decisionreg-linear-smooth}}
    Fix any $\ell \in \cL$. Let $a_t^* = \BR^\ell(p_t)$. We can compute the decision swap regret to any choice of benchmark policies $\{c_a \in \cC\}_{a\in\cA}$:
    \begin{align*}
        \frac{1}{T} \sum_{t=1}^T \E_{a\sim q^\ell(p_t)}[\ell(a, y_t) - \ell(c_a(x_t), y_t)] 
        &\leq \frac{1}{T} \sum_{t=1}^T \E_{a\sim q^\ell(p_t)}[\ell(a, p_t) - \ell(c_a(x_t), p_t)] \\ & \ \ \ \ \ \ + \frac{L|\cA|\beta(T/|\cA|) + L|\cA|^2 \alpha\left(\frac{T}{|\cA|^2}\right)}{T} \\
        &\leq \frac{1}{T} \sum_{t=1}^T \left(\E_{a\sim q^\ell(p_t)}[\ell(a, p_t)] - \ell(a_t^*, p_t)\right) \\ & \ \ \ \ \ \ + \frac{L|\cA|\beta(T/|\cA|) + L|\cA|^2 \alpha\left(\frac{T}{|\cA|^2}\right)}{T} \\
        &\leq \frac{\ln|\cA|+1}{\eta} + \frac{L|\cA|\beta(T/|\cA|) + L|\cA|^2 \alpha\left(\frac{T}{|\cA|^2}\right)}{T}
    \end{align*}
    We apply Lemmas \ref{lem:decisioncalibration-linear-smooth} and \ref{lem:multiaccuracy-linear-smooth} in the first inequality. The second inequality follows from the definition of best response, and the last inequality follows from Lemma \ref{lem:smoothBR}. 
\end{proofof}

Turning to the algorithmic problem, a separate instantiation of \textsc{Unbiased-Prediction} achieves the smooth variants of decision calibration and decision cross calibration; we will refer to this instantiation as \textsc{Smooth-Decision-Swap}. As before, the guarantees this algorithm will directly inherit from the guarantees of \textsc{Unbiased-Prediction}.

\begin{theorem}\label{thm:smooth-unbiasedalg}
    Consider a convex outcome space $\cY \subseteq [0,1]^d$ and a prediction space $\cP=\cY$. Let $\cL$ be a collection of loss functions $\ell: \cA \times \cY \to [0,1]$ that are linear and $L$-Lipschitz in the second argument. Let $\cC$ be a collection of policies $c: \cX \to \cA$. There is an instantiation of \textsc{Unbiased-Prediction} \citep{noarov2023highdimensional}---which we will call \textsc{Smooth-Decision-Swap}---producing predictions $\pi_1,...,\pi_T \in \Delta \cP$ satisfying, for any sequence of outcomes $y_1,...,y_T \in \cY$:
    \begin{itemize}
    \item For any $\ell \in \cL, a\in\cA$:
    \[
    \E_{p_t\sim\pi_t}\left[ \left\| \frac{1}{T} \sum_{t=1}^T q^\ell(a, p_t) (p_t - y_t) \right\|_\infty \right] \leq O\left( \ln(d|\cA||\cL|T) + \sqrt{\ln(d|\cA||\cL|T) \E_{p_t\sim\pi_t}[T^\ell_q(a)]} \right)
    \]
    \item For any $h \in \cH, a,b\in\cA, c\in\cC$:
    \[
    \E_{p_t\sim\pi_t}\left[ \left\| \frac{1}{T}\sum_{t=1}^T q^\ell(a, p_t) \1[c(x_t)=b] (p_t - y_t) \right\|_\infty \right] \leq O\left( \ln(d|\cA||\cL||\cC|T) + \sqrt{ \ln(d|\cA||\cL||\cC|T) \E_{p_t\sim\pi_t}[T^\ell_q(a, h)]}\right)
    \]
    \end{itemize} 
    where $q^\ell(a, p_t)$ is the weight placed on action $a$ under the smooth best response distribution $q^\ell(p_t)$.
\end{theorem}

Now we can substitute the above bounds into Theorem \ref{thm:decisionreg-linear-smooth}.

\begin{corollary}\label{cor:decisionreg-linear}
    Consider a convex outcome space $\cY \subseteq [0,1]^d$ and a prediction space $\cP=\cY$. Let $\cL$ be a family of loss functions $\ell: \cA \times \cY \to [0,1]$ that are linear and $L$-Lipschitz in the second argument. Let $\cC$ be a collection of policies $c: \cX \to \cA$. The sequence of predictions $\pi_1,...,\pi_T\in\Delta\cP$
    output by \textsc{Smooth-Decision-Swap} achieves $(\cL, \cC, \eps)$-decision swap regret (in expectation over the randomized predictions), where 
    \[
    \eps \leq \frac{\ln|\cA|+1}{\eta} + O\left(L |\cA| \sqrt{\frac{\ln(d|\cA||\cL||\cC|T)}{T}} \right)
    \]
    when agents play their smooth best response $q^\ell(p_t)$ with smoothness parameter $\eta$.
\end{corollary}

\subsection{Decision Swap Regret under Smooth Best Response}


Now we show how to make predictions of the outcome $y$ itself, for loss families that can be approximately represented by a basis, if agents smoothly best respond. To do so, we will make only a slight modification to our previous algorithm. As before, we will generate a prediction $\hat{p}_t$ of the higher dimensional basis representation $s(y)$. But, we will not broadcast $\hat{p}_t$ as our prediction; instead, we will broadcast a value $p_t$ in the \textit{pre-image} of $s$, i.e. $p_t\in s^{-1}(\hat{p}_t)$ (the inverse needs not to be unique; if there are multiple elements in the pre-image, we can choose on arbitrarily). Thus, we can interpret $p_t$ as a prediction of the outcome $y_t$. 

Now, agents with losses that have an exact basis representation (e.g. $L_p$ losses $\cF^d_p$, monomials of linear functions $\cF^d_{\Omega_{\beta,g}}$) can simply best respond to $p_t$ according to their loss, and so can obtain the same decision swap regret guarantees of Theorem \ref{thm:decisionreg-convex}. On the other hand, agents with losses that have only an approximate basis, must now smoothly best respond; this is to bypass discontinuities of the best response function when swapping between agents' actual losses and their approximate representations. Below we give decision swap regret guarantees for agents who smoothly best respond to predictions $p_t$.

Before stating the guarantees, we point out an additional assumption that we will rely on in this setting. In particular, it is now important that we predict within the space of basis transformations $\{s(y):y\in\cY\}$, in order to guarantee that the inverse exists. However, the $\textsc{Unbiased-Prediction}$ algorithm---and hence the instantiation we will use, $\textsc{Smooth-Decision-Swap}$---requires the prediction/outcome space to be convex. Thus, in order to implement its guarantees, we will be limited to handling loss functions with corresponding bases $\cS$ that define a convex space. 

\begin{algorithm}[H]
    \KwIn{Family of convex losses $\cL=\{\ell:\cA\times[0,1]^d\to[0,1]\}$, $(n,\gamma,\delta)$-approximation $\cS$ for $\cL$, collection of policies $\cC$, \textsc{Smooth-Decision-Swap} algorithm}
    \KwOut{Sequence of predictions $p_1,...,p_T \in [0,1]^d$} 
    \vspace{.5em}
    
    Construct $\hat{\cL}=\{\hat{\ell}\}_{\ell\in\cL}$, the corresponding family of linear losses given by the basis $\cS$ \;
    
    Construct $\hat{\cL}_\gamma$, the $(n\gamma)$-cover of $\hat{\cL}$ given by Lemma \ref{lem:discretization-linear}\;

    Instantiate a copy of $\textsc{Smooth-Decision-Swap}$ with loss family $\hat{\cL}_\gamma$ and collection of policies $\cC$\;

    \For{$t=1$ \KwTo $T$}{
        Receive $x_t$\;
        
        Let $\pi_t = \textsc{Smooth-Decision-Swap}_t(\{x_r\}_{r=1}^{t}, \{s(y_r)\}_{r=1}^{t-1})$, the distribution over predictions output by \textsc{Smooth-Decision-Swap} on round $t$ given contexts $\{x_r\}_{r=1}^{t}$ and outcomes $\{s(y_r)\}_{r=1}^{t-1})$\;

        Predict $p_t \in s^{-1}(\hat{p}_t)$, where $\hat{p}_t\sim\pi_t$\;
        
        Observe $y_t$\;
    }
    
    \caption{Smooth Decision Swap Regret for Non-Linear Losses}
    \label{alg:convex-smooth}
\end{algorithm}

\begin{theorem}\label{thm:decisionreg-convex-smooth}
    Let $\cY=[0,1]^d$ be the outcome space and $\cC$ be a collection of policies $c:\cX\to\cA$. Let $\cL$ be a family of loss functions $\ell:\cA\times\cY\to[0,1]$. Suppose $\cS$ is a $(n,\lambda,\delta)$-approximate basis for $\cL$ such that the space $\cP_\cS = \{s(y): y\in\cY\}$ is convex. Then, Algorithm \ref{alg:convex-smooth} produces predictions $p_1,...,p_T \in [0,1]^d$ that has $(\cL,\cC,\eps)$-decision swap regret for 
    \[
    \eps \leq \frac{\ln|\cA|+1}{\eta} + O\left(\lambda |\cA| \sqrt{\frac{n|\cA|\ln(n(\lambda/\gamma)|\cA||\cC|T)}{T}} \right) + 2\delta + 2n\gamma + |\cA|(e^{2\eta(\delta+n\gamma)}-1)
    \]
    when playing the smooth best response $q^\ell(p_t)$ with smoothness parameter $\eta$.  
\end{theorem}

The argument proceeds similarly to that of Theorem \ref{thm:decisionreg-convex}. Our starting point is again the basis representation $\cS$, which tells us that $\ell$ can be pointwise approximated by a higher dimensional linear loss $\hat{\ell}$ over a basis function $s$. That is, there exists $\hat{\ell}$ such that for all actions $a$, $|\hat{\ell}(a, s(p_t)) - \ell(a, p_t)| \leq \delta$. Since $p_t$ belongs to the pre-image of $\hat{p}_t$ under $s$ --- that is, $s(p_t) = s(s^{-1}(\hat{p}_t)) = \hat{p}_t$ --- we can write $|\hat{\ell}(a, \hat{p}_t) - \ell(a, p_t)| \leq \delta$. Thus, fixing a sequence of actions, we can compare the losses given by $\ell$ under $p_t$ to the losses given by $\hat{\ell}$ under predictions $\hat{p}_t$. 

The difficulty is that the (randomized) actions obtained by smoothly best responding to $p_t$ according to $\ell$ are not necessarily the same as those obtained by smoothly best responding to $\hat{p_t}$ according to $\hat{\ell}$. To address this issue, the proof relies on an additional conceptual step. We will show that since $\ell(a, p_t)$ is close to $\hat{\ell}(a, \hat{p}_t)$ for all $a$, the smooth best response distributions in both cases are close in total variation distance. Therefore, the expected losses are comparable. 

Putting everything together, then, the expected decision swap regret incurred by loss $\ell$ when smoothly best responding to $p_t$ is not too different than the expected decision swap regret incurred by $\tilde{\ell}$ when smoothly best responding to $\tilde{p}_t$. After an appropriate discretization of losses, we can apply the guarantees of \textsc{Smooth-Decision-Swap} (Theorem \ref{thm:unbiasedalg}) to conclude that the latter is $o(T)$ and therefore so is the former. 

Here we will see a tension in the choice of our smoothness parameter $\eta$: the guarantees of \textsc{Smooth-Decision-Swap} improve as $\eta$ increases (i.e. move closer to the exact best response), but we will see that the total variation distance between smooth best response distributions shrink as $\eta$ decreases (the ``smoother" the distributions, the more indistinguishable they are). Thus, the optimal regret bounds will also now depend on $\eta$, in addition to approximation parameters $n, \lambda$, and $\delta$.

Below we formalize these ideas. We first argue that nearby losses induce nearby smooth best response distributions.

\begin{lemma}\label{lem:smoothdist}
    For any two loss functions $\ell:\cA\times[0,1]^d\to[0,1]$ and $ \hat{\ell}:\cA\times[0,1]^n\to[0,1]$, and any transformation function $s:[0,1]^d\to[0,1]^n$, if $|\hat{\ell}(a, s(y)) - \ell(a, y)| \leq \delta$, then $|q^{\hat{\ell}}(a, s(y)) - q^\ell(a, y)| \leq e^{2\eta\delta}-1$. 
\end{lemma}
\begin{proof}
    Without loss, assume that $\ell(a, y) \leq \hat{\ell}(a, s(y))$. We compute:
    \begin{align*}
        \frac{q^\ell(a, y)}{q^{\hat{\ell}}(a, s(y))} &= \frac{\frac{\exp(-\eta\cdot\ell(a,y))}{\sum_{a'\in\cA}\exp(-\eta\cdot\ell(a',y))}}{\frac{\exp(-\eta\cdot\hat{\ell}(a, s(y)))}{\sum_{a'\in\cA}\exp(-\eta\cdot\hat{\ell}(a', s(y)))}} \\
        &= \exp(-\eta\cdot(\ell(a,y)-\hat{\ell}(a, s(y)))) \cdot \frac{\sum_{a'\in\cA}\exp(-\eta\cdot\hat{\ell}(a', s(y)))}{\sum_{a'\in\cA}\exp(-\eta\cdot\ell(a',y))} \\
        &\leq \exp(\eta\delta) \cdot \frac{\sum_{a'\in\cA}\exp(-\eta\cdot \hat{\ell}(a', s(y)))}{\sum_{a'\in\cA}\exp(-\eta\cdot(\hat{\ell}(a', s(y)) + \delta))} \\
        &= \exp(\eta\delta) \cdot \exp(\eta\delta) \\
        &= \exp(2\eta\delta)
    \end{align*}
    Rearranging the expression, we have that $q^\ell(a, y) \leq q^{\hat{\ell}}(a, s(y))\exp(2\eta\delta)$. Since $q^{\hat{\ell}}(a, s(y)) \leq 1$, we can conclude that
    \[
    q^\ell(a, y) - q^{\hat{\ell}}(a, s(y)) \leq q^{\hat{\ell}}(a, s(y))\exp(2\eta\delta) - q^{\hat{\ell}}(a, s(y)) = q^{\hat{\ell}}(a, s(y)) (\exp(2\eta\delta)-1) \leq \exp(2\eta\delta)-1
    \]
    This completes the proof. 
\end{proof}

Next we prove Theorem \ref{thm:decisionreg-convex-smooth}.
\vspace{.7em}
\begin{proofof}{Theorem \ref{thm:decisionreg-convex-smooth}}
    Fix a choice of $\ell\in\cL$, and consider a run of Algorithm \ref{alg:convex-smooth}. Let $\hat{\cL} = \{\hat{\ell}\}_{\ell\in\cL}$ be the family of linear approximations to $\cL$ given by $\cS$. By $(n,\lambda,\delta)$-approximation, there is an $\hat{\ell}\in\hat{\cL}$ such that $\hat{\ell}$ is a linear and $\lambda$-Lipschitz function of $[0,1]^n$ and for any $a\in\cA$ and $y\in\cY$, $\left| \hat{\ell}(a, s(y)) - \ell(a, y) \right| \leq \delta$. Moreover, by Lemma \ref{lem:discretization-linear}, there is a a $(n\gamma)$-cover $\hat{\cL}_\gamma$ of $\hat{\cL}$ such that there exists $\hat{\ell}_\gamma\in\hat{\cL}_\gamma$ satisfying $\left| \hat{\ell}_\gamma(a, s(y)) - \hat{\ell}(a, s(y)) \right| \leq n\gamma$, and therefore satisfying:
    \[
    \left| \hat{\ell}_\gamma(a, s(y)) - \ell(a, y) \right| \leq \delta + n\gamma
    \]
    Furthermore, $|\hat{\cL}_\gamma| \leq (2\lambda/\gamma)^{n|\cA|}$.

    We first analyze the guarantees of the sequence of predictions $\hat{p}_1,...\hat{p}_T \in [0,1]^n$ produced from calls to the \textsc{Smooth-Decision-Swap} algorithm. Fix any assignment of benchmark policies $\{c_a\}_{a\in\cA}$. Since $\hat{\ell}_\gamma$ is linear in $[0,1]^n$, Corollary \ref{cor:decisionreg-linear} bounds the decision swap regret of an agent with loss $\hat{\ell}_\gamma$ \textit{who smoothly best responds to $\hat{p}_t$}. We have that:
    \begin{align*}
        \frac{1}{T} \sum_{t=1}^T \E_{\hat{p}_t\sim\pi_t} \E_{a\sim q^{\hat{\ell}_\gamma}(\hat{p}_t)} \left[\hat{\ell}_\gamma(a, s(y_t)) - \hat{\ell}_\gamma(c_{a}(x_t), s(y_t))\right]
        &\leq \frac{\ln|\cA|+1}{\eta} + O\left(\lambda |\cA| \sqrt{\frac{\ln(n|\cA||\hat{\cL}_\gamma||\cC|T)}{T}} \right) \\
        &\leq \frac{\ln|\cA|+1}{\eta} + O\left(\lambda |\cA| \sqrt{\frac{\ln(n(\lambda/\gamma)^{n|\cA|}|\cA||\cC|T)}{T}} \right) \\
        &\leq \frac{\ln|\cA|+1}{\eta} + O\left(\lambda |\cA| \sqrt{\frac{n|\cA|\ln(n(\lambda/\gamma)|\cA||\cC|T)}{T}} \right)
        \tag{$*$}
    \end{align*}


    Our goal is to bound the decision swap regret of an agent with loss $\ell$ who smoothly best responds to $p_t$. We can translate the above guarantee to our desired guarantee by noticing that $\hat{\ell}_\gamma$ approximates $\ell$ pointwise across actions and outcomes---when applied the transformation $s$. And so, the smooth best response to $\hat{p}_t$ smoothly approximates the smooth best response to a prediction $p_t$ obtained by applying the inverse transformation $s^{-1}$. The inverse exists, since our predictions $\hat{p}_t$ lie in $\cP_\cS$. We can derive:
    \begin{align*}
        & \frac{1}{T} \sum_{t=1}^T \E_{\hat{p}_t\sim\pi_t} \E_{a\sim q^{\ell}(p_t)}[\ell(a, y_t) - \ell(c_a(x_t), y_t)] \\
        &\leq \frac{1}{T} \sum_{t=1}^T \E_{a\sim q^{\ell}(p_t)}[\hat{\ell}_\gamma(a, s(y_t)) - \hat{\ell}_\gamma(c_a(x_t), s(y_t))] + 2\delta + 2n\gamma \\
        &= \frac{1}{T} \sum_{t=1}^T \E_{\hat{p}_t\sim\pi_t}\left[ \sum_{a\in\cA} q^\ell(a, p_t) (\hat{\ell}_\gamma(a, s(y_t)) - \hat{\ell}_\gamma(c_a(x_t), s(y_t)))\right] + 2\delta + 2n\gamma \\
        &\leq \frac{1}{T} \sum_{t=1}^T \E_{\hat{p}_t\sim\pi_t}\left[ \sum_{a\in\cA} q^{\hat{\ell}_\gamma}(a, \hat{p}_t) (\hat{\ell}_\gamma(a, s(y_t)) - \hat{\ell}_\gamma(c_a(x_t), s(y_t)))\right] + 2\delta + 2n\gamma + |\cA|(e^{2\eta(\delta+n\gamma)}-1) \\
        &= \frac{1}{T} \sum_{t=1}^T \E_{\hat{p}_t\sim\pi_t} \E_{a\sim q^{\hat{\ell}_\gamma}(\hat{p}_t)}[\hat{\ell}_\gamma(a, s(y_t)) - \hat{\ell}_\gamma(c_a(x_t), s(y_t))] + 2\delta + 2n\gamma + |\cA|(e^{2\eta(\delta+n\gamma)}-1) \\
        &\leq \frac{\ln|\cA|+1}{\eta} + O\left(\lambda |\cA| \sqrt{\frac{n|\cA|\ln(n(\lambda/\gamma)|\cA||\cC|T)}{T}} \right) + 2\delta + 2n\gamma + |\cA|(e^{2\eta(\delta+n\gamma)}-1)
    \end{align*}
    In the first inequality, we exchange $\ell$ for $\hat{\ell}_\gamma$ by incurring at most a factor of $\delta+n\gamma$ for each the realized and the benchmark loss. In the second inequality, we make use of Lemma \ref{lem:smoothdist} and the fact that $s(p_t) = s(s^{-1}(\hat{p}_t)) = \hat{p}_t$ and thus $|\hat{\ell}_\gamma(a, s(p_t)) - \ell(a, p_t)| = |\hat{\ell}_\gamma(a, \hat{p}_t) - \ell(a, p_t)| \leq \delta+n\gamma$. The final inequality follows from Equation $(*)$, the bound on decision swap regret with respect to $\hat{\ell}_\gamma$ when playing $q^{\hat{\ell}_\gamma}(\hat{p}_t)$. 
    This completes the proof.
\end{proofof}

\section{Unbiased Prediction Algorithm}\label{app:unbiased-prediction}
Here we present the $\textsc{Unbiased-Prediction}$ algorithm of \citet{noarov2023highdimensional}, from which our guarantees in Theorem \ref{thm:unbiasedalg} follow. The algorithm makes predictions that are unbiased conditional on a collection of \textit{events} $\cE$. Formally an event $E\in\cE$ is a mapping from context and prediction to $[0,1]$ , i.e. $E: \cX\times\cP\to[0,1]$. The algorithm's conditional bias guarantee depends logarithmically on the number of events:

\begin{theorem}\citep{noarov2023highdimensional}\label{thm:unbiased-prediction-algorithm}
For a collection of events $\cE$ and convex prediction/outcome space $\cY\subseteq [0,1]^d$, Algorithm \ref{alg:unbiased-prediction} produces predictions $\pi_1,...,\pi_T \in \Delta \cY$ such that for any sequence of outcomes $y_1,...,y_T \in \cY$ chosen by the adversary:
    \[
    \left\| \sum_{t=1}^T \E_{p_t\sim\pi_t}[E(x_t, p_t)(p_t - y_t)] \right\|_\infty \leq O\left( \ln(d|\cE|T) + \sqrt{\ln(d|\cE|T) \cdot n_T(E) } \right)
    \]
    where $n_T(E) = \sum_{t=1}^T \E_{\hat{y}_t\sim p_t}[E(\hat{y}_t)]$. The algorithm can be implemented with per-round running time scaling polynomially in $d$ and $|\cE|$.
\end{theorem}

\begin{algorithm}[H]
    \For{$t=1$ \KwTo $T$}{
        Observe $x_t$\;
        
        Define the distribution $q_t \in \Delta [2d|\cE|]$ such that for $E \in \cE, i\in[d], \sigma\in \{\pm 1\}$,
        \[
        q_t^{E, i, \sigma} \propto \exp\left( \frac{\eta}{2} \sum_{s=1}^{t-1} \sigma \cdot \E_{p_s\sim\pi_s}[E(x_s, p_s) (p_s^i - y_s^i)] \right)
        \]

        Output the solution to the minmax problem:
        \[
        \pi_t \gets \argmin_{\pi_t' \in \Delta\cY} \max_{y \in \cC} \E_{p_t \sim \pi_t'}\left[\sum_{E, i, \sigma} q_t(E, i, \sigma) \cdot \sigma \cdot E(x_s, p_s) \cdot (p_s^i - y_s^i) \right]
        \]        
    }    
    \caption{$\textsc{Unbiased-Prediction}$}
    \label{alg:unbiased-prediction}
\end{algorithm}

For loss family $\cL$, action space $\cA$, and collection of policies $\cC$, Theorem \ref{thm:unbiasedalg} instantiates Algorithm \ref{alg:unbiased-prediction} with the collection of events $\cE = \{ \1[\BR^\ell(p_t) = a] \}_{\ell\in\cL, a\in\cA} \cup \{ \1[\BR^\ell(p_t) = a,c(x_t)=b] \}_{\ell\in\cL, a\in\cA, c\in\cC}$. 
\ifarxiv
Thus, $|\cE| = |\cA||\cL| + |\cA||\cL||\cC|$, from which our guarantees follow. 
\fi

\begin{theorem}\label{thm:unbiasedalg}
    Consider a convex outcome space $\cY \subseteq [0,1]^d$ and the prediction space $\cP=\cY$. Let $\cL$ be a collection of loss functions $\ell: \cA \times \cY \to [0,1]$ that are linear and $L$-Lipschitz in the second argument. Let $\cC$ be a collection of policies $c: \cX \to \cA$. There is an instantiation of \textsc{Unbiased-Prediction} \citep{noarov2023highdimensional}---which we call \textsc{Decision-Swap}---producing predictions $\pi_1,...,\pi_T \in \Delta \cP$ satisfying, for any sequence of outcomes $y_1,...,y_T \in \cY$:
    \begin{itemize}
    \item For any $\ell \in \cL, a\in\cA$:
    \[
    \E_{p_t\sim\pi_t}\left[ \left\| \sum_{t=1}^T \1[\BR^\ell(p_t) = a] (p_t - y_t) \right\|_\infty \right] \leq O\left( \ln(d|\cA||\cL||\cC|T) + \sqrt{\ln(d|\cA||\cL||\cC|T) \E_{p_t\sim\pi_t}[T^\ell(a)]} \right)
    \]
    \item For any $\ell\in\cL, a,b\in\cA, c\in\cC$:
    \begin{align*}
    \E_{p_t\sim\pi_t}\left[ \left\| \sum_{t=1}^T \1[\BR^\ell(p_t) = a,c(x_t)=b] (p_t - y_t) \right\|_\infty \right] \\ \leq O\left( \ln(d|\cA||\cL||\cC|T) + \sqrt{\ln(d|\cA||\cL||\cC|T) \E_{p_t\sim\pi_t}[T^\ell(a, b)]}\right)
    \end{align*}
    \end{itemize}
\end{theorem}

\section{Azuma-Hoeffding's Inequality}
\begin{lemma}\label{lem:azuma}
If $X_1, \ldots, X_T$ is a martingale difference sequence, and for every $t$, with probability $1,\left|X_t\right| \leq M$. Then with probability $1-\delta$:
$$
\left|\sum_{t=1}^T X_t\right| \leq M \sqrt{2 T \ln \frac{2}{\delta}}
$$
\end{lemma}

\end{document}